\documentclass[11pt,letterpaper]{article}
\usepackage[T1]{fontenc}
\usepackage[margin=1in]{geometry}
\usepackage[round,authoryear]{natbib}
\usepackage{microtype}
\usepackage{amsmath,amssymb,amsfonts,bm}
\usepackage{amsthm}
\usepackage{graphicx,booktabs,multirow,array}
\usepackage[section]{placeins}
\usepackage{hyperref,url}
\usepackage{xcolor}
\hypersetup{colorlinks=true,linkcolor=blue!50!black,citecolor=blue!50!black,urlcolor=blue!60!black}
\newcommand{\E}{\mathbb{E}}
\newcommand{\R}{\mathbb{R}}
\newcommand{\sg}{\operatorname{sg}}
\newcommand{\CE}{\operatorname{CE}}

\newtheorem{lemma}{Lemma}

\title{Reasoning with Continuous Latent Diffusion}
\author{Xiang Cheng\\
Duke University\\
Department of Electrical and Computer Engineering\\
\texttt{xiang.cheng@duke.edu}}
\date{}

\begin{document}
\maketitle

\begin{abstract}
We introduce the Continuous Embedding Diffusion Reasoner (CEDR), a training and inference pipeline for continuous latent diffusion models that achieves strong performance on mathematical reasoning and code generation. We learn compact answer representations from three hidden layers of the autoregressive Qwen3-4B model. A staged curriculum first trains the ELF-based diffusion backbone with teacher prompt conditioning, then initializes a compact prompt encoder through imitation and jointly adapts both networks. This replaces the large teacher contextual prompt encoder at inference with a compact learned Transformer. The multilayer answer representation also enables asynchronous inference, with different components denoised under separate clocks. We further adapt DiffusionNFT to learned self-conditioning guidance, enabling reinforcement learning with verifiable rewards. Controlled ablations evaluate the effects of representation design, prompt training, and denoising schedules. CEDR outperforms recent continuous-diffusion baselines at comparable backbone scales, including before reinforcement learning. With a 638M-parameter denoising backbone and DMD2-style stochastic sampling, post-NFT CEDR-L achieves \{pass@1, pass@8\} accuracies of \{68.97\%, 80.82\%\} on GSM8K and \{32.40\%, 53.80\%\} on MATH500 at 64 denoising steps per sample, and \{37.56\%, 51.00\%\} on HumanEval+ and \{26.96\%, 39.92\%\} on MBPP+ at 128 steps per sample.
Code and model weights are available at \url{https://github.com/chengxiang/CEDR} and \url{https://huggingface.co/xc91/CEDR}, respectively.
\end{abstract}

\section{Introduction}
\label{sec:introduction}

Autoregressive language models have made substantial progress on mathematical reasoning \citep{minerva,deepseekr1}. Continuous diffusion offers a complementary approach: it generates complete solutions by iteratively refining a continuous latent representation of all answer positions. In this work, we investigate the capabilities of continuous diffusion language models on challenging mathematical reasoning and code generation tasks, and the design choices that enable strong performance.

We introduce the Continuous Embedding Diffusion Reasoner (CEDR), based on the simple continuous flow formulation of Embedded Language Flows (ELF; \citealp{elf}). A key design choice is the latent representation, which must capture the complex semantics and dependencies within reasoning solutions while remaining amenable to generation via continuous denoising. Our pipeline learns compact answer representations from three hidden layers of Qwen3-4B \citep{qwen3} (Figure~\ref{fig:pipelineoverview}). A staged curriculum first trains the diffusion model with teacher prompt features, then fits a compact prompt encoder and jointly adapts both networks. The resulting model performs prompt encoding and answer generation without an external contextual encoder, retaining Qwen's frozen token-embedding table for prompt input. Our decomposable answer representation supports both standard synchronous denoising and asynchronous algorithms that refine different components at different rates (Section~\ref{sec:clocks}).

Empirically, CEDR-L exceeds recent comparable-scale continuous-diffusion baselines on mathematical reasoning and code generation (Tables~\ref{tab:headlinecomparison} and~\ref{tab:codingcomparison}). We further investigate the efficacy of diffusion reinforcement learning techniques on our model. We adapt DiffusionNFT \citep{diffusionnft} to accommodate learned self-conditioning guidance and show that it can further improve model accuracy, with clear gains on mathematical reasoning (Figure~\ref{fig:headlinestages}).

\paragraph{Contributions.}
\label{sec:contributions}

Our complete reasoning pipeline comprises the following key contributions:

\begin{enumerate}
    \item \textbf{Strong reasoning with continuous latent diffusion.}
    CEDR shows strong performance on math and coding tasks, outperforming recent comparable-sized continuous-diffusion baselines on GSM8K, MATH500, HumanEval and MBPP (Tables~\ref{tab:headlinecomparison}--\ref{tab:headlinelownfe} and~\ref{tab:codingcomparison}).

    \item \textbf{Representation design using multilayer hidden activations of an autoregressive reasoning model.}
    We empirically investigate the tradeoff between preserving reasoning information and ease of generation (Figure~\ref{fig:representationcomparison}). Guided by these findings, we learn compact projections of Qwen's multilayer hidden activations using a decoding cross-entropy loss through the frozen teacher (Figure~\ref{fig:pipelineoverview}). Our decomposable representation enables asynchronous inference, with separate denoising clocks for different layer components (Section~\ref{sec:clocks}; Table~\ref{tab:inferenceclocks}).

    \item \textbf{Replacing external prompt encoders with a lightweight Transformer.}
    We train a lightweight Transformer to replace Qwen for prompt encoding, removing dependence on an external contextual encoder. In the joint stage of our curriculum, the prompt encoder and diffusion backbone learn together through flow-matching and decoding losses on the answer (Figure~\ref{fig:promptcurriculum}). We show that preserving the prompt information needed for answer denoising is sufficient for an ideal denoiser; exact imitation of Qwen's features is unnecessary (Lemma~\ref{lem:conditionalrepresentation}). This provides a principled basis for learning prompt representations through the downstream generative objective.

    \item \textbf{Guidance-compatible diffusion reinforcement learning for reasoning.}
    We adapt DiffusionNFT \citep{diffusionnft} to work with learned self-conditioning guidance while retaining guided rollouts (Section~\ref{sec:nftsccfg}). Our modified NFT training uses both model rollouts and gold solutions from the training set. We show that the gold-solution term anchors the updated denoising field toward a combination of the gold-data denoising target and the previous policy's field (Lemma~\ref{lem:goldanchor}).
\end{enumerate}

\section{Related work}
\label{sec:relatedwork}

\textbf{Diffusion language models.} Language diffusion operates in continuous representations \citep{plaid,tess2,hyperspherical} or discrete token spaces \citep{mdlm,llada,dbtm}, with recent continuous models targeting reasoning, coding, and few-step generation \citep{posterior,mlfm,plaidq}. We build on ELF \citep{elf}, which denoises contextual representations and shares a backbone between denoising and token decoding. Concurrent to our paper, ELF-REG \citep{scalingdlm} augments ELF with teacher-feature alignment and a jointly denoised global representation. Like ELF and ELF-REG, CEDR generates answers by diffusion in latent space. A key distinction in this paper is prompt conditioning: we learn a compact encoder to replace the external pretrained contextual encoder, retaining only a frozen token lookup for prompt input. CEDR thus demonstrates that strong reasoning is achievable with the diffusion backbone and a compact learned conditioning network, without running a large pretrained Transformer to encode prompts at inference.

\textbf{Representation design and denoising order.} Related work learns or aligns diffusion representations in vision and language \citep{repa,textldm,ldlm}, and explores separate denoising clocks for representation components in vision \citep{sfd,latentforcing}. We study learned multilayer language representations for both decoding and generation, and use their structure for asynchronous inference.

\textbf{Stochastic sampling.} Consistency models alternate clean prediction with fresh noise injection \citep{consistencymodels}, a construction adopted for multi-step generation in DMD2 \citep{dmd2}. EDM adds noise churn before a numerical denoising step \citep{edm}. ELF-REG also evaluates clean-prediction re-noising for continuous language generation \citep{scalingdlm}. We apply these established sampling ideas to asynchronous local clocks (Section~\ref{sec:clocks}).

\textbf{Diffusion reinforcement learning.} Reward-based post-training has also been explored for diffusion reasoning \citep{d1,llada15,ladirl}. We adapt DiffusionNFT \citep{diffusionnft}, which optimizes rewarded endpoints through forward-process regression. Our adaptation accommodates ELF's learned self-conditioning guidance, adds gold-solution anchoring, and jointly updates the flow model and prompt encoder.

\section{Preliminaries}
\label{sec:preliminaries}

\textbf{Tokens, encodings, and latents.}
Let $q=(q_1,\ldots,q_{L_q})$ and $a=(a_1,\ldots,a_{L_a})$ be prompt and answer tokens in vocabulary $\mathcal V$, with the answer including its terminal token. Their padded canvas $y=(q,a,\mathrm{padding})$ has length $L$; canvas index $i$ and answer index $j$ satisfy $y_{L_q+j}=a_j$. An answer encoder gives $z_{\mathrm{clean}}=\mathcal E(q,a)\in\R^{L_a\times d}$, and a prompt-only encoder gives $c_\phi(q)=P_\phi(q)\in\R^{L_q\times d}$; rows are token vectors. Unlike the teacher's tokenwise lookup $B_\psi[q_i]$, $P_\phi$ uses prompt context. Parameters $\psi$ are frozen teacher weights, $\phi$ the prompt encoder, and $\theta$ the shared denoising/decoding network, with $\Theta=(\theta,\phi)$. Teacher layer $\ell$ contributes width $d_\ell$ to total latent width $d$. Uppercase letters denote random variables.

\textbf{Flow matching.}
Flow matching transports noise to data \citep{flowmatching}. For noise scale $\sigma>0$, our linear Gaussian path runs from noise at $t=0$ to clean latents at $t=1$:
\begin{equation}
 z_t=t z_{\mathrm{clean}}+(1-t)\sigma\epsilon,
 \qquad \epsilon\sim\mathcal N(0,I),\qquad
 u_t:=\frac{\mathrm d z_t}{\mathrm d t}=z_{\mathrm{clean}}-\sigma\epsilon.
 \label{eq:methodinterpolant}
\end{equation}
Treating text $q$ as the external condition and $c_\phi(q)$ as its internal encoding, write $v_\Theta(z,t;q):=v_\theta(z,t;c_\phi(q))$. The conditional objective is
\begin{equation}
 \mathcal L_{\mathrm{FM}}^{\mathrm{basic}}(\Theta)
 =\E\!\left[\left\|v_\theta\bigl(z_t,t;c_\phi(q)\bigr)-u_t\right\|_F^2\right].
 \label{eq:basicflowloss}
\end{equation}
The expectation samples training pairs $(q,a)$, times $t\sim\pi(t)$, and independent Gaussian noise.

\textbf{Text and encoded conditioning.}
Write $p_t(z\mid q)$ for the forward-path density and $s^\star(z,t\mid q)=\nabla_z\log p_t(z\mid q)$ for its conditional score. For $0<t<1$, the population-optimal velocity is
\begin{equation}
 v^\star(z,t\mid q)
 =\E[u_t\mid z_t=z,q]
 =\frac{z+\sigma^2(1-t)s^\star(z,t\mid q)}{t}.
 \label{eq:conditionalflowscore}
\end{equation}
Conditioning on $c_\phi(q)$ may merge distinct prompts; Section~\ref{sec:conditionalencoding} characterizes when the encoding preserves the text-conditional score.

\textbf{Sampling under a different clock.}
At inference, let $\tau\in[0,1]$ be solver time and $t=f(\tau)$ an increasing differentiable clock with $f(0)=0$, $f(1)=1$. Starting from $\bar z_0\sim\mathcal N(0,\sigma^2I)$, sampling obeys $\mathrm d\bar z_\tau/\mathrm d\tau=f'(\tau)v_\Theta(\bar z_\tau,f(\tau);q)$. For $f(\tau)=\tau^2$, the multiplier is $2\tau$. Exact integration preserves the endpoint, while finite-step accuracy can change. Section~\ref{sec:clocks} extends this to group-specific clocks.

\section{Continuous Diffusion Reasoning}
\label{sec:methodology}
\label{sec:methodflow}

For a general, fixed answer representation $\mathcal E(q,a)$ and prompt encoding $c=c(q)$, CEDR follows ELF's denoising and token-decoding formulation~\citep{elf} (Figure~\ref{fig:pipelineoverview}). We use canvas length $L=1024$, latent width $d=1024$, and noise scale $\sigma=2$.

\begin{figure}[t]
    \centering
    \includegraphics[width=\textwidth]{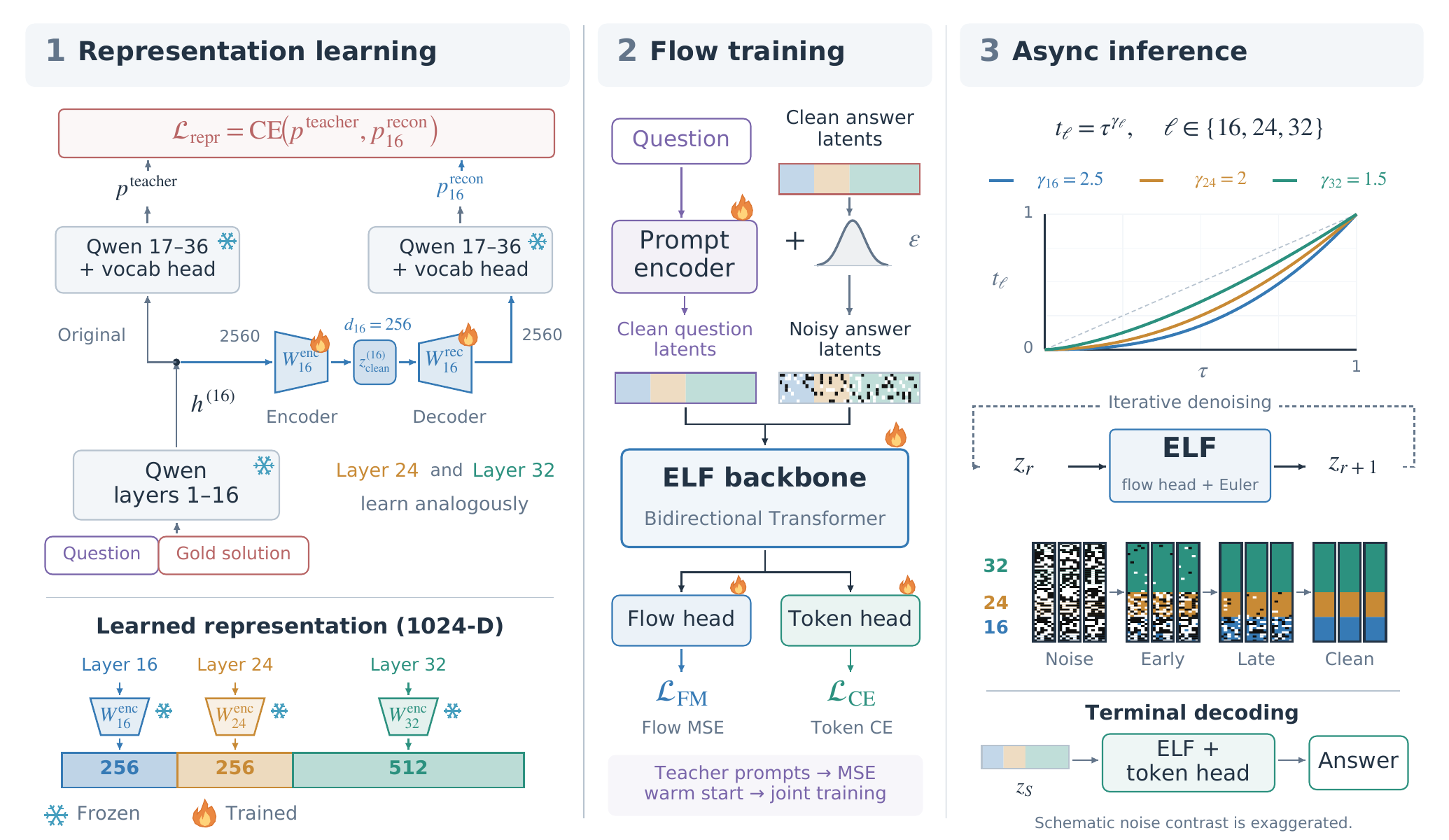}
    \caption{\textbf{The CEDR training and inference pipeline.}
    \textbf{Left:} Learn projections of Qwen layers 16/24/32 through teacher cross-entropy, then concatenate their outputs (Section~\ref{sec:representationconstruction}).
    \textbf{Center:} Train the ELF backbone with flow and token-decoding losses and staged prompt learning (Sections~\ref{sec:methodflow} and~\ref{sec:conditionalencoding}).
    \textbf{Right:} Denoise the representation groups under asynchronous clocks, then decode the answer in parallel (Section~\ref{sec:clocks}).}
    \label{fig:pipelineoverview}
\end{figure}

\textbf{Continuous denoising.}
ELF's bidirectional backbone predicts $\widehat z_{\mathrm{clean}}=F_\theta(z_t,t;c,\widehat z_{\mathrm{sc}},g)$ from the clean prompt and noisy answer at a shared training time $t$. Here $\widehat z_{\mathrm{sc}}$ is an earlier clean prediction and $g$ the guidance scale. We apply Eq.~\eqref{eq:basicflowloss} with velocity $v_\theta=(\widehat z_{\mathrm{clean}}-z_t)/\Delta_t$ and stabilized target $u_t^{\mathrm{stab}}=(z_{\mathrm{clean}}-z_t)/\Delta_t$, with $\Delta_t=\max(1-t,0.05)$; outside the clamp, $u_t^{\mathrm{stab}}=u_t$. Appendix~\ref{app:methodobjective} gives valid-token normalization.

\textbf{Self-conditioning and SCCFG.}
ELF's self-conditioning classifier-free guidance (SCCFG) learns to amplify answer self-conditioning at scale $g$. It compares predictions with and without \emph{answer self-conditioning}, retaining the prompt in both; Appendix~\ref{app:methodobjective} gives the detached guidance-adjusted target. Guidance is internal to the field, requiring one denoiser call per step.

\textbf{Recovering discrete tokens.}
A vocabulary head on the same backbone predicts same-position answer tokens under a separate decoder corruption $z^{\mathrm{dec}}$:
\begin{equation}
 \mathcal L_{\mathrm{CE}}
 =\E_{\mathrm{valid}}\!\left[-\log p_{\theta,j}^{\mathrm{tok}}(a_j\mid z^{\mathrm{dec}},c(q))\right].
 \label{eq:methoddecoderloss}
\end{equation}
The expectation averages valid answer tokens, including the terminal token; Appendix~\ref{app:methodobjective} specifies the exact batch weighting, decoder corruption, and mode mixing. At inference, encode the prompt once and keep it clean through $S$ Euler steps with recurrent self-conditioning; one additional backbone call decodes the answer in parallel and is excluded from denoising NFE.

\subsection{Learning the denoising representation}
\label{sec:representation}

The denoising representation must support both token recovery and generation from noise. We first define CEDR's learned multilayer representation, then analyze the evidence for combining teacher layers and learning how to compress them.

\subsubsection{Representation construction}
\label{sec:representationconstruction}

We extract \textbf{teacher-forced Qwen3-4B-Instruct-2507 activations} \citep{qwen2507} in one forward pass; although these features are causal, our bidirectional denoiser jointly generates answer representations.

\textbf{Latent construction.}
Let $h_j^{(\ell)}$ be the frozen teacher's post-block activation at answer position $j$. We concatenate centered linear projections as $z_{\mathrm{clean},j}=\big[(h_j^{(\ell)}-\mu_\ell)W_\ell^{\mathrm{enc}}\big]_{\ell\in\mathcal S}$ for $\mathcal S=\{16,24,32\}$, with widths $(d_{16},d_{24},d_{32})=(256,256,512)$ and fixed means $\mu_\ell$. These maps initially also encode prompts; answer representations stay fixed during flow training and NFT.

\textbf{Learning the projectors.}
Linear decoders reconstruct $\widehat h_j^{(\ell)}=\mu_\ell+z_{\mathrm{clean},j}^{(\ell)}W_\ell^{\mathrm{rec}}$ (Figure~\ref{fig:pipelineoverview}, left). Replacing layer-$\ell$ answer activations with these reconstructions and running the frozen Qwen suffix yields next-token predictions $p_{\ell,j}^{\mathrm{recon}}$. We preserve the teacher's predictions through $\mathcal L_{\mathrm{repr}}=\sum_{\ell\in\mathcal S}\E_{q,a,j}\,\CE(p_j^{\mathrm{teacher}},p_{\ell,j}^{\mathrm{recon}})$, where $\CE$ is full-vocabulary soft cross-entropy. Only $W_\ell^{\mathrm{enc}}$ and $W_\ell^{\mathrm{rec}}$ are learned; the reconstruction decoders are used only at this stage. Appendix~\ref{app:representationimplementation} gives initialization, whitening, and scoring details.

\subsubsection{What makes a good denoising representation?}
\label{sec:representationablations}

The construction in Section~\ref{sec:representationconstruction} is motivated by three ablations: whether accurate decoding suffices for generation, whether combining teacher layers helps, and how their features should be compressed. The representation comparisons in Figure~\ref{fig:representationcomparison} use matched six-epoch CEDR-B training on GSM8K, 1,024-dimensional whitened latents, and frozen Qwen prompt conditioning. Table~\ref{tab:representationcurvevalues} gives the four-seed endpoint results.

\begin{figure}[tbp]
\centering
\begin{minipage}[t]{.49\linewidth}
\vspace{0pt}
\centering
\includegraphics[width=\linewidth]{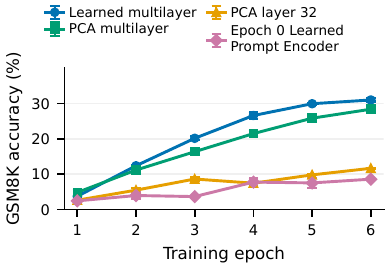}
\end{minipage}\hfill
\begin{minipage}[t]{.49\linewidth}
\vspace{0pt}
\centering
\includegraphics[width=\linewidth]{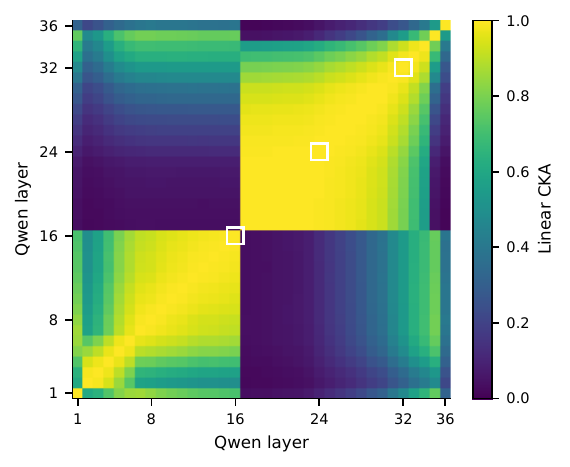}
\end{minipage}
\par
\begin{minipage}[t]{.49\linewidth}
\vspace{0pt}
\caption{\textbf{Representation and conditioning affect reasoning.} CEDR-B learning curves on GSM8K, with ODE sampling, 32 steps, and $t=\tau^2$. The first three variants retain Qwen conditioning; the fourth jointly trains a randomly initialized prompt encoder. The same epoch-six ordering holds under $t=\tau$ (Figure~\ref{fig:representationidentity}). Protocol: Appendix~\ref{app:representationclock}; values and SDs: Table~\ref{tab:representationcurvevalues} (Appendix~\ref{app:representationclock}).}
\label{fig:representationcomparison}
\end{minipage}\hfill
\begin{minipage}[t]{.49\linewidth}
\vspace{0pt}
\caption{\textbf{Qwen layer structure.} Linear CKA of teacher-forced activations on the combined discovery and confirmation panels. White markers select layers 16, 24, and 32. The transition after layer 16 motivates retaining that layer, followed by regularly spaced deeper features.}
\label{fig:representationcka}
\end{minipage}
\end{figure}

\textbf{Accurate decoding alone does not ensure strong reasoning.}
The top 1,024 PCA coordinates of Qwen's layer-32 activations support \textbf{99.24\%} held-out same-position token recovery with a decoder trained for two epochs (Appendix~\ref{app:representationdiagnostics}). Yet CEDR-B trained to generate this representation reaches only \textbf{11.20\%} GSM8K accuracy after six epochs. Recovering tokens from clean teacher features is therefore a different requirement from generating suitable features through denoising.

\textbf{Multilayer representations improve reasoning over final-layer features.}
Earlier-layer features may be easier to generate because they arise from fewer successive teacher transformations; they may also expose intermediate structures useful for generating later-layer features. To guide layer selection, we compute centered kernel alignment (CKA; \citealp{cka}) across Qwen's hidden layers (Figure~\ref{fig:representationcka}). The sharp transition after layer 16 motivates retaining that layer, followed by regularly spaced deeper features at layers 24 and 32 (Appendix~\ref{app:representationcka}). Keeping PCA and the total latent width fixed, combining these layers raises GSM8K accuracy from \textbf{11.20\% to 28.43\%}. Their decomposition also enables separate denoising clocks (Section~\ref{sec:clocks}).

\textbf{CE-trained projectors outperform PCA.}
PCA preserves high-variance directions, which need not be those most important for token prediction. Learning projections through the teacher's decoding cross-entropy (Section~\ref{sec:representationconstruction}) raises accuracy from \textbf{28.43\% to 30.76\%} with the same layers, latent widths, whitening, and flow-training protocol (Figure~\ref{fig:representationcomparison}). This supports predictive information preservation as a criterion for constructing diffusion latents.

\subsection{Learning the conditional encoding in three stages}
\label{sec:conditionalencoding}

Computing a useful denoising representation may require a large teacher, but inference needs that teacher only for prompt conditioning: answer latents start from Gaussian noise, whereas observed prompt tokens must still be encoded. Replacing the prompt encoder can therefore remove the remaining dependence on the teacher Transformer.

ELF's conditional-generation experiments use a pretrained T5-small encoder, reported as 35M parameters \citep{elf}, while ELF-REG uses Qwen3-0.6B-Base \citep{scalingdlm}. Our answer representations are extracted from the substantially larger Qwen3-4B teacher. Nevertheless, for CEDR-L we replace its contextual prompt encoder with a learned Transformer containing \textbf{121M parameters}, compared with 3.63B in the teacher's contextual network (both excluding vocabulary input/output maps). The student uses six bidirectional attention blocks and encodes each prompt once, retaining the teacher's frozen token-embedding table for prompt input. Appendix~\ref{app:training} gives the architecture.

\textbf{Preserving conditioning information does not require matching teacher prompt features.}
The answer representation defines the space in which diffusion generates solutions; the prompt encoding supplies the information needed to guide generation. We treat prompt text $q$ as the external condition and its encoding $P_\phi(q)$ as an internal component, $v_\Theta(z,t;q)=v_\theta(z,t;P_\phi(q))$, with $\Theta=(\theta,\phi)$. Unlike teacher-prompt imitation, this formulation lets the encoder and denoiser adapt together through the generative objective, keeping the answer representation and forward corruption fixed. Self-conditioning and guidance remain auxiliary inputs.

Lemma~\ref{lem:conditionalrepresentation} characterizes the information that must be preserved: an encoded prompt gives the same ideal text-conditional score exactly when it preserves the posterior mean clean answer.

\begin{lemma}[Preserving the text-conditional score]
\label{lem:conditionalrepresentation}
Let $Z_{\mathrm{clean}}$ and $Q$ denote the clean answer latents and prompt, with $\E\|Z_{\mathrm{clean}}\|_F^2<\infty$, and fix a deterministic encoder $C=P_\phi(Q)$. Let $s^\star(z,t\mid c)$ denote the score conditioned on $C=c$. Under the independent Gaussian corruption in Eq.~\eqref{eq:methodinterpolant}, for $0<t<1$ and $c=P_\phi(q)$,
\begingroup\small
\begin{equation}
 s^\star(z,t\mid q)-s^\star(z,t\mid c)
 =\frac{t}{\sigma^2(1-t)^2}
 \bigl(\E[Z_{\mathrm{clean}}\mid Z_t=z,Q=q]-\E[Z_{\mathrm{clean}}\mid Z_t=z,C=c]\bigr).
 \label{eq:conditionalscoredifference}
\end{equation}
\endgroup
Thus an encoder preserves the full-prompt score exactly when it preserves the posterior mean clean answer. If this holds along the path, the optimal unguided field also agrees by Eq.~\eqref{eq:conditionalflowscore}, giving the same conditional endpoint law under exact integration whenever the flow is well-defined.
\end{lemma}

\begin{figure}[tbp]
\centering
\includegraphics[width=\linewidth]{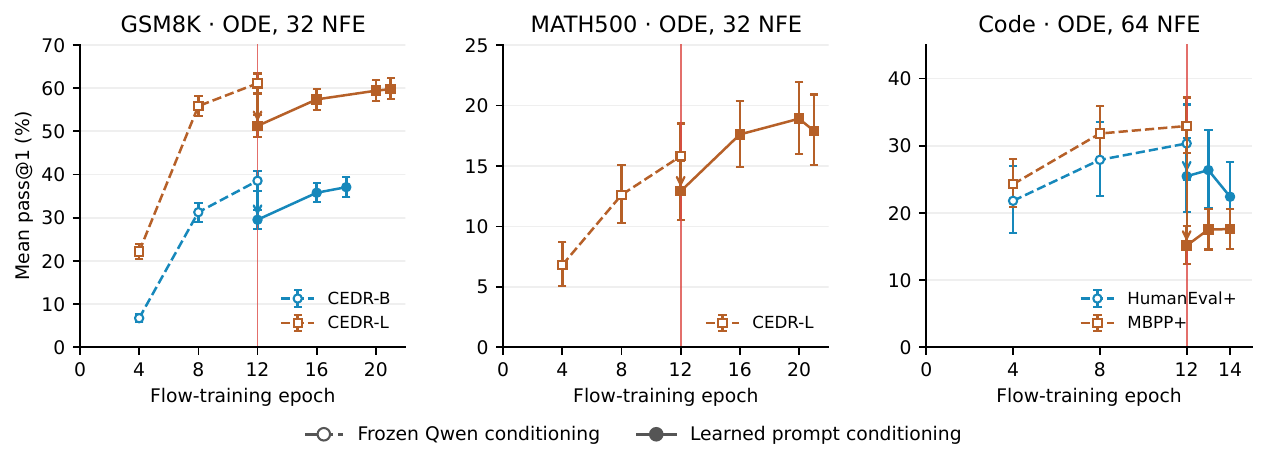}
\caption{\textbf{From teacher conditioning to learned prompts.} Left to right: GSM8K, MATH500, and code. All panels use ODE (Euler) and async $(2.5,2,1.5)$; math curves use 32 NFE and SCCFG 2, while code uses 64 NFE and SCCFG 3. Open/filled markers denote Qwen/learned prompt conditioning. The vertical line marks the epoch-12 switch; arrows show the immediate MSE-prompt substitution. Coding uses prompt-MSE epoch 20 followed by two joint epochs, and plots standard-scoring HumanEval+ and MBPP+. Numerical values and uncertainty: Appendices~\ref{app:promptcurves} and~\ref{app:codingdetails}.}
\label{fig:promptcurriculum}
\end{figure}

Encodings can preserve the required information while \emph{differing substantially in MSE}; Appendix~\ref{app:conditioninglemma} proves the lemma and gives an information-preserving encoding with arbitrarily large teacher-imitation MSE. This is a criterion for an ideal denoiser; learning a compact encoder that meets it remains an optimization problem.

\textbf{Early flow training benefits from fixed teacher conditioning.}
\label{sec:conditioningcurriculum}
Learning both networks from initialization changes the conditioning interface while the denoiser is still learning to use it. In the matched six-epoch comparison, this reaches only \textbf{8.70\%} GSM8K accuracy, versus \textbf{30.76\%} with frozen teacher conditioning (Figure~\ref{fig:representationcomparison}). We therefore learn the flow model first, then adapt a compact prompt encoder through three stages, keeping answer representations and token embeddings fixed:

\begin{enumerate}
    \item \textbf{Frozen conditioning.} Train the flow model with fixed Qwen prompt representations and the flow/decoder objective.
    \item \textbf{Prompt imitation.} Freeze the flow model and fit the compact encoder to teacher prompt representations using MSE, initializing it to the conditioning interface already in use.
    \item \textbf{Joint adaptation.} Substitute the learned encoder and optimize both networks through the flow/decoder objective, without an additional prompt-MSE loss.
\end{enumerate}

\textbf{MSE prompt imitation does not fully preserve generation quality.}
Figure~\ref{fig:promptcurriculum} shows a shared effect across mathematical reasoning and coding: substituting the MSE-trained prompt encoder reduces accuracy despite unchanged flow-model weights. Subsequent joint training recovers much of the loss on GSM8K, exceeds the teacher-conditioned baseline on MATH500, and partially recovers MBPP+, although HumanEval+ recovery is non-monotonic. This adaptation changes \emph{both} the prompt encoder and denoiser, allowing them to learn a compatible conditioning interface through the downstream objective. Tables~\ref{tab:promptcurvevalues} and~\ref{tab:codingstages} give the corresponding stage accuracies.

The prompt-swap loss is particularly severe on MBPP+, highlighting the greater conditioning demands of some coding tasks. Unique formatted training prompts average 233 tokens for code, versus 84 for GSM8K and 89 for MATH (Appendix~\ref{app:codingdetails}). Longer, more complex prompts are one possible explanation for the larger MBPP+ degradation, while HumanEval+ shows a smaller drop. The results suggest that compact conditioning requires task-dependent adaptation, even within a shared training curriculum.

Table~\ref{tab:trainingstages} gives the stage durations and reported NFT endpoints; Appendix~\ref{app:training} details training.

\begin{table}[!htb]
\centering\small
\setlength{\tabcolsep}{5pt}
\caption{Training stages by task. Flow-training epochs count passes over answer rows; prompt-MSE epochs count passes over unique prompts and do not advance the flow-training epoch. Joint epochs are additional to the first 12 flow epochs. The last column gives the NFT checkpoint used for reporting, in successful optimizer updates, rather than a further epoch count.}
\label{tab:trainingstages}
\begin{tabular}{@{}lcccc@{}}
\toprule
Task / model & \shortstack{Frozen-Qwen\\flow epochs} & \shortstack{Prompt-MSE\\epochs} & \shortstack{Joint epochs\\(total epochs)} & \shortstack{Reported\\NFT update}\\
\midrule
GSM8K / CEDR-B & 12 & 60 & 6 (18) & 300\\
GSM8K / CEDR-L & 12 & 30 & 9 (21) & 600\\
MATH / CEDR-L & 12 & 10 & 9 (21) & 600\\
OpenCodeInstruct / CEDR-L & 12 & 20 & 2 (14) & 100\\
\bottomrule
\end{tabular}
\end{table}

\subsection{Asynchronous inference with local clocks}
\label{sec:clocks}

Our multilayer representation allows different feature groups to become clean at different rates, potentially letting one group use cleaner information from another. This introduces two inference choices: the \emph{local clocks} specify each group's noise level along the trajectory, while the \emph{sampler} specifies how to advance the state and whether to inject fresh noise. We describe their common continuous-time formulation before comparing these choices separately.

\textbf{Asynchronous dynamics.}
Following the asynchronous flow construction of \citet{asyncschedules}, assign each layer group $\ell\in\{16,24,32\}$ an increasing local clock $t_\ell=f_\ell(\tau)$, with $f_\ell(0)=0$ and $f_\ell(1)=1$. At shared solver time $\tau$, the joint corruption is
\begin{equation}
 Z_{\bm t}^{(\ell)}=t_\ell Z_{\mathrm{clean}}^{(\ell)}
 +\sigma(1-t_\ell)\epsilon^{(\ell)},
 \qquad \bm t=(t_{16},t_{24},t_{32}),\quad \epsilon\sim\mathcal N(0,I).
 \label{eq:asynccorruption}
\end{equation}
Write $p_{\bm t}(z\mid q)$ for its joint density, $s_\ell^\star=\nabla_{z^{(\ell)}}\log p_{\bm t}(z\mid q)$, and $m_\ell^\star=\E[Z_{\mathrm{clean}}^{(\ell)}\mid Z_{\bm t}=z,q]$. For $0<t_\ell<1$, the local velocity and score satisfy
\begin{equation}
 v_\ell^\star=\frac{m_\ell^\star-z^{(\ell)}}{1-t_\ell},
 \qquad
 s_\ell^\star=\frac{t_\ell m_\ell^\star-z^{(\ell)}}{\sigma^2(1-t_\ell)^2}.
 \label{eq:asynclocalfields}
\end{equation}
Each field depends on the \emph{entire} noisy state, so the groups can inform one another despite having different clocks. Adding score-based drift and matching diffusion to the probability flow gives the asynchronous SDE family \citep{edm}
\begin{equation}
 \mathrm d Z_\tau^{(\ell)}
 =f'_\ell(\tau)\bigl[v_\ell^\star+\nu_\ell(\tau)s_\ell^\star\bigr]\,\mathrm d\tau
 +\sqrt{2f'_\ell(\tau)\nu_\ell(\tau)}\,\mathrm d W_\tau^{(\ell)},
 \qquad \nu_\ell(\tau)\geq0,
 \label{eq:asyncsde}
\end{equation}
where the fields are evaluated at $(Z_\tau,\bm t(\tau);q)$, $W^{(\ell)}$ are independent Brownian motions, and $\nu_\ell$ is a time-dependent diffusivity per unit local time. The clock derivative therefore scales both drift and noise variance. With exact fields and the usual regularity conditions, every such SDE follows the same scheduled marginals; $\nu_\ell=0$ recovers the asynchronous probability-flow ODE. Appendix~\ref{app:asyncsde} gives the derivation and relates the stochastic updates below to this formulation.

Our model is trained synchronously. At inference, we apply the same learned time-embedding function to each local time and average its outputs, $\bar e_\theta^{\mathrm{time}}(\bm t)=\frac13\sum_\ell e_\theta^{\mathrm{time}}(t_\ell)$. This needs no retraining and recovers the original embedding when local times agree. The marginal-equivalence result assumes exact fields; our sampler approximates the fields in Eq.~\eqref{eq:asynclocalfields} using guided predictions from a synchronously trained model.

\textbf{Sampler choice.}
Let $z_r$ be the answer state at $\tau_r$, $t_{\ell,r}=f_\ell(\tau_r)$, and $\widehat z_{\mathrm{clean},r}$ the guided clean prediction. We consider three update rules, retaining the same destination grid and local clocks.

\emph{ODE.} Setting $\nu_\ell=0$ in Eq.~\eqref{eq:asyncsde} and discretizing in the shared solver clock gives our Euler baseline:
\begin{equation}
 z_{r+1}^{(\ell)}=z_r^{(\ell)}+
 (\tau_{r+1}-\tau_r)\,f'_\ell(\tau_r)
 \frac{\widehat z_{\mathrm{clean},r}^{(\ell)}-z_r^{(\ell)}}
 {\Delta_{t_{\ell,r}}},
 \label{eq:asyncsampler}
\end{equation}
where $\Delta_t=\max(1-t,0.05)$ is the practical denominator stabilization from Section~\ref{sec:methodflow}. No noise is added after initialization.

\emph{Stochastic (DMD2).} Following consistency models and DMD2 \citep{consistencymodels,dmd2}, take one Euler step all the way to the clean endpoint, then apply forward Gaussian noising at the next local times. For the unstabilized velocity in Eq.~\eqref{eq:asynclocalfields}, a step of size $1-t_{\ell,r}$ reaches $\widehat z_{\mathrm{clean},r}^{(\ell)}$ directly. The combined update is
\begin{equation}
 z_{r+1}^{(\ell)}
 =t_{\ell,r+1}\widehat z_{\mathrm{clean},r}^{(\ell)}
 +\sigma(1-t_{\ell,r+1})\xi_r^{(\ell)},
 \qquad \xi_r\sim\mathcal N(0,I).
 \label{eq:asyncdmd2}
\end{equation}
This finite denoise--diffuse composition refreshes the noise at each nonterminal step; its relation to the SDE is discussed in Appendix~\ref{app:asyncsde}. We use the sampling construction without DMD2 distillation; ELF-REG also evaluates this construction for continuous language generation \citep{scalingdlm}.

\emph{Stochastic (EDM).} EDM's noise-churn principle instead splits a step into partial forward noising and ODE denoising \citep{edm}. First move to noisier local times $\bar t_{\ell,r}\leq t_{\ell,r}$. For an interior step, set $\alpha_{\ell,r}=\bar t_{\ell,r}/t_{\ell,r}$ and draw
\begin{equation}
 \bar z_r^{(\ell)}
 =\alpha_{\ell,r}z_r^{(\ell)}
 +\sigma\sqrt{(1-\bar t_{\ell,r})^2
 -\alpha_{\ell,r}^2(1-t_{\ell,r})^2}\,\xi_r^{(\ell)}.
 \label{eq:asyncchurnback}
\end{equation}
This is an exact forward-noising transition for the corruption in Eq.~\eqref{eq:asynccorruption}. Evaluate the denoiser at $(\bar z_r,\bar{\bm t}_r)$ and advance each local probability-flow ODE with one Euler step:
\begin{equation}
 z_{r+1}^{(\ell)}
 =\bar z_r^{(\ell)}
 +\frac{t_{\ell,r+1}-\bar t_{\ell,r}}{1-\bar t_{\ell,r}}
 \bigl(\widehat z_{\mathrm{clean},r}^{(\ell)}-\bar z_r^{(\ell)}\bigr).
 \label{eq:asyncchurnforward}
\end{equation}
For small churn and step sizes, this splitting recovers a member of Eq.~\eqref{eq:asyncsde}. Appendix~\ref{app:samplerimplementation} specifies our finite index-churn implementation, including local-time Euler increments and endpoint handling.

All three samplers restore clean prompt positions and carry the guided clean prediction as recurrent self-conditioning. Each step uses one shared denoiser call for all groups, so asynchrony does not increase NFE; token decoding requires one additional call.

\textbf{The preferred sampler depends on the budget and answer-selection rule.}
DMD2-style sampling gives the strongest pass@1 under our headline settings. ODE remains competitive at low NFE and can give higher oracle pass@$k$ with more attempts (Figures~\ref{fig:samplerallocationgsml} and~\ref{fig:samplerallocationmathl}; Appendix~\ref{app:samplerallocations}).

\textbf{Asynchronous clocks can improve generation accuracy.}
We use power clocks $f_\ell(\tau)=\tau^{\gamma_\ell}$: equal exponents give synchronous denoising, while unequal exponents give different group noise levels. Our headline clock $(\gamma_{16},\gamma_{24},\gamma_{32})=(2.5,2,1.5)$ has the highest observed ODE validation mean on both datasets (Appendix~\ref{app:schedulecomparisons}). At 32 NFE, nonidentity ODE clocks improve on $(1,1,1)$, with smaller differences among them, whereas both stochastic samplers favor $(2.5,2,1.5)$ on both test sets (Table~\ref{tab:inferenceclocks}). Under DMD2-style sampling, this clock improves on $(2,2,2)$ by 1.82 percentage points on GSM8K and 4.00 on MATH500. We retain this clock across headline samplers and checkpoints.

\begin{table}[!htbp]
\centering\small
\setlength{\tabcolsep}{3pt}
\caption{\textbf{Inference-clock ablation across samplers.} CEDR-L epoch 21, 32 denoising steps, with headline guidance (SCCFG 2 for GSM8K, 3 for MATH500); mean pass@1 (\%) with seed SD over seeds 42/123. Bold marks the best clock within each dataset/sampler. EDM-style sampling uses index churn $0.7$. Validation results, individual seeds, and paired comparisons appear in Appendix~\ref{app:schedulecomparisons}.}
\label{tab:inferenceclocks}
\begin{tabular}{@{}lrrrrrr@{}}
\toprule
& \multicolumn{2}{c}{ODE} & \multicolumn{2}{c}{Stochastic (EDM)} & \multicolumn{2}{c}{Stochastic (DMD2)} \\
\cmidrule(lr){2-3}\cmidrule(lr){4-5}\cmidrule(lr){6-7}
Clock powers & GSM8K & MATH500 & GSM8K & MATH500 & GSM8K & MATH500 \\
\midrule
$(1,1,1)$ & 52.62 {\scriptsize($\pm 0.75$)} & 16.70 {\scriptsize($\pm 0.42$)} & 59.78 {\scriptsize($\pm 0.05$)} & 19.60 {\scriptsize($\pm 1.70$)} & 60.50 {\scriptsize($\pm 0.43$)} & 21.10 {\scriptsize($\pm 2.69$)} \\
$(2,2,2)$ & 59.97 {\scriptsize($\pm 0.64$)} & 18.80 {\scriptsize($\pm 0.85$)} & 61.83 {\scriptsize($\pm 0.91$)} & 16.80 {\scriptsize($\pm 2.26$)} & 64.86 {\scriptsize($\pm 0.16$)} & 23.90 {\scriptsize($\pm 2.40$)} \\
$(1.5,2,2.5)$ & \textbf{60.08} {\scriptsize($\pm 0.38$)} & 18.50 {\scriptsize($\pm 0.14$)} & 63.42 {\scriptsize($\pm 0.59$)} & 19.90 {\scriptsize($\pm 0.99$)} & 66.49 {\scriptsize($\pm 0.86$)} & 25.00 {\scriptsize($\pm 2.55$)} \\
$(2.5,2,1.5)$ & 59.78 {\scriptsize($\pm 0.16$)} & \textbf{19.70} {\scriptsize($\pm 0.42$)} & \textbf{64.44} {\scriptsize($\pm 0.32$)} & \textbf{21.60} {\scriptsize($\pm 1.70$)} & \textbf{66.68} {\scriptsize($\pm 0.05$)} & \textbf{27.90} {\scriptsize($\pm 2.12$)} \\
\bottomrule
\end{tabular}
\end{table}

Reconstruction probes show greater CE degradation when perturbing layer-16 features than layer-24 or layer-32 features (Figure~\ref{fig:layerreconstruction}). These findings show distinct reconstruction sensitivities across the representation groups.


\section{Reinforcement learning with self-conditioned flows}
\label{sec:nft}

DiffusionNFT \citep{diffusionnft} improves continuous generators through rewarded endpoint regression, without differentiating through sampling trajectories. We adapt it to CEDR's learned self-conditioning guidance and supplement sparse correctness rewards with gold-solution endpoints.

\subsection{NFT with learned self-conditioning guidance}
\label{sec:nftsccfg}

\textbf{The velocity field optimized by standard NFT.}
Let $v_{\mathrm{cur}}$ and $v_{\mathrm{old}}$ denote the current and old policy velocity fields, $u$ a sample's forward velocity target, and $\rho\in[0,1]$ its normalized optimality coefficient. NFT forms implicit positive and negative velocity fields, $v^+=(1-\beta)v_{\mathrm{old}}+\beta v_{\mathrm{cur}}$ and $v^-=(1+\beta)v_{\mathrm{old}}-\beta v_{\mathrm{cur}}$. For current and old velocity-field values $f,h$, the pointwise loss is
\begin{equation}
 \ell_\beta(f,h;u,\rho)
 =\rho\|(1-\beta)h+\beta f-u\|_2^2
 +(1-\rho)\|(1+\beta)h-\beta f-u\|_2^2.
\label{eq:nftquadratic}
\end{equation}
Standard NFT applies this loss to $f=v_{\mathrm{cur}}$ and $h=v_{\mathrm{old}}$, holding the old field fixed; its original experiments use CFG-free optimization \citep{diffusionnft}. CEDR instead learns a guidance-conditioned field, motivating a correction before applying NFT's regression objective.

\textbf{Guided rollouts with guidance-corrected optimization.}
For policy $p\in\{\mathrm{cur},\mathrm{old},\mathrm{ref}\}$, forward-noise an endpoint and obtain velocity $v_p^0$ from a detached bootstrap without answer self-conditioning. The main prediction $v_p^{\mathrm{main}}$ uses the bootstrap clean endpoint when $b_{\mathrm{sc}}=1$ and zero answer self-conditioning otherwise; both passes retain the question. Following ELF's guidance construction, define
\begin{equation}
 \widetilde v_p=v_p^{\mathrm{main}}-
 \sg\!\left[b_{\mathrm{sc}}(1-g^{-1})(v_p^{\mathrm{main}}-v_p^0)\right],
 \label{eq:nftcorrectedfield}
\end{equation}
where $g>0$ is the SCCFG scale. Our update applies Eq.~\eqref{eq:nftquadratic} to $f=\widetilde v_{\mathrm{cur}}$, $h=\sg(\widetilde v_{\mathrm{old}})$, and $u=u_t^{\mathrm{stab}}$ from Section~\ref{sec:methodflow}. We use $\beta=1$ and add a squared-velocity penalty toward the corrected field $\widetilde v_{\mathrm{ref}}$ of the fixed initial reference. With self-conditioning enabled, the corrected value is $v_p^0+(v_p^{\mathrm{main}}-v_p^0)/g$. Detaching the correction preserves the main-pass Jacobian rather than scaling it by $1/g$ (Appendix~\ref{app:nftdetails}).

Collection retains recurrent self-conditioning and SCCFG guidance. We reward decoded answers but optimize their saved continuous endpoints with the corrected fields, using fresh forward corruption and bootstrap predictions rather than replaying rollout histories.

\subsection{Gold solutions supplement rewards and anchor the velocity field}
\label{sec:nftgold}

Each question contributes 15 old-policy generations and one gold latent endpoint, with raw rewards 0.75 for correct generated answers, 0 for incorrect answers, and 1 for gold. The gold endpoint supplies a preferred target even when all generated answers are incorrect. We discard groups whose generated answers are all correct, concentrating updates on observed failures; retained rewards are normalized into $\rho$ as described in Appendix~\ref{app:nftdetails}. This recipe uses reference solutions as well as verifiable rewards. Gold raw reward one need not give $\rho_G=1$; the following lemma characterizes its contribution.

\begin{lemma}[Gold endpoints as field anchors]
\label{lem:goldanchor}
For a gold endpoint, let $\mathcal H$ contain the noisy state, time, question, coefficient $\rho_G\in[1/2,1]$, and any auxiliary inputs held fixed for the update. Let $U$ denote the random stabilized target $u_t^{\mathrm{stab}}$, assume its conditional second moment is finite, and set $m_G=\E[U\mid\mathcal H,\mathrm{gold}]$. For $\mathcal H$-measurable current and old velocity-field values $f,h$, define $a_G=(2\rho_G-1)m_G+2(1-\rho_G)h$. The beta-one NFT loss satisfies
\begin{equation}
 \E[\ell_1(f,h;U,\rho_G)\mid\mathcal H,\mathrm{gold}]
 =\|f-a_G\|_2^2+C_{\mathcal H},
\label{eq:nftgoldanchor}
\end{equation}
where $C_{\mathcal H}$ is independent of $f$ with the conditioning information and auxiliary fields fixed. Thus the gold term attracts the current field toward a convex combination of the gold posterior target and the old field.
\end{lemma}

At $\rho_G=1$, the gold term directly regresses toward $m_G$. Under an ideal reference model matching that gold posterior target, the anchor interpolates between the reference and old fields; Appendix~\ref{app:nftgoldproof} gives the proof and scope. This gold-induced anchor complements the explicit reference penalty.

Because the external condition is the question text, NFT jointly updates the prompt encoder and flow model as one conditional policy (Lemma~\ref{lem:conditionalrepresentation}), without a prompt-imitation objective. Training uses synchronous identity clocks; asynchronous evaluation follows Section~\ref{sec:clocks}. Appendix~\ref{app:training} gives optimization details, and Sections~\ref{sec:nftresults} and~\ref{sec:codingresults} discuss the empirical gains.

\section{Empirical Evaluation}
\label{sec:headlineresults}

We evaluate mathematical reasoning on GSM8K and MATH500 \citep{gsm8k,math,math500}, and code generation in Section~\ref{sec:codingresults}. Math evaluations use zero-shot prompts, async clocks $(2.5,2,1.5)$, numerical-answer scoring for GSM8K, and Math-Verify for MATH500 \citep{mathverify}. All CEDR headline results use the same DMD2-style stochastic sampler (Eq.~\eqref{eq:asyncdmd2}) across checkpoints, datasets, and NFEs. Math headline pass@1 averages eight sampling seeds at 64 denoising steps. Representation and prompt-learning ablations retain ODE sampling; the schedule ablation compares all three samplers. NFE counts denoising calls, excluding prompt encoding and terminal decoding. Appendices~\ref{app:experimentdetails} and~\ref{app:seedvariability} give checkpoint settings and seed standard deviations.

\subsection{Mathematical reasoning}
\label{sec:mathresults}
\begin{table}[!htb]
\caption{Mathematical reasoning across model families, mean pass@1 (\%). Small brackets give available sampling-seed SDs (eight seeds for CEDR; four for ELF-REG). Parameters are in millions (three significant digits). Embedding Params count vocabulary input/output maps, tied weights once; Other Params count remaining weights, with $(+X)$ used once per response. Frozen components are included. Bold/underline mark first/second within each sub-billion continuous-backbone group.}
\label{tab:headlinecomparison}
\centering\small
\setlength{\tabcolsep}{3.5pt}
\begin{tabular}{@{}lrrrrr@{}}
\toprule
Model & \shortstack{Embedding\\Params (M)} & \shortstack{Other\\Params (M)} & NFE & GSM8K & MATH500\\
\midrule
\multicolumn{6}{@{}l}{\textit{Autoregressive}}\\
Qwen3-0.6B-Base \citep{qwen3} & 156 & 440 & --- & 59.59 & ---\\
Qwen3-0.6B, thinking \citep{qwen3} & 156 & 440 & --- & --- & 77.60\\
\midrule
\multicolumn{6}{@{}l}{\textit{Discrete diffusion / blockwise diffusion}}\\
LLaDA-8B-Base \citep{llada} & 1,040 & 6,980 & 1,024 & 70.30 & ---\\
Dream-v0-Base-7B \citep{dream} & 1,090 & 6,530 & 256 & 77.20 & ---\\
SDAR-1.7B-Chat \citep{sdar} & 622 & 1,410 & 4/block & 80.10 & 63.20\\
\midrule
\multicolumn{6}{@{}l}{\textit{Continuous diffusion / flow: denoising backbone below 200M}}\\
\shortstack[l]{$\mathbb S$-FLM (top-1)\\\citep{hyperspherical}} & 75.5 & 87.5 & 1,024 & 18.00$^\dagger$ & ---\\
FMLM+ Init \citep{posterior} & 75.5 & 92 & 64 & 33.60$^\dagger$ & ---\\
ELF-REG-B (stoch.) \citep{scalingdlm} & 311 & 104 $(+441)$ & 63 & 42.97 {\scriptsize [0.35]} & ---\\
\textbf{CEDR-B, pre-NFT (ours)} & 545 & 90.4 $(+45.8)$ & 64 & \textbf{47.62} {\scriptsize [0.90]} & ---\\
\textbf{CEDR-B, post-NFT (ours)} & 545 & 90.4 $(+45.8)$ & 64 & \underline{47.00} {\scriptsize [0.98]} & ---\\
\midrule
\multicolumn{6}{@{}l}{\textit{Continuous diffusion / flow: denoising backbone 200M--1B}}\\
ELF-REG-L (stoch.) \citep{scalingdlm} & 311 & 652 $(+442)$ & 63 / 127 & 61.28 {\scriptsize [0.42]} & 20.15 {\scriptsize [1.29]}\\
\textbf{CEDR-L, pre-NFT (ours)} & 545 & 638 $(+123)$ & 64 & \underline{67.13} {\scriptsize [0.58]} & \underline{28.58} {\scriptsize [0.90]}\\
\textbf{CEDR-L, post-NFT (ours)} & 545 & 638 $(+123)$ & 64 & \textbf{68.97} {\scriptsize [0.50]} & \textbf{32.40} {\scriptsize [1.17]}\\
\midrule
\multicolumn{6}{@{}l}{\textit{Continuous diffusion / flow: larger backbones (reference)}}\\
MLFM \citep{mlfm} & 123 & 1,220 & 256$^\S$ & 31.24 & ---\\
TESS 2, GSM8K FT \citep{tess2} & 262 & 6,740 & 1,000 & $\approx68.9^\ddagger$ & ---\\
\bottomrule
\end{tabular}
\vspace{3pt}
\parbox{\linewidth}{\footnotesize External results are author-reported. ELF-REG uses its best reported stochastic settings (v3, Tables 4 and 30); CEDR uses DMD2-style sampling throughout. Table~\ref{tab:elfregselected} records the selected ELF-REG settings and seed SDs. $^\dagger$Python scoring, $^\ddagger$Approximate accuracy read from TESS 2 Figure 3b. $^\S$MLFM reports 256 sampling steps; its exact NFE with guidance and online token promotion is unclear. Slashes separate GSM8K/MATH500 budgets. Appendix~\ref{app:experimentdetails} details prompting and parameter accounting; Table~\ref{tab:headlineseedsd} in Appendix~\ref{app:seedvariability} gives our seed counts and SDs across denoising budgets.}
\end{table}

\textbf{Strong reasoning performance before reinforcement learning.}
\label{sec:headlinecomparison}
\label{sec:headlinelownfe}
At comparable denoising-backbone sizes, both supervised and post-NFT CEDR outperform the reported continuous-diffusion baselines on GSM8K and MATH500 (Table~\ref{tab:headlinecomparison}). This advantage is already present before reinforcement learning and is achieved with compact learned prompt encoders, without running the teacher Transformer at inference.

\begin{table}[!htb]
\caption{Reasoning at lower NFE, mean pass@1 (\%). Small brackets beneath CEDR accuracies give sampling-seed SDs. Table~\ref{tab:headlinecomparison} gives model parameter counts. Bold/underline mark first/second within each benchmark and backbone-scale block; $^{*}$ marks one fewer denoising call than the heading. External results follow \citet{posterior,scalingdlm}; Appendix~\ref{app:experimentdetails} gives source and protocol details.}
\label{tab:headlinelownfe}
\centering\small
\setlength{\tabcolsep}{1.7pt}
\begin{tabular}{@{}l*{12}{r}@{}}
\toprule
& \multicolumn{4}{c}{GSM8K ($\sim$100M)} & \multicolumn{4}{c}{GSM8K ($\sim$650M)} & \multicolumn{4}{c}{MATH500 ($\sim$650M)}\\
\cmidrule(lr){2-5}\cmidrule(lr){6-9}\cmidrule(lr){10-13}
Method / NFE & 8 & 16 & 32 & 64 & 8 & 16 & 32 & 64 & 8 & 16 & 32 & 64\\
\midrule
FMLM+ Init & 15.10 & 26.10 & 31.80 & 33.60 & --- & --- & --- & --- & --- & --- & --- & ---\\
ELF-REG (best) & $^{*}$\underline{26.88} & $^{*}$\underline{38.53} & $^{*}$41.17 & $^{*}$42.97 & $^{*}$48.69 & $^{*}$56.90 & $^{*}$59.95 & $^{*}$61.28 & $^{*}$9.85$^{\dagger}$ & $^{*}$14.00 & $^{*}$17.25 & $^{*}$20.00\\
CEDR, pre-NFT & 24.95 & 38.25 & \underline{44.39} & \textbf{47.62} & \underline{54.23} & \underline{62.57} & \underline{65.71} & \underline{67.13} & \underline{16.70} & \underline{21.75} & \underline{25.76} & \underline{28.58}\\
 & {\scriptsize [0.68]} & {\scriptsize [0.86]} & {\scriptsize [0.67]} & {\scriptsize [0.90]} & {\scriptsize [0.94]} & {\scriptsize [0.70]} & {\scriptsize [0.69]} & {\scriptsize [0.58]} & {\scriptsize [1.23]} & {\scriptsize [1.48]} & {\scriptsize [1.17]} & {\scriptsize [0.90]}\\
\addlinespace[2pt]
CEDR, post-NFT & \textbf{28.48} & \textbf{40.60} & \textbf{45.59} & \underline{47.00} & \textbf{57.93} & \textbf{64.96} & \textbf{67.59} & \textbf{68.97} & \textbf{18.82} & \textbf{25.26} & \textbf{29.80} & \textbf{32.40}\\
 & {\scriptsize [0.79]} & {\scriptsize [0.81]} & {\scriptsize [0.64]} & {\scriptsize [0.98]} & {\scriptsize [0.84]} & {\scriptsize [0.78]} & {\scriptsize [0.57]} & {\scriptsize [0.50]} & {\scriptsize [1.20]} & {\scriptsize [1.32]} & {\scriptsize [0.56]} & {\scriptsize [1.17]}\\
\bottomrule
\end{tabular}
\vspace{3pt}
\parbox{\linewidth}{\footnotesize CEDR uses DMD2-style sampling and fixed checkpoints, with SCCFG 2 for GSM8K and SCCFG 3 for MATH500. ELF-REG reports four-seed means: stochastic sampling except $^\dagger$ODE; settings can differ across budgets (Table~\ref{tab:elfregselected}). Our sampling-seed counts and SDs: Table~\ref{tab:headlineseedsd} in Appendix~\ref{app:seedvariability}.}
\end{table}

CEDR-L's advantage persists across all displayed denoising budgets (Table~\ref{tab:headlinelownfe}). For CEDR-B, the supervised model leads at 32 and 64 calls, while NFT extends the advantage to the smallest budgets. The same checkpoints support all these operating points, without separate few-step distillation.

\textbf{Oracle selection favors more attempts; voting favors deeper denoising.}
\label{sec:headlinebudget}
For $k$ samples with $S$ denoising steps each, we use $kS$ as the inference budget. Oracle pass@$k$ measures whether any sample is correct; majority vote selects the most frequent final answer after grouping equivalent answers, without consulting the reference. For selectors returning one of the generated answers, voting provides an achievable \emph{lower bound} on optimal selection accuracy, while oracle pass@$k$ supplies its \emph{upper bound}. We compare $S\in\{8,16,32,64\}$ up to a budget of 512 calls; Appendix~\ref{app:budgetvalues} gives aggregation details.

\begin{figure}[!htbp]
\centering
\includegraphics[width=\linewidth]{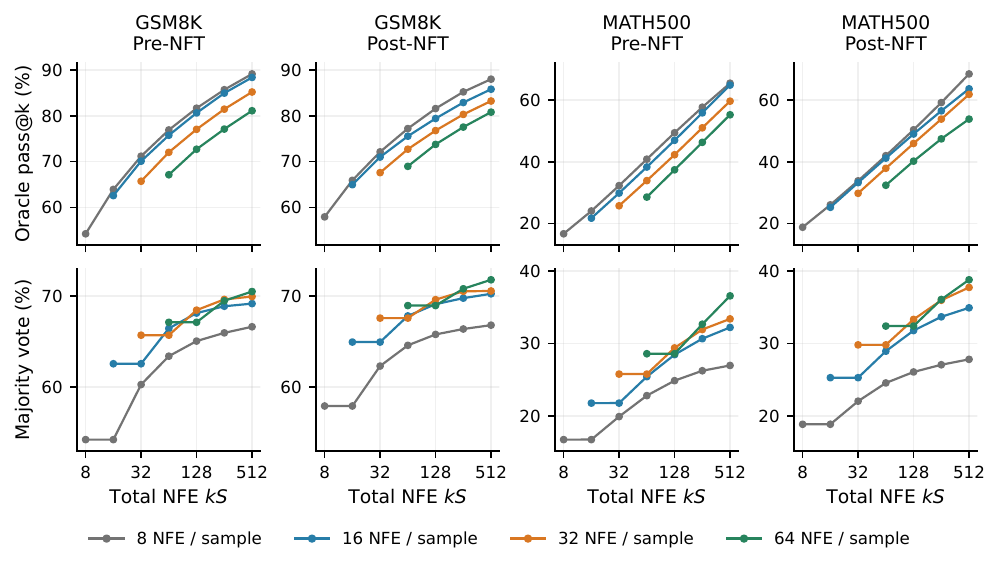}
\caption{CEDR-L before and after NFT with DMD2-style sampling: GSM8K uses SCCFG 2 and MATH500 uses SCCFG 3. Top: oracle pass@$k$; bottom: majority-vote accuracy. Colors denote denoising NFE per sample; the budget is $kS$. Numerical values appear in Appendix~\ref{app:budgetvalues}; Appendix~\ref{app:elfbbudgets} gives the corresponding CEDR-B curves.}
\label{fig:headlinestages}
\end{figure}

The contrast in Figure~\ref{fig:headlinestages} is clear on MATH500 after NFT: 64 eight-step samples reach 68.40\% oracle accuracy, versus 53.80\% for eight 64-step samples; voting reverses this preference, at 27.80\% versus 38.80\%. Intermediate depths remain competitive for voting. The voting--oracle gap motivates better answer selection, potentially using latent-space or hidden-state classifiers. Appendices~\ref{app:elfbbudgets} and~\ref{app:samplerallocations} give the CEDR-B and cross-sampler comparisons.

\textbf{NFT improves single-sample accuracy more consistently than oracle coverage.}
\label{sec:nftresults}
Under DMD2-style sampling, NFT improves CEDR-L pass@1 at all four denoising depths; CEDR-B improves at 8--32 calls but declines slightly at 64. Higher-pass effects depend on the allocation. With eight 64-step samples, MATH500 oracle accuracy falls from 55.20\% to 53.80\%, while voting improves from 36.58\% to 38.80\%; at eight steps with 64 samples, oracle accuracy instead increases. Voting benefits from a higher frequency of the correct answer relative to alternatives, whereas oracle selection needs only one successful draw. Appendix~\ref{app:budgetvalues} gives paired uncertainty, and Appendices~\ref{app:generations}--\ref{app:trajectories} illustrate ODE generations.

\paragraph{Higher-pass comparison with continuous-diffusion baselines.}
With eight DMD2-style samples at 64 denoising calls each, both supervised and post-NFT CEDR-L exceed ELF-REG-L's reported ODE pass@10 on MATH500 (Table~\ref{tab:mathhigherpass}; \citealp[v3, Table 20]{scalingdlm}).

\begin{table}[!htbp]
\centering\small
\setlength{\tabcolsep}{8pt}
\caption{\textbf{MATH500 oracle accuracy at higher sample counts.} CEDR-L uses learned prompt conditioning, SCCFG 3, and eight saved draws. ELF-REG-L retains its external prompt encoder. CEDR uses DMD2-style sampling and ELF-REG uses ODE; ELF-REG's stochastic study reports only four draws and cannot supply pass@8/10. NFE counts denoising calls per sample; $k\times\mathrm{NFE}$ makes the differing total budgets explicit. These are cross-paper reported results, with protocol details in Appendix~\ref{app:experimentdetails}. Small brackets give CEDR's problem-bootstrap SDs (standard errors; Appendix~\ref{app:uncertainty}). Bold/underline mark the best/second-best displayed accuracy.}
\label{tab:mathhigherpass}
\begin{tabular}{@{}lrrrr@{}}
\toprule
Model & NFE & $k$ & $k\times\mathrm{NFE}$ & Pass@$k$ (\%)\\
\midrule
ELF-REG-L \citep{scalingdlm} & 63 & 8 & 504 & 40.41\\
ELF-REG-L \citep{scalingdlm} & 63 & 10 & 630 & 44.04\\
CEDR-L, pre-NFT & 64 & 8 & 512 & \textbf{55.20} {\scriptsize [2.24]}\\
CEDR-L, post-NFT & 64 & 8 & 512 & \underline{53.80} {\scriptsize [2.24]}\\
\bottomrule
\end{tabular}
\end{table}

\FloatBarrier
\subsection{Code generation}
\label{sec:codingresults}

CEDR-L follows the staged curriculum on OpenCodeInstruct \citep{opencodeinstruct}, then uses execution-based rewards for NFT (Table~\ref{tab:trainingstages}; Appendix~\ref{app:training}). Our results in Table~\ref{tab:codingcomparison} use DMD2-style sampling, ten seeds, 128 denoising calls, async $(2.5,2,1.5)$, and SCCFG 3. Figure~\ref{fig:promptcurriculum} and Section~\ref{sec:conditionalencoding} analyze prompt adaptation separately using ODE at 64 NFE.

\begin{table}[htbp]
\centering\small
\setlength{\tabcolsep}{3pt}
\caption{\textbf{Code generation under separate scoring conventions.} Mean pass@1 (\%) with generation-seed SD in small brackets where available. Bold/underline mark the best/second-best mean in each block. HE/HE+ use base/full extended tests; both MBPP columns use 378 tasks. CEDR uses DMD2-style sampling with ten seeds; pre-NFT is joint epoch 14 and post-NFT is update 100. ELF-REG uses ODE with 16 seeds in the standard block and its stronger stochastic sampler with four seeds in the alias block (v3, Tables 15 and 4). Budgets exclude terminal decoding. Appendix~\ref{app:codingdetails} specifies the scoring protocols.}
\label{tab:codingcomparison}
\begin{tabular}{@{}lrrrrr@{}}
\toprule
Model / setting & \shortstack{Denoising\\NFE} & HE & HE+ & \shortstack{MBPP-378\\base} & \shortstack{MBPP+\\full}\\
\midrule
\multicolumn{6}{@{}l}{\textit{Standard scoring: no function-name repair}}\\
PlaidQ \citep{plaidq} & 128 & 19.02 & 18.05 & 22.18 & ---\\
PlaidQ+CFG \citep{plaidq} & 256 & 22.04 & 19.70 & 24.58 & ---\\
ELF-REG-L (ODE) \citep{scalingdlm} & 127 & 22.26 {\scriptsize [2.30]} & 21.15 {\scriptsize [2.13]} & 24.21 {\scriptsize [1.17]} & ---\\
CEDR-L, pre-NFT (ours) & 128 & \underline{39.51} {\scriptsize [3.20]} & \underline{36.10} {\scriptsize [2.60]} & \underline{30.32} {\scriptsize [1.26]} & \underline{26.69} {\scriptsize [0.95]}\\
CEDR-L, post-NFT (ours) & 128 & \textbf{40.98} {\scriptsize [1.81]} & \textbf{37.56} {\scriptsize [2.04]} & \textbf{30.74} {\scriptsize [1.18]} & \textbf{26.96} {\scriptsize [1.51]}\\
\midrule
\multicolumn{6}{@{}l}{\textit{Single-function name aliasing}}\\
ELF-REG-L (stoch.) \citep{scalingdlm} & 127 & 33.08 {\scriptsize [3.16]} & 31.40 {\scriptsize [3.61]} & 38.76 {\scriptsize [1.09]} & ---\\
CEDR-L, pre-NFT (ours) & 128 & \underline{39.57} {\scriptsize [3.22]} & \underline{36.16} {\scriptsize [2.64]} & \underline{44.89} {\scriptsize [1.69]} & \underline{38.62} {\scriptsize [1.33]}\\
CEDR-L, post-NFT (ours) & 128 & \textbf{40.98} {\scriptsize [1.81]} & \textbf{37.56} {\scriptsize [2.04]} & \textbf{45.87} {\scriptsize [1.46]} & \textbf{39.37} {\scriptsize [1.25]}\\
\bottomrule
\end{tabular}
\end{table}

CEDR-L, both pre-NFT and post-NFT, outperforms the reported alternatives on HumanEval, HumanEval+, and MBPP-378 base within each scoring convention (Table~\ref{tab:codingcomparison}). PlaidQ and ELF-REG report base-test accuracy on the 378-task MBPP subset, leaving full extended-test MBPP+ results unavailable. NFT gives modest increases in mean pass@1 on all four benchmarks, with smaller changes on MBPP. Mean pass@10 is lower after NFT (Table~\ref{tab:codinghigherpass}). The paired comparisons in Tables~\ref{tab:codingnftdeltas} and~\ref{tab:codingsdedeltas} distinguish the effects under ODE and the headline stochastic sampler.

\paragraph{Higher-pass code generation.}
Both supervised and post-NFT CEDR-L exceed PlaidQ+CFG's standard pass@10 and ELF-REG-L's pass@10 with function-name aliasing on HumanEval, HumanEval+, and MBPP-378 base (Table~\ref{tab:codinghigherpass}). ELF-REG's higher-pass results use ODE because its stochastic study provides only four draws. Tables~\ref{tab:codingstandarddetail}--\ref{tab:codingaliasdetail} give pass@1/2/4/8/10 and uncertainty.

\begin{table}[htbp]
\centering\small
\setlength{\tabcolsep}{3pt}
\caption{\textbf{Oracle code accuracy with ten attempts.} Pass@10 (\%) with denoising NFE per sample. Small brackets give CEDR's problem-bootstrap SDs (standard errors; Appendix~\ref{app:uncertainty}). Bold/underline mark the best/second-best accuracy within each scoring block. CEDR uses learned prompt conditioning and DMD2-style sampling; PlaidQ uses token-prefix conditioning, and ELF-REG uses an external contextual encoder and ODE sampling. Both MBPP columns use 378 tasks. External values are from PlaidQ Table 1 and ELF-REG v3 Table 18; stochastic pass@10 is unavailable for ELF-REG. Generation protocols and total compute differ.}
\label{tab:codinghigherpass}
\begin{tabular}{@{}lrrrrrr@{}}
\toprule
Model & NFE & $k$ & HE & HE+ & \shortstack{MBPP-378\\base} & \shortstack{MBPP+\\full}\\
\midrule
\multicolumn{7}{@{}l}{\textit{Standard scoring: no function-name repair}}\\
PlaidQ \citep{plaidq} & 128 & 10 & 25.72 & 24.03 & 34.33 & ---\\
PlaidQ+CFG \citep{plaidq} & 256 & 10 & 39.87 & 36.74 & 44.79 & ---\\
CEDR-L, pre-NFT & 128 & 10 & \textbf{60.98} {\scriptsize [3.84]} & \textbf{54.88} {\scriptsize [3.91]} & \textbf{49.47} {\scriptsize [2.60]} & \textbf{43.65} {\scriptsize [2.59]}\\
CEDR-L, post-NFT & 128 & 10 & \underline{58.54} {\scriptsize [3.82]} & \underline{51.83} {\scriptsize [3.91]} & \underline{48.15} {\scriptsize [2.60]} & \underline{41.01} {\scriptsize [2.57]}\\
\midrule
\multicolumn{7}{@{}l}{\textit{Single-function name aliasing}}\\
ELF-REG-L \citep{scalingdlm} & 127 & 10 & 48.87 & 46.69 & 55.05 & ---\\
CEDR-L, pre-NFT & 128 & 10 & \textbf{60.98} {\scriptsize [3.84]} & \textbf{54.88} {\scriptsize [3.91]} & \textbf{61.90} {\scriptsize [2.56]} & \textbf{52.91} {\scriptsize [2.63]}\\
CEDR-L, post-NFT & 128 & 10 & \underline{58.54} {\scriptsize [3.82]} & \underline{51.83} {\scriptsize [3.91]} & \underline{61.38} {\scriptsize [2.54]} & \underline{52.12} {\scriptsize [2.62]}\\
\bottomrule
\end{tabular}
\end{table}

\section{Discussion and Conclusion}
\label{sec:conclusion}

CEDR shows that continuous latent diffusion can support strong mathematical reasoning and code generation with compact denoising backbones. Representation design is central: accurate decoding alone does not ensure strong reasoning performance, while learned multilayer features improve generation and provide separate components whose denoising trajectories can be controlled. Treating prompt text as the condition makes its encoding an internal interface that can adapt with the flow model, rather than a fixed target that must be reproduced exactly. A staged curriculum makes this flexibility practical. Guidance-corrected NFT further shows that diffusion reinforcement learning can improve this system while retaining learned guidance during sampling.

The results also expose useful directions for further work. Coding retains a substantial gap between compact and teacher prompt conditioning, motivating better context compression and adaptation. High-pass oracle accuracy exceeds majority-vote accuracy, leaving room for selectors that use latent trajectories or hidden states. The layer-sensitivity analyses motivate coordinated denoising but further analysis is needed to understand the optimal clocks. Progress on these questions could improve both reasoning quality and how effectively parallel generation uses a fixed inference budget.

\section*{Acknowledgments}

We acknowledge the computational resources provided by NCShare, which is supported by National Science Foundation (NSF) grants OAC-2201525, OAC-2201105, and OAC-2430141.

We thank Duke University Research Computing for providing the high-performance computing resources---specifically the Duke Compute Cluster (DCC) and its NVIDIA H200 Tensor Core allocation---that contributed to the research results reported in this paper.

\clearpage
\bibliography{references}

@article{odlm,
  title = {Don't Retrain, Align: Adapting Autoregressive {LMs} to Diffusion {LMs} via Representation Alignment},
  author = {Peng, Fred Zhangzhi and Fox, Alexis and Zhang, Anru R. and Tong, Alexander},
  journal = {arXiv preprint arXiv:2605.06885},
  year = {2026},
  url = {https://arxiv.org/abs/2605.06885}
}

@article{editflow,
  title = {Edit Flows: Flow Matching with Edit Operations},
  author = {Havasi, Marton and Karrer, Brian and Gat, Itai and Chen, Ricky T. Q.},
  journal = {arXiv preprint arXiv:2506.09018},
  year = {2025},
  url = {https://arxiv.org/abs/2506.09018}
}

@article{scalingdlm,
  title = {{ELF-REG}: Scaling Continuous Diffusion Language Models to Reasoning Tasks},
  author = {Li, Zeyu Michael and Chen, William Xingxu and Qian, Bingshuo and Liu, Jiayin and Cheng, Xiang},
  journal = {arXiv preprint arXiv:2609.29102},
  year = {2026},
  url = {https://arxiv.org/abs/2609.29102}
}

@article{textldm,
  title = {{TextLDM}: Language Modeling with Continuous Latent Diffusion},
  author = {Jiang, Jiaxiu and Ren, Jingjing and Li, Wenbo and Wang, Bo and Sun, Haoze and Yang, Yijun and Liu, Jianhui and Zhang, Yanbing and Zheng, Shenghe and Zhang, Yuan and Huang, Haoyang and Duan, Nan and Zuo, Wangmeng},
  journal = {arXiv preprint arXiv:2605.07748},
  year = {2026},
  url = {https://arxiv.org/abs/2605.07748}
}

@article{ldlm,
  title = {How to Train Your Latent Diffusion Language Model Jointly With the Latent Space},
  author = {Meshchaninov, Viacheslav and Shabalin, Alexander and Chimbulatov, Egor and Gushchin, Nikita and Koziev, Ilya and Korotin, Alexander and Vetrov, Dmitry},
  journal = {arXiv preprint arXiv:2605.07933},
  year = {2026},
  url = {https://arxiv.org/abs/2605.07933}
}

@article{dbtm,
  title = {Discrete {Beckmann} Transport Models for One-Step Language Modeling and Reasoning},
  author = {Tang, Sophia and Wang, Shiyi},
  journal = {arXiv preprint arXiv:2609.15903},
  year = {2026},
  url = {https://arxiv.org/abs/2609.15903}
}

@article{repa,
  title = {Representation Alignment for Generation: Training Diffusion Transformers Is Easier Than You Think},
  author = {Yu, Sihyun and Kwak, Sangkyung and Jang, Huiwon and Jeong, Jongheon and Huang, Jonathan and Shin, Jinwoo and Xie, Saining},
  journal = {arXiv preprint arXiv:2410.06940},
  year = {2024},
  url = {https://arxiv.org/abs/2410.06940}
}

@article{sfd,
  title = {Semantics Lead the Way: Harmonizing Semantic and Texture Modeling with Asynchronous Latent Diffusion},
  author = {Pan, Yueming and Feng, Ruoyu and Dai, Qi and Wang, Yuqi and Lin, Wenfeng and Guo, Mingyu and Luo, Chong and Zheng, Nanning},
  journal = {arXiv preprint arXiv:2512.04926},
  year = {2025},
  url = {https://arxiv.org/abs/2512.04926}
}

@article{asyncschedules,
  title = {Learning When to Denoise: Optimizing Asynchronous Schedules for Latent Diffusion},
  author = {Qian, Bingshuo and Cheng, Xiang},
  journal = {arXiv preprint arXiv:2606.19662},
  year = {2026},
  url = {https://arxiv.org/abs/2606.19662}
}

@article{latentforcing,
  title = {Latent Forcing: Reordering the Diffusion Trajectory for Pixel-Space Image Generation},
  author = {Baade, Alan and Chan, Eric Ryan and Sargent, Kyle and Chen, Changan and Johnson, Justin and Adeli, Ehsan and Fei-Fei, Li},
  journal = {arXiv preprint arXiv:2602.11401},
  year = {2026},
  url = {https://arxiv.org/abs/2602.11401}
}

@article{elf,
  title = {{ELF}: Embedded Language Flows},
  author = {Hu, Keya and Qiu, Linlu and Lu, Yiyang and Zhao, Hanhong and Li, Tianhong and Kim, Yoon and Andreas, Jacob and He, Kaiming},
  journal = {arXiv preprint arXiv:2605.10938},
  year = {2026},
  url = {https://arxiv.org/abs/2605.10938}
}

@article{mlfm,
  title = {Masked Language Flow Models},
  author = {Azangulov, Iskander and Ashouritaklimi, Kianoosh and Zhang, Leo and Vary, Simon and Rebeschini, Patrick},
  journal = {arXiv preprint arXiv:2606.27617},
  year = {2026},
  url = {https://arxiv.org/abs/2606.27617}
}

@article{posterior,
  title = {Posterior Refinement: Fast Language Generation via Any-Order Flow Maps},
  author = {Agarwal, Manan and Shah, Sheel and Lee, Chanhyuk and Yoo, Jaehoon and Huang, Jerry and Hong, Seunghoon and Raghunathan, Aditi and Kim, Jinwoo and Boffi, Nicholas M.},
  journal = {arXiv preprint arXiv:2606.24773},
  year = {2026},
  url = {https://arxiv.org/abs/2606.24773}
}

@article{hyperspherical,
  title = {Language Modeling with Hyperspherical Flows},
  author = {Deschenaux, Justin and Gulcehre, Caglar},
  journal = {arXiv preprint arXiv:2605.11125},
  year = {2026},
  url = {https://arxiv.org/abs/2605.11125}
}

@article{diffusionnft,
  title = {{DiffusionNFT}: Online Diffusion Reinforcement with Forward Process},
  author = {Zheng, Kaiwen and Chen, Huayu and Ye, Haotian and Wang, Haoxiang and Zhang, Qinsheng and Jiang, Kai and Su, Hang and Ermon, Stefano and Zhu, Jun and Liu, Ming-Yu},
  journal = {arXiv preprint arXiv:2509.16117},
  year = {2025},
  url = {https://arxiv.org/abs/2509.16117}
}

@article{minerva,
  title = {Solving Quantitative Reasoning Problems with Language Models},
  author = {Lewkowycz, Aitor and Andreassen, Anders and Dohan, David and Dyer, Ethan and Michalewski, Henryk and Ramasesh, Vinay and Slone, Ambrose and Anil, Cem and Schlag, Imanol and Gutman-Solo, Theo and Wu, Yuhuai and Neyshabur, Behnam and Gur-Ari, Guy and Misra, Vedant},
  journal = {arXiv preprint arXiv:2206.14858},
  year = {2022},
  url = {https://arxiv.org/abs/2206.14858}
}

@article{deepseekr1,
  title = {{DeepSeek-R1}: Incentivizing Reasoning Capability in {LLMs} via Reinforcement Learning},
  author = {{DeepSeek-AI} and Guo, Daya and Yang, Dejian and Zhang, Haowei and others},
  journal = {arXiv preprint arXiv:2501.12948},
  year = {2025},
  url = {https://arxiv.org/abs/2501.12948v1}
}

@article{qwen3,
  title = {{Qwen3} Technical Report},
  author = {Yang, An and Li, Anfeng and Yang, Baosong and others},
  journal = {arXiv preprint arXiv:2505.09388},
  year = {2025},
  url = {https://arxiv.org/abs/2505.09388}
}

@article{gsm8k,
  title = {Training Verifiers to Solve Math Word Problems},
  author = {Cobbe, Karl and Kosaraju, Vineet and Bavarian, Mohammad and Chen, Mark and Jun, Heewoo and Kaiser, Lukasz and Plappert, Matthias and Tworek, Jerry and Hilton, Jacob and Nakano, Reiichiro and Hesse, Christopher and Schulman, John},
  journal = {arXiv preprint arXiv:2110.14168},
  year = {2021},
  url = {https://arxiv.org/abs/2110.14168}
}

@article{math,
  title = {Measuring Mathematical Problem Solving With the {MATH} Dataset},
  author = {Hendrycks, Dan and Burns, Collin and Kadavath, Saurav and Arora, Akul and Basart, Steven and Tang, Eric and Song, Dawn and Steinhardt, Jacob},
  journal = {arXiv preprint arXiv:2103.03874},
  year = {2021},
  url = {https://arxiv.org/abs/2103.03874}
}

@article{openmath,
  title = {{OpenMathInstruct-2}: Accelerating {AI} for Math with Massive Open-Source Instruction Data},
  author = {Toshniwal, Shubham and Du, Wei and Moshkov, Ivan and Kisacanin, Branislav and Ayrapetyan, Alexan and Gitman, Igor},
  journal = {arXiv preprint arXiv:2410.01560},
  year = {2024},
  url = {https://arxiv.org/abs/2410.01560}
}

@article{llada,
  title = {Large Language Diffusion Models},
  author = {Nie, Shen and Zhu, Fengqi and You, Zebin and Zhang, Xiaolu and Ou, Jingyang and Hu, Jun and Zhou, Jun and Lin, Yankai and Wen, Ji-Rong and Li, Chongxuan},
  journal = {arXiv preprint arXiv:2502.09992},
  year = {2025},
  url = {https://arxiv.org/abs/2502.09992}
}

@article{flowmatching,
  title = {Flow Matching for Generative Modeling},
  author = {Lipman, Yaron and Chen, Ricky T. Q. and Ben-Hamu, Heli and Nickel, Maximilian and Le, Matt},
  journal = {arXiv preprint arXiv:2210.02747},
  year = {2022},
  url = {https://arxiv.org/abs/2210.02747}
}

@article{cka,
  title = {Similarity of Neural Network Representations Revisited},
  author = {Kornblith, Simon and Norouzi, Mohammad and Lee, Honglak and Hinton, Geoffrey},
  journal = {arXiv preprint arXiv:1905.00414},
  year = {2019},
  url = {https://arxiv.org/abs/1905.00414}
}

@misc{qwen2507,
  author = {{Qwen Team}},
  title = {{Qwen3-4B-Instruct-2507} Model Card},
  year = {2025},
  url = {https://huggingface.co/Qwen/Qwen3-4B-Instruct-2507}
}

@misc{math500,
  author = {{Hugging Face H4}},
  title = {{MATH-500}},
  year = {2024},
  url = {https://huggingface.co/datasets/HuggingFaceH4/MATH-500}
}

@misc{mathverify,
  author = {Kydl\'{i}\v{c}ek, Hynek},
  title = {{Math-Verify: Math Verification Library}},
  year = {2025},
  url = {https://github.com/huggingface/Math-Verify}
}

@article{orcamath,
  author = {Mitra, Arindam and Khanpour, Hamed and Rosset, Corby and Awadallah, Ahmed},
  title = {{Orca-Math}: Unlocking the Potential of {SLMs} in Grade School Math},
  journal = {arXiv preprint arXiv:2402.14830},
  year = {2024},
  url = {https://arxiv.org/abs/2402.14830}
}

@article{metamath,
  author = {Yu, Longhui and Jiang, Weisen and Shi, Han and Yu, Jincheng and Liu, Zhengying and Zhang, Yu and Kwok, James T. and Li, Zhenguo and Weller, Adrian and Liu, Weiyang},
  title = {{MetaMath}: Bootstrap Your Own Mathematical Questions for Large Language Models},
  journal = {arXiv preprint arXiv:2309.12284},
  year = {2023},
  url = {https://arxiv.org/abs/2309.12284}
}

@article{d1,
  title = {{d1: Scaling Reasoning in Diffusion Large Language Models via Reinforcement Learning}},
  author = {Zhao, Siyan and Gupta, Devaansh and Zheng, Qinqing and Grover, Aditya},
  journal = {arXiv preprint arXiv:2504.12216},
  year = {2025},
  url = {https://arxiv.org/abs/2504.12216}
}

@article{dream,
  title = {{Dream 7B: Diffusion Large Language Models}},
  author = {Ye, Jiacheng and Xie, Zhihui and Zheng, Lin and Gao, Jiahui and Wu, Zirui and Jiang, Xin and Li, Zhenguo and Kong, Lingpeng},
  journal = {arXiv preprint arXiv:2508.15487},
  year = {2025},
  url = {https://arxiv.org/abs/2508.15487}
}

@article{ladirl,
  title = {{LaDi-RL: Latent Diffusion Reasoning Prevents Entropy Collapse in Reinforcement Learning}},
  author = {Kang, Haoqiang and Zhang, Yizhe and Kuang, Nikki Lijing and Ma, Yi-An and Qin, Lianhui},
  journal = {arXiv preprint arXiv:2602.01705},
  year = {2026},
  url = {https://arxiv.org/abs/2602.01705}
}

@article{llada15,
  title = {{LLaDA 1.5: Variance-Reduced Preference Optimization for Large Language Diffusion Models}},
  author = {Zhu, Fengqi and Wang, Rongzhen and Nie, Shen and Zhang, Xiaolu and Wu, Chunwei and Hu, Jun and Zhou, Jun and Chen, Jianfei and Lin, Yankai and Wen, Ji-Rong and Li, Chongxuan},
  journal = {arXiv preprint arXiv:2505.19223},
  year = {2025},
  url = {https://arxiv.org/abs/2505.19223}
}

@article{mdlm,
  title = {{Simple and Effective Masked Diffusion Language Models}},
  author = {Sahoo, Subham Sekhar and Arriola, Marianne and Schiff, Yair and Gokaslan, Aaron and Marroquin, Edgar and Chiu, Justin T and Rush, Alexander and Kuleshov, Volodymyr},
  journal = {arXiv preprint arXiv:2406.07524},
  year = {2024},
  url = {https://arxiv.org/abs/2406.07524}
}

@article{plaid,
  title = {{Likelihood-Based Diffusion Language Models}},
  author = {Gulrajani, Ishaan and Hashimoto, Tatsunori B.},
  journal = {arXiv preprint arXiv:2305.18619},
  year = {2023},
  url = {https://arxiv.org/abs/2305.18619}
}

@article{plaidq,
  title = {{Distilled Continuous Diffusion Language Models Can Write Code in Few Steps---or One}},
  author = {Peng, Fred Zhangzhi and Zheng, Kaiwen and Zhang, Anru R.},
  journal = {arXiv preprint arXiv:2609.04531v1},
  year = {2026},
  note = {Version 1; subsequently withdrawn by the authors},
  url = {https://arxiv.org/abs/2609.04531v1}
}

@article{sdar,
  title = {{SDAR: A Synergistic Diffusion-AutoRegression Paradigm for Scalable Sequence Generation}},
  author = {Cheng, Shuang and Bian, Yihan and Liu, Dawei and Zhang, Linfeng and Yao, Qian and Tian, Zhongbo and Wang, Wenhai and Guo, Qipeng and Chen, Kai and Qi, Biqing and Zhou, Bowen},
  journal = {arXiv preprint arXiv:2510.06303},
  year = {2025},
  url = {https://arxiv.org/abs/2510.06303}
}

@article{tess2,
  title = {{TESS 2: A Large-Scale Generalist Diffusion Language Model}},
  author = {Tae, Jaesung and Ivison, Hamish and Kumar, Sachin and Cohan, Arman},
  journal = {arXiv preprint arXiv:2502.13917v2},
  year = {2025},
  url = {https://arxiv.org/abs/2502.13917v2}
}

@article{opencodeinstruct,
  title = {{OpenCodeInstruct}: A Large-scale Instruction Tuning Dataset for Code {LLMs}},
  author = {Ahmad, Wasi Uddin and Ficek, Aleksander and Samadi, Mehrzad and Huang, Jocelyn and Noroozi, Vahid and Majumdar, Somshubra and Ginsburg, Boris},
  journal = {arXiv preprint arXiv:2504.04030},
  year = {2025},
  url = {https://arxiv.org/abs/2504.04030}
}

@article{evalplus,
  title = {Is Your Code Generated by {ChatGPT} Really Correct? {Rigorous} Evaluation of Large Language Models for Code Generation},
  author = {Liu, Jiawei and Xia, Chunqiu Steven and Wang, Yuyao and Zhang, Lingming},
  journal = {arXiv preprint arXiv:2305.01210},
  year = {2023},
  url = {https://arxiv.org/abs/2305.01210}
}

@article{humaneval,
  title = {Evaluating Large Language Models Trained on Code},
  author = {Chen, Mark and Tworek, Jerry and Jun, Heewoo and Yuan, Qiming and Pinto, Henrique Ponde de Oliveira and others},
  journal = {arXiv preprint arXiv:2107.03374},
  year = {2021},
  url = {https://arxiv.org/abs/2107.03374}
}

@article{mbpp,
  title = {Program Synthesis with Large Language Models},
  author = {Austin, Jacob and Odena, Augustus and Nye, Maxwell and Bosma, Maarten and Michalewski, Henryk and Dohan, David and Jiang, Ellen and Cai, Carrie and Terry, Michael and Le, Quoc and Sutton, Charles},
  journal = {arXiv preprint arXiv:2108.07732},
  year = {2021},
  url = {https://arxiv.org/abs/2108.07732}
}

@misc{muon,
  title = {{Muon}: An Optimizer for Hidden Layers in Neural Networks},
  author = {Jordan, Keller},
  year = {2024},
  howpublished = {Software},
  url = {https://github.com/KellerJordan/Muon}
}

@inproceedings{consistencymodels,
  title = {Consistency Models},
  author = {Song, Yang and Dhariwal, Prafulla and Chen, Mark and Sutskever, Ilya},
  booktitle = {Proceedings of the 40th International Conference on Machine Learning},
  year = {2023},
  url = {https://proceedings.mlr.press/v202/song23a.html}
}

@article{dmd2,
  title = {Improved Distribution Matching Distillation for Fast Image Synthesis},
  author = {Yin, Tianwei and Gharbi, Micha{\"e}l and Park, Taesung and Zhang, Richard and Shechtman, Eli and Durand, Fr{\'e}do and Freeman, William T.},
  journal = {arXiv preprint arXiv:2405.14867},
  year = {2024},
  url = {https://arxiv.org/abs/2405.14867}
}

@inproceedings{edm,
  title = {Elucidating the Design Space of Diffusion-Based Generative Models},
  author = {Karras, Tero and Aittala, Miika and Aila, Timo and Laine, Samuli},
  booktitle = {Advances in Neural Information Processing Systems},
  year = {2022},
  url = {https://arxiv.org/abs/2206.00364}
}
\bibliographystyle{plainnat}

\clearpage
\appendix
\section{Training recipe and model architectures}
\label{app:training}

This section collects the architectures, training stages, and optimization settings. All tasks use the same representation and conditioning design; Table~\ref{tab:trainingstages} specifies their training durations. Detailed representation-learning equations, supervised losses, and NFT updates appear in Appendices~\ref{app:representationimplementation}, \ref{app:methodobjective}, and~\ref{app:nftdetails}, respectively.

\subsection{Architectures and fixed representations}
\label{app:architecture}

The teacher is Qwen3-4B-Instruct-2507, with hidden width 2,560 and a 151,936-token vocabulary. We learn task-specific projections of its post-block activations at layers 16, 24, and 32, with output widths 256, 256, and 512. Concatenating these covariance-whitened features gives the 1,024-dimensional representation used for both teacher prompt conditions and answer targets. Projectors are learned before flow training and then fixed; their reconstruction decoders and frozen Qwen suffixes are used only during representation learning (Appendix~\ref{app:representationimplementation}). Teacher activations and projected targets can be precomputed, so the later training stages do not require live Qwen Transformer calls.

\begin{table}[!htb]
\centering\small
\setlength{\tabcolsep}{7pt}
\caption{CEDR architectures. Both flow models use the ELF backbone, shared between denoising and token decoding. The prompt encoder is a separate bidirectional Transformer; its Qwen vocabulary lookup is frozen from initialization onward.}
\label{tab:architectures}
\begin{tabular}{@{}lcc@{}}
\toprule
Component & CEDR-B & CEDR-L\\
\midrule
ELF backbone: blocks / width / heads & 12 / 768 / 12 & 32 / 1,280 / 16\\
ELF input bottleneck & 512 & 512\\
Prompt Transformer: blocks / width / heads & 6 / 768 / 12 & 6 / 1,280 / 16\\
Prompt frontend & $2560\to512\to768$ & $2560\to512\to1280$\\
Prompt output width & 1,024 & 1,024\\
Frozen Qwen token lookup & \multicolumn{2}{c}{$151{,}936\times2{,}560$}\\
Teacher layers / projected widths & \multicolumn{2}{c}{16/24/32\quad /\quad 256/256/512}\\
Canvas length / latent width & \multicolumn{2}{c}{1,024 / 1,024}\\
\bottomrule
\end{tabular}
\end{table}

\paragraph{ELF backbone and token head.}
Table~\ref{tab:architectures} gives the model sizes. Both backbones use bidirectional attention, rotary position encoding, and SwiGLU blocks with MLP ratio four. The flow head predicts 1,024-dimensional clean latents. A separate token head maps the backbone output through 1,024 decoder features to the native vocabulary, sharing the Transformer with denoising mode. Answer self-conditioning supplies a second latent channel. Time, mode, and SCCFG embeddings condition the backbone; the inference-only duplication of time embeddings is described in Section~\ref{sec:clocks}.

\paragraph{Compact prompt encoder.}
The encoder first looks up each question token in Qwen's frozen $151{,}936\times2{,}560$ input table. Two learned linear maps, $2560\to512\to d_{\mathrm{model}}$, have no intervening nonlinearity. Six bidirectional ELF-style blocks use rotary positions and MLP ratio four, followed by RMSNorm and a 1,024-dimensional output projection. The first frontend map has no bias; the second and output maps have biases. The CEDR-L contextual encoder contains 121,333,180 trainable parameters, excluding the frozen lookup. The full 36-block Qwen Transformer contains 3,633,511,936 parameters excluding its tied vocabulary maps, which contain 388,956,160 parameters. Only question-prefix tokens enter the encoder; padding is masked and padded outputs are zero. The lookup stays fixed during prompt imitation, joint training, and NFT, and is excluded from optimizers and prompt EMAs. Learned conditioning thus executes the compact encoder and vocabulary lookup once per response, without the Qwen Transformer.

\subsection{Training stages and durations}
\label{app:trainingstages}

\paragraph{Training populations and epoch units.}
The GSM8K reasoning runs share 2,434,330 answer rows and 275,661 unique prompts. The MATH runs use the MATH/augmented-MATH portion of OpenMathInstruct-2's five-million-row split \citep{openmath}, yielding 3,967,527 eligible answer rows and 508,953 unique prompts. Coding uses 4,893,784 eligible OpenCodeInstruct answer rows and 4,772,508 unique prompts. Different solutions to one question remain separate flow-training rows, whereas standalone prompt imitation weights each unique prompt equally. Native Qwen chat sequences must fit the 1,024-token canvas; overlength rows are excluded rather than truncated. Supervised losses include answer content and one terminal token, excluding prompt and padding positions.

\paragraph{GSM8K training-data provenance.}
Both GSM8K models use the same mixture of the original \texttt{train} splits of \href{https://huggingface.co/datasets/nvidia/OpenMathInstruct-2}{OpenMathInstruct-2} \citep{openmath}, \href{https://huggingface.co/datasets/microsoft/orca-math-word-problems-200k}{Orca-Math} \citep{orcamath}, and \href{https://huggingface.co/datasets/meta-math/MetaMathQA}{MetaMathQA} \citep{metamath}. OpenMathInstruct-2 contributes only rows whose \texttt{problem\_source} is \texttt{gsm8k} or \texttt{augmented\_gsm8k}; MetaMathQA contributes only types beginning with \texttt{GSM}. Orca-Math contributes all eligible rows. Table~\ref{tab:gsmtrainingmixture} gives the final training counts. Solutions to the same question remain separate training rows, with no additional source reweighting.

\begin{table}[!htb]
\centering\small
\setlength{\tabcolsep}{6pt}
\caption{GSM8K training mixture after filtering and question-group splitting. Counts are answer rows, shared by CEDR-B and CEDR-L. All sources originate from their original training splits.}
\label{tab:gsmtrainingmixture}
\begin{tabular}{@{}llr@{}}
\toprule
Source & Retained subset & Training rows\\
\midrule
OpenMathInstruct-2 & \texttt{gsm8k} & 292,485\\
 & \texttt{augmented\_gsm8k} & 1,841,419\\
Orca-Math & Eligible train rows & 147,464\\
MetaMathQA & \texttt{GSM} types & 152,962\\
\midrule
Total & & 2,434,330\\
\bottomrule
\end{tabular}
\end{table}

\textbf{Filtering and formatting.} Corpus construction requires a nonempty question and solution and an extractable numerical final answer, allowing signed integers, decimals, and fractions. OpenMathInstruct-2 supplies its expected answer; Orca-Math and MetaMathQA answers are extracted from the solution. The builder stores a standardized response ending in one \texttt{\#\#\#\# answer} line for the initial length filter. This filter uses the \texttt{inclusionAI/LLaDA-MoE-7B-A1B-Base} tokenizer, requiring at most 256 prompt tokens and at most 1,024 prompt-plus-response tokens, including EOS. For Qwen representation extraction, the retained \texttt{problem} and \texttt{solution} fields are rendered with its native chat template; the question instruction requests a boxed final answer. All 2,434,330 rows fit: the saved Qwen token cache has maximum length 1,019 and zero truncated rows.

\textbf{Splitting and benchmark screening.} Question keys normalize Unicode, case, quotation marks, and whitespace; MetaMathQA also links rewrites through \texttt{original\_question}. Exact matches and near-duplicate candidates retrieved by MinHash/LSH are merged across sources. Near-duplicate splitting uses substring coverage at least 0.90, or word 5-/8-gram Jaccard similarity at least 0.55/0.40. A deterministic hash with seed 20260702 selects 10\% of source-local question keys for validation, and their entire merged groups are held out. The resulting split has 171,556 training groups and 25,949 validation groups (568,999 answer rows); the training population contains 275,661 distinct Qwen prompt prefixes. A 10\% training-row cap on MetaMathQA removes no further rows, since its retained share is 6.28\%.

Before splitting, the builder screens question texts against all 1,319 official GSM8K test questions, using normalized exact matching, substring coverage at least 0.85, and the same 5-/8-gram thresholds over candidates retrieved through shared n-grams. This removes 128 OpenMathInstruct-2, two Orca-Math, and 92 MetaMathQA candidate rows. The final Qwen-cache audit records zero whitespace-normalized exact question matches to GSM8K test; the saved training and validation group IDs are disjoint. These checks establish the stated lexical screening, rather than semantic decontamination or absence of overlap in the teacher's pretraining data.

\paragraph{Stage 1: flow training with frozen teacher conditioning.}
Initialize the ELF backbone and token decoder with seed 42 and train for 12 answer-row epochs using fixed Qwen prompt representations. Both answer targets and prompt conditions use the learned projectors. Optimize the mixed flow/decoder-CE objective, choosing decoder mode with probability $0.2$ and flow mode otherwise. All representation groups share a synchronous scalar time $t=\operatorname{sigmoid}(-1.5+0.8\xi)$, $\xi\sim\mathcal N(0,1)$, with flow noise scale two. Prompt positions remain clean; bootstrap and SCCFG auxiliary targets are detached. Appendix~\ref{app:methodobjective} gives the exact corruptions, masks, guidance distribution, and loss reductions.

\paragraph{Stage 2: standalone prompt imitation.}
Hold the flow model and representation fixed and initialize the compact prompt encoder's trainable layers with seed 42, using the frozen Qwen token table from the outset. Writing the projected question-only teacher states as $c_{\mathrm T}(q)$, minimize
\begin{equation}
 \mathcal L_{\mathrm{prompt}}=
 \E_q\!\left[\frac{1}{L_qd}\sum_{i=1}^{L_q}
 \|c_{\phi,i}(q)-c_{\mathrm T,i}(q)\|_2^2\right].
 \label{eq:methodpromptmse}
\end{equation}
The loss first averages valid positions and latent coordinates within a prompt, then weights prompts equally. Prompt-MSE training has its own epoch counter and does not change the epoch-12 flow checkpoint. Table~\ref{tab:trainingstages} gives the task-specific durations; prompt EMA $0.999$ initializes the math joint stage, and $0.9999$ initializes the coding joint stage.

\paragraph{Stage 3: joint adaptation.}
Replace the teacher condition with the learned encoder and continue flow training with differentiable question-prefix latents and fixed answer targets. Preserve the flow model's raw weights, optimizer state, and independent EMAs; initialize the prompt model from its selected standalone EMA and start fresh prompt optimizers. The flow/decoder objective updates both trainable networks through the main conditioned forward, while auxiliary self-conditioning and guidance targets remain detached. No task uses an additional prompt-MSE loss during joint training: mathematical reasoning trains six more epochs to epoch 18 for CEDR-B or nine to epoch 21 for CEDR-L, while coding trains two more epochs to epoch 14. All three stages retain the same prompt architecture and frozen token table.

\paragraph{Supervised optimization.}
Flow training, standalone prompt imitation, and joint adaptation use effective batch 512, constant learning rate $0.002$ (coding joint prompt encoder: $0.0005$), zero weight decay, and gradient clipping at norm one, applied separately to the flow and prompt models during joint training. Eligible matrices use Muon \citep{muon} with momentum $0.95$ and five FP32 Newton--Schulz iterations. Our implementation adapts the \texttt{muon-optimizer} package, with FP32 orthogonalization and Optax-compatible momentum bias correction and shape scaling. Remaining flow-model parameters use the auxiliary Nesterov-Adam optimizer; remaining prompt parameters use ordinary AdamW, with betas $(0.9,0.999)$ and epsilon $10^{-8}$. Trainable parameters and optimizer state use FP32, with compiled BF16 forward computation. Each trainable component maintains independent raw weights and EMAs with decays $0.99$, $0.999$, and $0.9999$, updated at optimizer boundaries. Evaluation selects flow and prompt states independently; Appendices~\ref{app:experimentdetails} and~\ref{app:codingdetails} specify the selectors for the reported results.

\subsection{Coding training stages}
\label{app:codingtraining}

The coding model uses the same frozen-conditioning, prompt-imitation, and joint-adaptation curriculum. Twelve flow-training epochs are followed by 20 epochs of standalone prompt MSE and two joint epochs (19,116 additional optimizer updates). Joint adaptation optimizes the prompt encoder and flow model through the downstream flow/decoder objective, without an auxiliary prompt-MSE anchor. Answer representations and the vocabulary lookup remain fixed.

\subsection{NFT fine-tuning}
\label{app:nfttraining}

Initialize current, old, and reference policies from the selected supervised endpoint, with fresh NFT optimizers. The current flow backbone and prompt body are trainable. The Qwen vocabulary lookup and representation projectors remain fixed, as do the decoder-specific mode tokens, output projection, and vocabulary-head parameters. The shared backbone continues to change, so freezing these decoder-specific parameters does not freeze the entire decoding function. Old and reference policies receive no optimizer gradients; each retains its own prompt encoder.

Both trainable components use constant learning rate $10^{-4}$ (coding prompt encoder: $2.5\times10^{-5}$), zero weight decay, and separate norm-one gradient clipping. The flow backbone uses Muon with auxiliary Nesterov-Adam; the prompt body uses Muon/AdamW. Computation uses BF16. Collection uses 32 denoising steps, synchronous identity clocks, and SCCFG two; optimization also retains synchronous scalar-time conditioning. Each round collects 24 prompt groups with 15 generated endpoints and one gold endpoint per group, followed by eight independently corrupted records per retained endpoint. Filtering makes the update size variable, so the supervised batch size 512 does not apply to NFT. Appendix~\ref{app:nftdetails} specifies rewards, filtering, loss normalization, and old-policy updates. Table~\ref{tab:trainingstages} lists the reported NFT endpoints. Evaluation uses flow/prompt EMA $0.99/0.99$ for both math and code, with asynchronous clocks only at inference.

\paragraph{Coding reward population.}
Coding NFT starts from joint epoch 14 with flow/prompt EMA $0.9999/0.9999$ and fresh optimizers, retaining the objective and optimization settings above. We select 4,096 unique OpenCodeInstruct training prompts and reserve 256 disjoint validation prompts, with selection seed 20260926. Eligible rows explicitly specify the tested interface, have unchanged native assertion tests and successful reference metadata, and pass reference execution in the reward harness. Prompt deduplication and exact held-out benchmark prompt exclusion precede selection. The training tests supply binary execution rewards; HumanEval/MBPP evaluation tests are never used for training. Gold endpoints receive reward one, generated programs receive $0.75$ if all retained assertions pass and zero otherwise. Synthetic training tests remain imperfect measures of correctness. Execution is isolated, network-disabled, and resource-limited, with a five-second timeout per assertion. Both flow and prompt models train jointly without an auxiliary prompt-MSE loss. We report successful NFT update 100.

\paragraph{NFT checkpoint selection.}
We compare updates 100/200/300/400 for GSM8K CEDR-B, 100/200/300/400/500/600 for GSM8K CEDR-L, and 200/400/600 for MATH CEDR-L. Selection uses mean validation accuracy over seeds 42/123 on fixed validation sets of 1,000 GSM8K and 500 MATH problems (Appendix~\ref{app:inferenceclocks}), with async $(2.5,2,1.5)$, 32 steps, and SCCFG two. Update 300 for GSM8K CEDR-B and 600 for GSM8K CEDR-L were selected based on validation accuracy; Math-Verify selects update 600 for MATH CEDR-L.

\section{CEDR objective and conditional encoding}
\label{app:methodobjective}

\subsection{Flow and token-decoding objectives}

\paragraph{Canvas and scoring positions.}
Training sequences satisfy $L_q+L_a\le L=1024$, with latent width $d=1024$. For example $b$, the mask $M_{bi}$ equals one on $L_{q,b}<i\le L_{q,b}+L_{a,b}$ and zero elsewhere. It includes the answer's terminal token and excludes prompt and padding positions. Answer latents and velocities at answer index $j$ occupy canvas position $i=L_{q,b}+j$.

\paragraph{Time and self-conditioning.}
For a flow row, sample $t=\operatorname{sigmoid}(-1.5+0.8\xi)$ with $\xi\sim\mathcal N(0,1)$ and use noise scale $\sigma=2$. A detached bootstrap pass with zero answer self-conditioning produces $\widehat z_{\mathrm{boot}}$. With $b_{\mathrm{sc}}\sim\mathrm{Bernoulli}(0.5)$, the main pass receives $\widehat z_{\mathrm{sc}}=b_{\mathrm{sc}}\widehat z_{\mathrm{boot}}$ in answer self-conditioning positions. Prompt states are restored to the clean condition independently of $b_{\mathrm{sc}}$. The endpoint predictor and velocity denominator are
\begin{equation}
 \widehat z_{\mathrm{clean}}=F_\theta(z_t,t;c,\widehat z_{\mathrm{sc}},g),
 \qquad \Delta_t=\max(1-t,0.05).
 \label{eq:methodendpoint}
\end{equation}
The velocity conversion in Section~\ref{sec:methodflow} uses this denominator for both prediction and target, so the executed target differs from $u_t=z_{\mathrm{clean}}-2\epsilon$ where the clamp is active.

\paragraph{Masked flow regression.}
The normalized counterpart of Eq.~\eqref{eq:basicflowloss}, before the SCCFG target adjustment, is
\begin{equation}
 \mathcal L_{\mathrm{FM}}=
 \E\!\left[\frac{\sum_{bi} M_{bi}\|v_{\theta,bi}-u_{bi}^{\mathrm{stab}}\|_2^2}
 {d\sum_{bi} M_{bi}}\right].
 \label{eq:methodflowloss}
\end{equation}
Here $u_{bi}^{\mathrm{stab}}$ is the target at example $b$'s sampled time $t_b$. All latent coordinates contribute. SCCFG replaces this target by $\widetilde u_{bi}$ below.

\paragraph{SCCFG target.}
Draw the guidance scale $g$ from $\log(1+g)\sim\mathrm{Uniform}(\log1.5,\log6)$, giving $g\in[0.5,5]$. Let $v^0$ and $v^1$ denote auxiliary velocity predictions without and with answer self-conditioning. The main loss uses
\begin{equation}
 \widetilde u_t=\sg\!\left[u_t^{\mathrm{stab}}+b_{\mathrm{sc}}(1-g^{-1})(v^1-v^0)\right].
 \label{eq:methodscguidance}
\end{equation}
The question is present in both auxiliary branches. All auxiliary predictions and targets are detached; during joint training, gradients enter the prompt encoder through the main conditioned forward. At inference, the preceding denoising prediction supplies recurrent self-conditioning and the requested guidance scale is passed to the network. Guidance scale one retains recurrent self-conditioning. Prompt-drop probability is zero in supervised training, and all reported evaluations use prompt CFG one.

\paragraph{Decoder corruption.}
Each row independently draws $\delta_b^{\mathrm{dec}}\sim\mathrm{Bernoulli}(0.2)$ to choose decoder mode. At each position of a decoder row, sample
\[
 \lambda_{bi}=\operatorname{sigmoid}(0.8+0.8\xi_{bi}),\qquad
 z^{\mathrm{dec}}_{bi}=\lambda_{bi}z_{\mathrm{clean},bi}+(1-\lambda_{bi})\epsilon'_{bi},
\]
where $\xi_{bi}$ and $\epsilon'_{bi}$ are standard Gaussian. Decoder noise scale is one. The model receives decoder mode, time one, zero answer self-conditioning, and the clean question prefix. Token CE predicts the original token at each position, without a next-token shift.

The valid-token expectation in Eq.~\eqref{eq:methoddecoderloss} samples a training microbatch and then uniformly selects a valid answer position across its rows. With $c_b=c(q_b)$, its explicit form is
\begin{equation}
 \mathcal L_{\mathrm{CE}}=
 \E\!\left[-\frac{\sum_{bi} M_{bi}\log p_{\theta,i}^{\mathrm{tok}}(y_{bi}\mid z_b^{\mathrm{dec}},c_b)}
 {\sum_{bi} M_{bi}}\right].
 \label{eq:methoddecoderlossexpanded}
\end{equation}
Here the answer-indexed output $p_{\theta,j}^{\mathrm{tok}}$ in the main text corresponds to canvas position $i=L_q+j$. This averages tokens within each microbatch, rather than weighting solutions equally. Decoder and flow rows use separate output heads on the same backbone, with mode probabilities $0.2$ and $0.8$, respectively.

For valid response mask $M_{bi}$, including the terminal token and excluding question/padding positions, a rank-local microbatch minimizes
\begin{equation}
 \mathcal L_{\mathrm{CEDR}}^{\mathrm{micro}}=
 \frac{\sum_{bi}M_{bi}\left[
 (1-\delta_b^{\mathrm{dec}})\|v_{\theta,bi}-\widetilde u_{bi}\|_2^2/d
 -\delta_b^{\mathrm{dec}}\log p_{\theta,i}^{\mathrm{tok}}(y_{bi}\mid z_b^{\mathrm{dec}},c_b)\right]}
 {\sum_{bi}M_{bi}}.
 \label{eq:methodmicroobjective}
\end{equation}
Here $\widetilde u_{bi}$ is the SCCFG-adjusted target at example $b$'s sampled time. There is one denominator across both row types, with no additional normalization by mode probability. Averaging over mode assignments gives $0.8\mathcal L_{\mathrm{FM}}^{\mathrm{SCCFG}}+0.2\mathcal L_{\mathrm{CE}}$ under this common normalization. The flow term averages all 1,024 coordinates, preserving the relative 256/256/512 group widths. Gradients are averaged over data-parallel ranks and accumulation microsteps. Both models share their backbone between flow and token-decoder modes (Table~\ref{tab:architectures}).

\paragraph{Sampling initialization and termination.}
Given $q$, compute $c(q)$ once and initialize all $L-L_q$ available answer positions independently from $\mathcal N(0,\sigma^2I)$. The prompt remains clean throughout. For $S$ denoising calls, Euler evaluates at $\tau_r$, $r=0,\ldots,S-1$, and advances to $\tau_S=1$; Appendix~\ref{app:inferenceclocks} specifies the grid. Answer self-conditioning starts at zero and subsequently carries the preceding clean prediction. The requested SCCFG scale is supplied to each call. A separate call through the shared backbone and token head decodes all answer positions in parallel; the first terminal token determines the returned length. Denoising NFE counts only the $S$ denoising calls.

The prompt architecture, imitation objective, and stage durations are collected in Appendix~\ref{app:training}; Eq.~\eqref{eq:methodpromptmse} defines the standalone prompt loss. Figure~\ref{fig:promptcurriculum}'s evaluation cohorts are specified in Appendix~\ref{app:promptcurves}.

\subsection{Conditional representation: proof and scope}
\label{app:conditioninglemma}

\begin{lemma}[General conditional regression]
\label{lem:conditionalregression}
Let $Q$ be the prompt, $W=(Z_t,t)$ the noisy answer and time, and $U$ a square-integrable forward-regression target. Fix a deterministic encoder $C_\phi=P_\phi(Q)$ and the joint law of $(Q,W,U)$. Define $m_Q=\E[U\mid W,Q]$ and $m_\phi=\E[U\mid W,C_\phi]$. Let $\mathcal R_Q^\star$ and $\mathcal R_\phi^\star$ be the minimum squared regression risks among all measurable predictors given $(W,Q)$ and $(W,C_\phi)$, respectively. Then
\begin{equation}
 \mathcal R_\phi^\star-\mathcal R_Q^\star
 =\E\|m_Q-m_\phi\|_2^2.
 \label{eq:conditioningrisk}
\end{equation}
In particular, the encoding attains the full-prompt optimum exactly when it preserves the posterior mean target almost surely.
\end{lemma}

\begin{proof}[Proof of Lemma~\ref{lem:conditionalregression}]
Because $C_\phi$ is a deterministic function of $Q$, the information in $(W,C_\phi)$ is contained in that of $(W,Q)$. The tower property gives $m_\phi=\E[m_Q\mid W,C_\phi]$. For any square-integrable predictor $f(W,C_\phi)$, conditional-expectation orthogonality yields
\begin{equation}
\begin{split}
 \E\|U-f(W,C_\phi)\|_2^2
 ={}&\E\|U-m_Q\|_2^2+\E\|m_Q-m_\phi\|_2^2\\
 &+\E\|m_\phi-f(W,C_\phi)\|_2^2.
\end{split}
\label{eq:conditioningdecomposition}
\end{equation}
The first term is $\mathcal R_Q^\star$. Minimizing over $f$ removes the final term and gives Eq.~\eqref{eq:conditioningrisk}. Equality of the two optimal risks is equivalent to $m_Q=m_\phi$ almost surely.
\end{proof}

\begin{proof}[Proof of Lemma~\ref{lem:conditionalrepresentation}]
For fixed $0<t<1$, differentiating the Gaussian mixture induced by the forward path gives
\begin{equation}
 s^\star(z,t\mid q)
 =\frac{t\,\E[Z_{\mathrm{clean}}\mid Z_t=z,Q=q]-z}{\sigma^2(1-t)^2}.
 \label{eq:gaussianconditionalscore}
\end{equation}
The same identity holds with $C=P_\phi(Q)$ in place of $Q$. Subtracting gives Eq.~\eqref{eq:conditionalscoredifference}; its scalar factor is nonzero on $(0,1)$. Moreover, $Z_{\mathrm{clean}}-\sigma\epsilon=(Z_{\mathrm{clean}}-Z_t)/(1-t)$ on the forward path, so substituting Eq.~\eqref{eq:gaussianconditionalscore} into its conditional expectation gives Eq.~\eqref{eq:conditionalflowscore}. Equal scores therefore give equal optimal unguided fields. Starting from the same Gaussian prior, their exact flows have the same endpoint law under the usual existence and uniqueness assumptions, interpreting the endpoint as a limit if necessary. Applying Lemma~\ref{lem:conditionalregression} to the endpoint or velocity target also identifies posterior-mean preservation with equality of the optimal regression risks.
\end{proof}

\paragraph{Posterior sufficiency and teacher imitation.}
For the random counterpart $U$ of $u_t^{\mathrm{stab}}$, namely $U=(Z_{\mathrm{clean}}-Z_t)/\Delta_t$, the posterior mean is $(\E[Z_{\mathrm{clean}}\mid Z_t,t,Q]-Z_t)/\Delta_t$. Thus preserving the posterior mean clean endpoint suffices for this regression target, including where the deterministic denominator clamp is active. Equality of the conditional distributions of $Z_{\mathrm{clean}}$ given $Q$ and $P_\phi(Q)$ is a sufficient condition. The criterion is independent of the coordinate system used by the teacher: an invertible change of encoding coordinates preserves the available information and can be compensated by the denoiser. The finite network class may limit that compensation. Prompt MSE supplies a useful initialization for the existing denoiser, while subsequent training can adapt the interface through the downstream objective.

For a concrete example, suppose $c_{\mathrm T}(q)$ preserves the posterior mean. For any scalar $\alpha$, define $c'_i(q)=c_{\mathrm T,i}(q)+\alpha\mathbf 1_d$ at each valid prompt position. Subtracting the known offset recovers the teacher encoding exactly, so $c'$ preserves the same information and ideal conditional score, while its normalized teacher-imitation MSE in Eq.~\eqref{eq:methodpromptmse} is $\alpha^2$ and can be arbitrarily large.

\paragraph{Noise, masks, and optimization.}
Conditioning Eq.~\eqref{eq:conditioningdecomposition} on time defines an excess risk $\Delta_\phi(t)=\E[\|m_Q-m_\phi\|_2^2\mid t]$. For a token-weighted objective, the same argument applies under the induced probability measure over scored positions, including the position index among the observed variables; all conditional expectations must use that same measure. The lemma assumes that the answer targets and corruption law do not depend on the encoder parameters, as in our joint stage. It justifies the conditional-model parameterization and characterizes its information requirement. The result concerns ideal unguided posterior fields.

\subsection{Prompt-learning curves and numerical results}
\label{app:promptcurves}
Table~\ref{tab:promptcurvevalues} contains the mathematical-reasoning points in Figure~\ref{fig:promptcurriculum}; Tables~\ref{tab:codingstages} and~\ref{tab:codingstageintervals} give the coding panel. Mathematical-reasoning points are means over seeds 42/123, evaluated with ODE sampling and async $(2.5,2,1.5)$, 32 denoising steps, SCCFG 2, and backbone EMA $0.9999$. Prompt EMA is $0.999$ at the MSE swap and $0.9999$ during joint training. GSM8K uses numerical-answer scoring and MATH500 uses Math-Verify. Error bars show pointwise 95\% intervals from 4,000 problem-bootstrap replicates; brackets in the table give their bounds. Epoch 12 appears twice because prompt substitution does not advance the flow-training epoch; the two rows have identical backbone weights. The axes start at zero, but no epoch-zero accuracy was measured or plotted.

The curve samples flow epochs 4, 8, and 12 before prompt substitution; epoch 12 immediately after substitution; and epochs 16/18 for CEDR-B or 16/20/21 for CEDR-L. All 20 points retain their saved predictions and scores. The joint-from-start comparison in Figure~\ref{fig:representationcomparison} instead uses the common $t=\tau^2$ clock and backbone EMA $0.999$ for its six-epoch controls; its four-seed endpoint means are not points on these curriculum curves. Training durations are summarized in Table~\ref{tab:trainingstages}.

\begin{table}[htbp]
\centering\small
\setlength{\tabcolsep}{5pt}
\caption{Mathematical-reasoning numerical values for Figure~\ref{fig:promptcurriculum}. Correct counts for seeds 42/123 and the benchmark size $N$ specify the exact per-seed accuracies. Mean pass@1 and sample seed SD (parentheses) are percentages; brackets give the plotted 95\% problem-bootstrap intervals.}
\label{tab:promptcurvevalues}
\begin{tabular}{@{}rlrrrr@{}}
\toprule
Epoch & Prompt & Correct$_{42}$ & Correct$_{123}$ & Mean (SD) & 95\% CI \\
\midrule
\multicolumn{6}{@{}l}{\textit{GSM8K, CEDR-B; $N=1319$}}\\
4 & Qwen & 92 & 86 & 6.75 {\scriptsize (0.32)} & $[5.76, 7.81]$ \\
8 & Qwen & 401 & 424 & 31.27 {\scriptsize (1.23)} & $[29.11, 33.43]$ \\
12 & Qwen & 514 & 503 & 38.55 {\scriptsize (0.59)} & $[36.28, 40.83]$ \\
12 & MSE swap & 398 & 381 & 29.53 {\scriptsize (0.91)} & $[27.29, 31.69]$ \\
16 & Joint & 482 & 462 & 35.78 {\scriptsize (1.07)} & $[33.55, 38.02]$ \\
18 & Joint & 495 & 483 & 37.07 {\scriptsize (0.64)} & $[34.76, 39.39]$ \\
\addlinespace
\multicolumn{6}{@{}l}{\textit{GSM8K, CEDR-L; $N=1319$}}\\
4 & Qwen & 304 & 281 & 22.18 {\scriptsize (1.23)} & $[20.43, 24.00]$ \\
8 & Qwen & 749 & 724 & 55.84 {\scriptsize (1.34)} & $[53.53, 58.11]$ \\
12 & Qwen & 824 & 788 & 61.11 {\scriptsize (1.93)} & $[58.76, 63.38]$ \\
12 & MSE swap & 683 & 669 & 51.25 {\scriptsize (0.75)} & $[48.79, 53.68]$ \\
16 & Joint & 760 & 754 & 57.39 {\scriptsize (0.32)} & $[55.00, 59.74]$ \\
20 & Joint & 794 & 773 & 59.40 {\scriptsize (1.13)} & $[57.01, 61.83]$ \\
21 & Joint & 790 & 787 & 59.78 {\scriptsize (0.16)} & $[57.43, 62.24]$ \\
\addlinespace
\multicolumn{6}{@{}l}{\textit{MATH500, CEDR-L; $N=500$}}\\
4 & Qwen & 30 & 38 & 6.80 {\scriptsize (1.13)} & $[5.10, 8.70]$ \\
8 & Qwen & 61 & 65 & 12.60 {\scriptsize (0.57)} & $[10.30, 15.10]$ \\
12 & Qwen & 82 & 76 & 15.80 {\scriptsize (0.85)} & $[13.10, 18.50]$ \\
12 & MSE swap & 64 & 65 & 12.90 {\scriptsize (0.14)} & $[10.50, 15.40]$ \\
16 & Joint & 84 & 92 & 17.60 {\scriptsize (1.13)} & $[14.90, 20.40]$ \\
20 & Joint & 98 & 91 & 18.90 {\scriptsize (0.99)} & $[16.00, 21.90]$ \\
21 & Joint & 85 & 94 & 17.90 {\scriptsize (1.27)} & $[15.10, 20.90]$ \\
\addlinespace
\bottomrule
\end{tabular}
\end{table}

\section{Representation construction and ablations}
\label{app:representationimplementation}

\subsection{Construction and covariance whitening}
\label{app:representationgeometry}

\paragraph{Feature geometry.}
The frozen teacher's post-block residual is $h_j^{(\ell)}=h_\psi^{(\ell)}(q,a_{\leq j})\in\R^{d_{\mathrm{teacher}}}$, with $d_{\mathrm{teacher}}=2560$. Each encoder $W_\ell^{\mathrm{enc}}\in\R^{d_{\mathrm{teacher}}\times d_\ell}$ produces $z_{\mathrm{clean},j}^{(\ell)}=(h_j^{(\ell)}-\mu_\ell)W_\ell^{\mathrm{enc}}$, and its reconstruction decoder has shape $W_\ell^{\mathrm{rec}}\in\R^{d_\ell\times d_{\mathrm{teacher}}}$. Brackets in Section~\ref{sec:representation} concatenate layer blocks in ascending order. The same maps initially encode prompt states, which cannot access subsequent answers under the teacher's causal mask.

\paragraph{Whitening and initialization.}
The linear encoder and reconstruction decoder admit changes of latent scale that leave reconstruction unchanged. We fix this freedom before applying isotropic diffusion noise. For each selected layer, estimate the mean $\mu_\ell$ and covariance $\Sigma_\ell$ over training assistant-content tokens. Write $\Sigma_\ell=U_\ell\Lambda_\ell U_\ell^\top$, with eigenvalues in descending order, and define
\[
 \widetilde\Sigma_\ell=U_\ell\operatorname{diag}\!\big(\max(\lambda_{\ell i},10^{-5})\big)U_\ell^\top.
\]
We parameterize the encoder by
\begin{equation}
 Q_\ell=\operatorname{qf}(R_\ell),\qquad
 W_\ell^{\mathrm{enc}}=\widetilde\Sigma_\ell^{-1/2}Q_\ell,\qquad
 (W_\ell^{\mathrm{enc}})^\top\widetilde\Sigma_\ell W_\ell^{\mathrm{enc}}=I_{d_\ell},
 \label{eq:representationwhitening}
\end{equation}
where $\operatorname{qf}$ is a differentiable thin QR factorization and $R_\ell$ is trainable. This maintains a nondegenerate coordinate scale while the subspace changes.
Initialize $Q_{\ell,0}=U_{\ell,:d_\ell}$, $W_{\ell,0}^{\mathrm{enc}}=\widetilde\Sigma_\ell^{-1/2}Q_{\ell,0}$, and $W_{\ell,0}^{\mathrm{rec}}=Q_{\ell,0}^\top\widetilde\Sigma_\ell^{1/2}$. The resulting reconstruction is ordinary rank-$d_\ell$ PCA:
\[
 \widehat h_{j,0}^{(\ell)}=\mu_\ell+(h_j^{(\ell)}-\mu_\ell)U_{\ell,:d_\ell}U_{\ell,:d_\ell}^\top.
\]
Starting from this initialization, we optimize Eq.~\eqref{eq:representationloss}, passing gradients through the thin QR factorization while leaving the decoder unconstrained. The whitening identity in Eq.~\eqref{eq:representationwhitening} refers to the fixed regularized training covariance; it does not imply identity covariance on every minibatch, on question features, or across concatenated layer groups.

\paragraph{Teacher-predictive reconstruction.}
For each layer, the linear decoder reconstructs the original hidden width:
\begin{equation}
 \widehat h_j^{(\ell)}=\mu_\ell+z_{\mathrm{clean},j}^{(\ell)}W_\ell^{\mathrm{rec}}.
 \label{eq:representationreconstruction}
\end{equation}
The three reconstruction branches run frozen Qwen blocks 17--36, 25--36, and 33--36, respectively, followed by final normalization and the vocabulary head. We intervene on all assistant-content states at a single layer in each branch, retaining clean prefix/control states. Full-vocabulary soft cross-entropy scores positions that predict the next assistant-content token; the first assistant token, terminal token, prompt, and padding are excluded as targets. The objective is
\begin{equation}
 \mathcal L_{\mathrm{repr}}
 =\sum_{\ell\in\mathcal S}\E_{q,a,j}
 \left[-\sum_{w\in\mathcal V}p_j^{\mathrm{teacher}}(w)
 \log p_{\ell,j}^{\mathrm{recon}}(w)\right].
 \label{eq:representationloss}
\end{equation}
Each layer is intervened on independently. Both distributions predict $a_{j+1}$ from the state at $a_j$, whereas CEDR's token decoder predicts the same-position token. The clean teacher distribution is detached; gradients through the frozen suffix train only $W_\ell^{\mathrm{enc}}$ and $W_\ell^{\mathrm{rec}}$. Teacher-to-reconstruction KL differs from this objective only by the fixed teacher entropy. These reconstruction decoders are separate from CEDR's token decoder and, together with the Qwen suffixes, are used only to learn the representation.

\subsection{Clean-token decoding from the fixed PCA representation}
\label{app:representationdiagnostics}
We use the exact frozen, whitened top-1,024 PCA coordinates from Qwen layer 32 used by the single-layer generation arm in Figure~\ref{fig:representationcomparison}. A separately trained decoder maps these coordinates through a learned $1{,}024\!\to\!1{,}024$ linear map, GELU, and a vocabulary projection to the same-position token. The PCA transform remains fixed and is applied in FP32 to cached BF16 teacher activations. Layer 32 is a deep-layer control within the teacher's 36 blocks.

The probe trains for two epochs (856 updates) on 218,857 rows, with seed 42, batch size 512, learning rate 0.002, Muon/auxiliary optimization, and FP32 vocabulary cross-entropy. Training uses four GPUs with gradients weighted by each rank's valid-token count. We evaluate final raw weights on 12,436 held-out solution rows from 4,647 prompt groups, comprising 2,112,672 assistant-content tokens. Training, validation, and test splits are disjoint by prompt. Same-position targets include the first assistant token and exclude prompt, terminal, and padding tokens.

Table~\ref{tab:pcadecoding} reports token-weighted same-position reconstruction accuracy and cross-entropy from the separately trained probe. Intervals quantify held-out sampling uncertainty using 4,000 prompt-cluster bootstrap replicates (seed 20260920), retaining all rows and tokens within each resampled prompt. These results demonstrate substantial clean-token recoverability from the actual PCA generation representation.

\begin{table}[htbp]
\centering\small
\caption{Clean-token decoding from fixed PCA coordinates after two epochs, using final raw probe weights. Brackets give 95\% prompt-cluster bootstrap intervals. The input is the representation used by the single-layer generation ablation.}
\label{tab:pcadecoding}
\begin{tabular}{@{}lcc@{}}
\toprule
Probe input & Token accuracy (\%) & CE (nats/token) \\
\midrule
Fixed layer-32 PCA & 99.24 [99.19, 99.28] & 0.0548 [0.0502, 0.0597] \\
\bottomrule
\end{tabular}
\end{table}

\subsection{Layer similarity and CKA}
\label{app:representationcka}
The CKA study uses disjoint discovery and confirmation panels of 4,096 problems each, totaling 8,192 equally weighted problems. For each problem, teacher-forced assistant activations are mean-pooled into eight relative-position bins and concatenated before computing linear CKA. Figure~\ref{fig:representationcka} shows the combined panel. The layer-16/17 CKA is 0.038; the selected-layer similarities are 0.124 (16/24), 0.543 (16/32), and 0.842 (24/32). We use the transition after layer 16 and regular spacing thereafter to motivate the representation.

\subsection{Representation ablations and clock robustness}
\label{app:representationclock}

Figure~\ref{fig:representationidentity} changes the evaluation clock to $t=\tau$, retaining the trained checkpoints, 32 steps, guidance, and EMA selectors; the identity-clock comparison shows that the epoch-six ordering persists under another schedule. The two-seed identity-clock means are 26.76\%, 24.22\%, 13.38\%, and 8.68\% for learned multilayer, PCA multilayer, PCA layer 32, and joint prompt training from initialization. Learned multilayer leads PCA multilayer from epoch two onward under both clocks, and multilayer PCA leads layer-32 PCA at every epoch.

\begin{figure}[htbp]
\centering
\includegraphics[width=.66\linewidth]{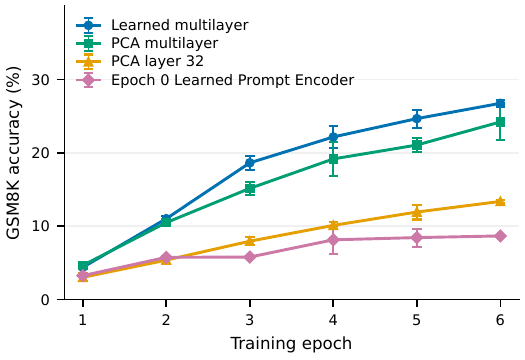}
\caption{\textbf{Identity-clock robustness.} The four models from Figure~\ref{fig:representationcomparison}, evaluated with $t=\tau$ and 32 steps. Points average seeds 42/123 (bars: seed SD). Each epoch uses one trained checkpoint per arm. Numerical values appear in Table~\ref{tab:representationcurvevalues}.}
\label{fig:representationidentity}
\end{figure}

All four arms start from the same seed-42 backbone initialization, use 2,434,330 training rows, and preserve effective batch 512. All answer representations are whitened and 1,024-dimensional. The first three arms use frozen Qwen prompt activations; the fourth uses learned multilayer answer latents and a randomly initialized prompt encoder trained jointly from epoch zero. All select backbone EMA $0.999$, and the fourth selects prompt EMA $0.999$. The PCA multilayer run changes from two to four GPUs during training while maintaining that batch size. Each arm has one training run; seed variation in the figures measures sampling randomness only. The comparisons isolate projector choice at fixed layers and layer choice under PCA. Paired intervals in the main text use 4,000 problem-bootstrap replicates and the four matched epoch-six generation seeds.

\paragraph{Representation-ablation curves.}
Table~\ref{tab:representationcurvevalues} gives all points and error bars in Figures~\ref{fig:representationcomparison} and~\ref{fig:representationidentity}, followed by the four-seed epoch-six summary discussed in Sections~\ref{sec:representationablations} and~\ref{sec:conditionalencoding}. The curves use seeds 42/123; the endpoint summary adds 456/789. Error bars and parenthesized values are sample SDs across their respective generation seeds. Both use 32 steps and backbone EMA $0.999$; The controls above apply to both clocks.

\begin{table}[htbp]
\centering\small
\setlength{\tabcolsep}{5pt}
\caption{Numerical learning curves and endpoint summary for Figures~\ref{fig:representationcomparison} and~\ref{fig:representationidentity}: mean GSM8K accuracy (\%) with sample seed SD in parentheses. Curves use two seeds (42/123); the final block reports the four-seed epoch-six comparison (42/123/456/789). The final column trains the prompt encoder jointly from epoch zero; the other three retain Qwen prompt conditioning.}
\label{tab:representationcurvevalues}
\begin{tabular}{@{}rrrrr@{}}
\toprule
Epoch & \shortstack{Learned\\16/24/32} & \shortstack{PCA\\16/24/32} & \shortstack{PCA\\layer 32} & \shortstack{Epoch-0 joint\\prompt training} \\
\midrule
\multicolumn{5}{@{}l}{$t=\tau^2$}\\
1 & 3.75 {\scriptsize (0.05)} & 4.85 {\scriptsize (0.75)} & 2.65 {\scriptsize (0.11)} & 2.46 {\scriptsize (0.05)} \\
2 & 12.40 {\scriptsize (0.27)} & 11.18 {\scriptsize (0.16)} & 5.50 {\scriptsize (0.70)} & 3.98 {\scriptsize (1.02)} \\
3 & 20.17 {\scriptsize (0.54)} & 16.38 {\scriptsize (0.32)} & 8.64 {\scriptsize (0.43)} & 3.68 {\scriptsize (0.16)} \\
4 & 26.61 {\scriptsize (0.86)} & 21.49 {\scriptsize (0.80)} & 7.47 {\scriptsize (0.05)} & 7.77 {\scriptsize (1.02)} \\
5 & 29.95 {\scriptsize (0.21)} & 25.85 {\scriptsize (0.21)} & 9.82 {\scriptsize (0.16)} & 7.54 {\scriptsize (1.45)} \\
6 & 31.01 {\scriptsize (0.43)} & 28.39 {\scriptsize (0.27)} & 11.68 {\scriptsize (0.21)} & 8.61 {\scriptsize (0.05)} \\
\addlinespace
\multicolumn{5}{@{}l}{$t=\tau$}\\
1 & 4.40 {\scriptsize (0.75)} & 4.66 {\scriptsize (0.48)} & 3.03 {\scriptsize (0.54)} & 3.26 {\scriptsize (0.11)} \\
2 & 11.03 {\scriptsize (0.38)} & 10.50 {\scriptsize (0.16)} & 5.38 {\scriptsize (0.54)} & 5.76 {\scriptsize (0.00)} \\
3 & 18.65 {\scriptsize (0.96)} & 15.16 {\scriptsize (0.86)} & 7.96 {\scriptsize (0.54)} & 5.80 {\scriptsize (0.05)} \\
4 & 22.18 {\scriptsize (1.45)} & 19.18 {\scriptsize (2.36)} & 10.12 {\scriptsize (0.48)} & 8.15 {\scriptsize (1.88)} \\
5 & 24.68 {\scriptsize (1.23)} & 21.08 {\scriptsize (0.96)} & 11.94 {\scriptsize (1.02)} & 8.45 {\scriptsize (1.23)} \\
6 & 26.76 {\scriptsize (0.43)} & 24.22 {\scriptsize (2.52)} & 13.38 {\scriptsize (0.16)} & 8.68 {\scriptsize (0.16)} \\
\addlinespace
\multicolumn{5}{@{}l}{Epoch-six summary: $t=\tau^2$, four seeds}\\
6 & 30.76 {\scriptsize (0.39)} & 28.43 {\scriptsize (0.19)} & 11.20 {\scriptsize (0.56)} & 8.70 {\scriptsize (0.33)} \\
\addlinespace
\bottomrule
\end{tabular}
\end{table}

\section{Inference clocks and reconstruction diagnostics}
\label{app:inferenceclocks}

\subsection{Local clocks, time embeddings, and samplers}

\paragraph{Asynchronous SDE and its probability-flow limit.}
\label{app:asyncsde}
The group-clock construction follows \citet{asyncschedules}. For fixed prompt $q$, write $p_\tau(z)=p_{\bm t(\tau)}(z\mid q)$. Differentiating expectations along Eq.~\eqref{eq:asynccorruption} and conditioning on the full state yields the continuity equation $\partial_\tau p_\tau=-\sum_\ell\nabla_\ell\!\cdot(p_\tau f'_\ell v_\ell^\star)$. Here $\nabla_\ell$ acts on all answer coordinates in group $\ell$. The Gaussian score identity gives Eq.~\eqref{eq:asynclocalfields}. Adding the score drift and diffusion in Eq.~\eqref{eq:asyncsde} contributes
\[
 \sum_\ell f'_\ell\nu_\ell
 \left[-\nabla_\ell\!\cdot(p_\tau s_\ell^\star)
 +\nabla_\ell^2 p_\tau\right]=0,
\]
because $p_\tau s_\ell^\star=\nabla_\ell p_\tau$ and $f'_\ell\nu_\ell$ is independent of the state. This is the probability-flow-plus-Langevin construction \citep{edm}, applied separately to each group while using the joint score. It establishes the common marginals for exact fields on interior times under regularity and well-posedness assumptions, with endpoints interpreted by limits. Guided predictions, synchronous-to-asynchronous model adaptation, finite steps, and denominator stabilization need not preserve those ideal marginals.

\paragraph{Connecting the discrete samplers to these dynamics.}
For $0\leq\bar t\leq t\leq1$ with $t>0$, a Markov forward-noising transition for one group is
\begin{equation}
 K_{\bar t\leftarrow t}(\bar z\mid z)
 =\mathcal N\!\left(\frac{\bar t}{t}z,
 \sigma^2\left[(1-\bar t)^2-\left(\frac{\bar t}{t}\right)^2(1-t)^2\right]I\right).
 \label{eq:asyncforwardkernel}
\end{equation}
It rescales the signal coefficient from $t$ to $\bar t$ and adds precisely the missing noise variance. Applying these kernels independently across groups preserves the joint corruption law, including correlations inherited from the clean representation. Equivalently, this kernel integrates the forward-noising SDE $\mathrm dX_u=-X_u/(1-u)\,\mathrm du+\sigma\sqrt{2u/(1-u)}\,\mathrm dB_u$ in increasing noise time $u=1-t$.

For DMD2-style sampling, let $m=\widehat z_{\mathrm{clean},r}^{(\ell)}$. One Euler step of size $1-t_{\ell,r}$ with the unstabilized local velocity $(m-z_r^{(\ell)})/(1-t_{\ell,r})$ gives exactly $m$. Following this endpoint approximation with $K_{t_{\ell,r+1}\leftarrow1}$ gives Eq.~\eqref{eq:asyncdmd2}. The deterministic substep therefore advances to $t=1$, while the stochastic substep returns to $t_{\ell,r+1}$; this differs from Euler--Maruyama's drift and noise update over a single time interval. The denoising jump and noise refresh remain finite as the destination grid is refined, so this implementation has no fixed finite-diffusivity Euler--Maruyama limit.

For EDM-style sampling, Eq.~\eqref{eq:asyncchurnback} applies $K_{\bar t\leftarrow t}$ before a local ODE step. To relate it to Eq.~\eqref{eq:asyncsde}, let $h_\ell=t_{\ell,r+1}-t_{\ell,r}$ and $\delta_\ell=t_{\ell,r}-\bar t_{\ell,r}$. At an interior time, as $h_\ell\to0$ with $\delta_\ell/h_\ell\to\chi_\ell\geq0$, the conditional mean and covariance of the combined increment are, to first order,
\begin{equation}
 \begin{aligned}
 \E[\Delta z^{(\ell)}\mid z]
 &=h_\ell v_\ell^\star
 +\delta_\ell\frac{\sigma^2(1-t_\ell)}{t_\ell}s_\ell^\star+o(h_\ell),\\
 \operatorname{Cov}(\Delta z^{(\ell)}\mid z)
 &=2\delta_\ell\frac{\sigma^2(1-t_\ell)}{t_\ell}I+o(h_\ell).
 \end{aligned}
 \label{eq:asyncchurnlimit}
\end{equation}
Thus the small-step splitting corresponds to Eq.~\eqref{eq:asyncsde} with $\nu_\ell=\chi_\ell\sigma^2(1-t_\ell)/t_\ell$, using $h_\ell=f'_\ell(\tau_r)\Delta\tau+o(\Delta\tau)$. Our evaluated index contraction is a finite-churn rule; it does not enforce this small-step scaling. Both stochastic implementations are therefore specified by their actual discrete updates, rather than identified with a shared fixed SDE coefficient.

\paragraph{Local corruption and time embeddings.}
The local clocks in Section~\ref{sec:clocks} assign each group its own Gaussian corruption, $z_{\bm t}^{(\ell)}=t_\ell z_{\mathrm{clean}}^{(\ell)}+(1-t_\ell)\sigma\epsilon^{(\ell)}$. The three time embedders are identical copies of the selected scalar-time checkpoint weights:
\begin{equation}
 \bar e_\theta^{\mathrm{time}}(\bm t)
 =\frac{1}{3}\sum_{\ell\in\{16,24,32\}}e_{\theta,\ell}^{\mathrm{time}}(t_\ell),
 \qquad e_{\theta,\ell}^{\mathrm{time}}=e_\theta^{\mathrm{time}}.
 \label{eq:asynctimeembedding}
\end{equation}
All other tensors are unchanged. The embedding outputs are weighted equally, irrespective of the 256/256/512 group widths, with the mean accumulated in FP64 before casting back to model dtype. The resulting embedding is a shared network condition, not a separate Transformer for each group. In particular, permuting the three exponents preserves this conditioning average while changing the coordinate groups to which the local velocity factors apply. Because the group widths differ, that permutation need not preserve total latent noise energy.

\paragraph{Sampling grid and conditioning.}
For $S$ denoising steps, we set $\tau_0=0$, $\tau_S=1$, and use the same logit-normal-quantile grid across schedules:
\begin{equation}
 \tau_r=\operatorname{sigmoid}\!\left[-1.5+0.8\Phi^{-1}(r/S)\right],
 \qquad 0<r<S.
\end{equation}
The base $\tau$ grid is fixed across schedule and sampler comparisons. For the ODE baseline, we retain the full derivative $f'_\ell(\tau)=\gamma_\ell\tau^{\gamma_\ell-1}$ in Eq.~\eqref{eq:asyncsampler}, with the endpoint denominator $\Delta_{t_{\ell,r}}=\max(1-\tau_r^{\gamma_\ell},0.05)$. Question states remain clean, and recurrent self-conditioning carries the previous endpoint prediction to the next call. Changing the clock therefore alters both the integration grid in local time and the prediction history. One shared denoiser call updates all groups; the terminal token-decoding call remains separate.

\paragraph{Stochastic-sampler implementation.}
\label{app:samplerimplementation}
The labels \emph{Stochastic (DMD2)} and \emph{Stochastic (EDM)} identify the update families in Eqs.~\eqref{eq:asyncdmd2}--\eqref{eq:asyncchurnforward}. Both use the same destination grid as ODE, $\sigma=2$, and default local-clock exponents $(2.5,2,1.5)$; the schedule comparison varies these powers. They operate on the fixed supervised and NFT endpoints without retraining or checkpoint reselection.

For DMD2-style sampling, the injected noise standard deviation in group $\ell$ is $\sigma(1-t_{\ell,r+1})$, as in Eq.~\eqref{eq:asyncdmd2}. Each nonterminal step draws fresh Gaussian noise after predicting the clean endpoint, including the first step. At $r=S-1$, Eq.~\eqref{eq:asyncdmd2} gives $z_S=\widehat z_{\mathrm{clean},S-1}$, and no noise is drawn.

For EDM-style sampling, let $G(u)=\operatorname{sigmoid}[-1.5+0.8\Phi^{-1}(u)]$ for $0<u<1$, with $G(0)=0$ and $G(1)=1$. On interior transitions $1\leq r\leq S-2$, contract the \emph{uniform quantile index} by a factor of $0.7$:
\begin{equation}
 \bar\tau_r=G(0.7r/S),\qquad
 \bar t_{\ell,r}=f_\ell(\bar\tau_r),\qquad
 \tau_r=G(r/S).
 \label{eq:indexchurnclock}
\end{equation}
Thus $\bar\tau_r$ is generally not $0.7\tau_r$. The variance in Eq.~\eqref{eq:asyncchurnback} follows by matching the forward Gaussian corruption: scaling $z_r^{(\ell)}$ by $\alpha_{\ell,r}$ gives signal coefficient $\bar t_{\ell,r}$, and independent noise supplies the remaining variance. This is nonnegative for $0\leq\bar t_{\ell,r}\leq t_{\ell,r}<1$. The denoiser is evaluated at these noisier times, using the same averaged time embeddings as above. Eq.~\eqref{eq:asyncchurnforward} uses the finite increment $t_{\ell,r+1}-\bar t_{\ell,r}$ and denominator $1-\bar t_{\ell,r}$, without the derivative factor or denominator clamp of Eq.~\eqref{eq:asyncsampler}.

At the initial transition, churn is bypassed: $\bar z_0=z_0$ and $\bar t_{\ell,0}=0$, followed by Eq.~\eqref{eq:asyncchurnforward}. At the terminal transition, churn is again bypassed; the denoiser is evaluated at the current state and times, and its clean prediction is used directly as $z_S$. We restore prompt positions after noise injection and every state update, and retain the previous guided clean prediction for self-conditioning. The noise-injection and local Euler rules adapt EDM's churn principle; the quantile-index contraction and first-order update differ from its original sampler.

Both stochastic samplers initialize from the same seeded Gaussian distribution as ODE, use $S$ denoiser calls and one separate terminal decoder call, and add no terminal noise. Fresh noise comes from a separate per-batch CUDA generator, seeded by $(s+1{,}000{,}003+b\cdot1{,}000{,}003)\bmod(2^{63}-1)$, where $s$ is the evaluation seed and $b$ the first benchmark-row index in the batch. DMD2-style sampling applies $S-1$ fresh draws. EDM-style sampling consumes the first draw without applying it and applies the $S-2$ interior draws. The network runs in BF16, latent updates in FP32, and the time-embedding mean in FP64 before casting back.

The stochastic samplers decode their final predicted-clean state, while the ODE baseline decodes its final Euler-updated state. Comparisons with ODE therefore include this terminal-rule difference, as well as the distinct update formulas, and do not isolate the effect of fresh noise alone.

\subsection{Schedule comparisons}
\label{app:schedulecomparisons}
All test-set schedule comparisons use CEDR-L epoch 21, 32 denoising steps, CFG 1, SCCFG 2 for GSM8K and 3 for MATH500, backbone/prompt EMA $0.9999/0.9999$, batch size 32, and seeds 42/123. GSM8K CEDR-L uses the same learned projectors as CEDR-B. Scores use numerical-answer accuracy for all 1,319 GSM8K test problems and Math-Verify for all 500 MATH500 problems (Appendix~\ref{app:experimentdetails}); PRM800K remains a secondary scorer. Table~\ref{tab:inferenceclocks} summarizes all three samplers; Table~\ref{tab:stochasticclockdetail} gives the individual stochastic-sampler seed results and pass@2. Bold marks the highest mean among the four schedules within each dataset/sampler. The stochastic comparison uses the same checkpoints and seeds as the ODE references and covers test sets at epoch 21. Thus both datasets use the same guidance as their headline results.

Table~\ref{tab:inferenceclockfull} retains the original ODE validation sweep at SCCFG 2 for both datasets, alongside the test results at headline guidance. The 500-problem MATH validation set comes from the same original MATH test pool as MATH500, but excluding the MATH500 questions, whereas our GSM8K validation set comprises held-out original training questions. We therefore regard the MATH validation comparison as more directly matched to its benchmark's source distribution. Both validation sets give their highest observed mean to $(2.5,2,1.5)$. On GSM8K validation, its advantages over $(2,2,2)$ and $(1.5,2,2.5)$ are 0.75 and 0.50 percentage points, respectively, with paired 95\% intervals $[-0.30,1.85]$ and $[-0.60,1.55]$.

\begin{table}[htbp]
\centering\small
\caption{Four-schedule CEDR-L epoch-21 comparison: mean pass@1 (\%) with sample seed SD in parentheses, over seeds 42/123 with ODE sampling at 32 denoising steps. Validation uses SCCFG 2; test uses SCCFG 2 for GSM8K and SCCFG 3 for MATH500. Validation sets contain 1,000 GSM8K and 500 MATH problems. Test columns reproduce the ODE entries and SDs from Table~\ref{tab:inferenceclocks}. Bold marks each column maximum, including ties.}
\label{tab:inferenceclockfull}
\begin{tabular}{lrrrr}
\toprule
& \multicolumn{2}{c}{Validation} & \multicolumn{2}{c}{Test} \\
\cmidrule(lr){2-3}\cmidrule(lr){4-5}
Clock powers & GSM8K & MATH & GSM8K & MATH500 \\
\midrule
$(1,1,1)$ & 64.60 {\scriptsize (0.57)} & 15.40 {\scriptsize (0.28)} & 52.62 {\scriptsize (0.75)} & 16.70 {\scriptsize (0.42)} \\
$(2,2,2)$ & 72.85 {\scriptsize (0.21)} & 17.80 {\scriptsize (0.57)} & 59.97 {\scriptsize (0.64)} & 18.80 {\scriptsize (0.85)} \\
$(1.5,2,2.5)$ & 73.10 {\scriptsize (0.85)} & 18.30 {\scriptsize (0.14)} & \textbf{60.08} {\scriptsize (0.38)} & 18.50 {\scriptsize (0.14)} \\
$(2.5,2,1.5)$ & \textbf{73.60} {\scriptsize (0.28)} & \textbf{18.70} {\scriptsize (0.14)} & 59.78 {\scriptsize (0.16)} & \textbf{19.70} {\scriptsize (0.42)} \\
\bottomrule
\end{tabular}
\end{table}

\paragraph{Stochastic schedule comparisons.}
Table~\ref{tab:stochasticclockdetail} gives individual seed results and pass@2 for both stochastic samplers. Pass@1 averages the two draws within each problem and then averages problems equally; pass@2 is the fraction of problems solved by either draw. Sample SD describes variation between the two full-benchmark seed accuracies, rather than problem-bootstrap uncertainty.

\begin{table}[htbp]
\centering\small
\setlength{\tabcolsep}{4pt}
\caption{Detailed stochastic-sampler schedule results underlying Table~\ref{tab:inferenceclocks}: individual seed accuracies, mean pass@1 with sample seed SD, and oracle pass@2, all in percent. Settings match the 32-step ODE comparison, with SCCFG 2 for GSM8K and 3 for MATH500; each problem has two generated solutions.}
\label{tab:stochasticclockdetail}
\begin{tabular}{@{}llrrrr@{}}
\toprule
Sampler & Clock powers & Seed 42 & Seed 123 & Mean pass@1 & pass@2 \\
\midrule
\multicolumn{6}{l}{\textbf{GSM8K}} \\
\multirow{4}{*}{Stochastic (EDM)} & $(1,1,1)$ & 59.82 & 59.74 & 59.78 {\scriptsize($\pm 0.05$)} & 69.60 \\
 & $(2,2,2)$ & 61.18 & 62.47 & 61.83 {\scriptsize($\pm 0.91$)} & 69.83 \\
 & $(1.5,2,2.5)$ & 63.00 & 63.84 & 63.42 {\scriptsize($\pm 0.59$)} & 70.81 \\
 & $(2.5,2,1.5)$ & 64.67 & 64.22 & \textbf{64.44} {\scriptsize($\pm 0.32$)} & 71.87 \\
\addlinespace[3pt]
\multirow{4}{*}{Stochastic (DMD2)} & $(1,1,1)$ & 60.80 & 60.20 & 60.50 {\scriptsize($\pm 0.43$)} & 69.75 \\
 & $(2,2,2)$ & 64.75 & 64.97 & 64.86 {\scriptsize($\pm 0.16$)} & 71.04 \\
 & $(1.5,2,2.5)$ & 67.10 & 65.88 & 66.49 {\scriptsize($\pm 0.86$)} & 72.10 \\
 & $(2.5,2,1.5)$ & 66.72 & 66.64 & \textbf{66.68} {\scriptsize($\pm 0.05$)} & 72.40 \\
\midrule
\multicolumn{6}{l}{\textbf{MATH500 (Math-Verify)}} \\
\multirow{4}{*}{Stochastic (EDM)} & $(1,1,1)$ & 18.40 & 20.80 & 19.60 {\scriptsize($\pm 1.70$)} & 27.20 \\
 & $(2,2,2)$ & 18.40 & 15.20 & 16.80 {\scriptsize($\pm 2.26$)} & 25.00 \\
 & $(1.5,2,2.5)$ & 20.60 & 19.20 & 19.90 {\scriptsize($\pm 0.99$)} & 28.40 \\
 & $(2.5,2,1.5)$ & 20.40 & 22.80 & \textbf{21.60} {\scriptsize($\pm 1.70$)} & 30.20 \\
\addlinespace[3pt]
\multirow{4}{*}{Stochastic (DMD2)} & $(1,1,1)$ & 19.20 & 23.00 & 21.10 {\scriptsize($\pm 2.69$)} & 29.40 \\
 & $(2,2,2)$ & 22.20 & 25.60 & 23.90 {\scriptsize($\pm 2.40$)} & 31.80 \\
 & $(1.5,2,2.5)$ & 23.20 & 26.80 & 25.00 {\scriptsize($\pm 2.55$)} & 33.20 \\
 & $(2.5,2,1.5)$ & 26.40 & 29.40 & \textbf{27.90} {\scriptsize($\pm 2.12$)} & 36.20 \\
\bottomrule
\end{tabular}
\end{table}

\paragraph{Paired schedule differences.}
Table~\ref{tab:clockpaireddifferences} compares $(2.5,2,1.5)$ against $(2,2,2)$ and $(1.5,2,2.5)$ within each sampler. Its GSM8K advantage over $(2,2,2)$ has a positive paired interval for both stochastic samplers, while the corresponding ODE interval includes zero. On MATH500, DMD2-style sampling gives positive paired intervals against both comparators: $+4.00$ points over $(2,2,2)$ and $+2.90$ over $(1.5,2,2.5)$. The EDM-style gain over $(2,2,2)$ also has a positive interval, while its reverse-versus-forward interval and both ODE intervals include zero. Neither stochastic sampler separates the two asynchronous clocks on GSM8K with these two seeds. The shared-square clock is also not uniformly preferable: under EDM-style sampling, MATH500 accuracy falls from 19.60\% with identity clocks to 16.80\% with shared-square clocks.

\begin{table}[htbp]
\centering\small
\setlength{\tabcolsep}{5pt}
\caption{Paired pass@1 differences, in percentage points, for $(2.5,2,1.5)$ minus the listed comparator clock within each sampler. Brackets give 95\% problem-bootstrap intervals. All comparisons use CEDR-L epoch 21, 32 steps, SCCFG 2 for GSM8K / 3 for MATH500, and seeds 42/123.}
\label{tab:clockpaireddifferences}
\begin{tabular}{@{}llrr@{}}
\toprule
Sampler & Comparator & GSM8K & MATH500 \\
\midrule
\multirow{2}{*}{ODE} & $(2,2,2)$ & $-0.19\;[-1.06,+0.68]$ & $+0.90\;[-1.00,+2.90]$ \\
 & $(1.5,2,2.5)$ & $-0.30\;[-1.18,+0.57]$ & $+1.20\;[-0.90,+3.40]$ \\
\addlinespace[3pt]
\multirow{2}{*}{Stochastic (EDM)} & $(2,2,2)$ & $+2.62\;[+1.25,+4.02]$ & $+4.80\;[+2.50,+7.10]$ \\
 & $(1.5,2,2.5)$ & $+1.02\;[-0.23,+2.27]$ & $+1.70\;[-0.60,+4.10]$ \\
\addlinespace[3pt]
\multirow{2}{*}{Stochastic (DMD2)} & $(2,2,2)$ & $+1.82\;[+0.53,+3.07]$ & $+4.00\;[+1.40,+6.50]$ \\
 & $(1.5,2,2.5)$ & $+0.19\;[-1.02,+1.40]$ & $+2.90\;[+0.40,+5.30]$ \\
\bottomrule
\end{tabular}
\end{table}

Intervals use 4,000 paired problem-bootstrap replicates with bootstrap seed 20260918, resampling the same problem indices across compared settings and retaining both draws per problem. These pointwise intervals quantify uncertainty across problems for the fixed checkpoints and draws (Appendix~\ref{app:uncertainty}). Within each sampler, only the clocks change. Across samplers, comparisons include the distinct integration and terminal-decoding rules described in Appendix~\ref{app:samplerimplementation}.

\FloatBarrier
\subsection{Layer decomposition and reconstruction sensitivity}
We test whether the three layer components differ in how easily they can be reconstructed after corruption. Using CEDR-B at epoch 18, we add noise to one component of the gold-solution latents, predict the clean representation, and decode it with the model's token head. The metric is the change in gold-token decoder CE relative to the same model's prediction from an entirely clean input (Figure~\ref{fig:layerreconstruction}).

\begin{figure}[htbp]
\centering
\includegraphics[width=\linewidth]{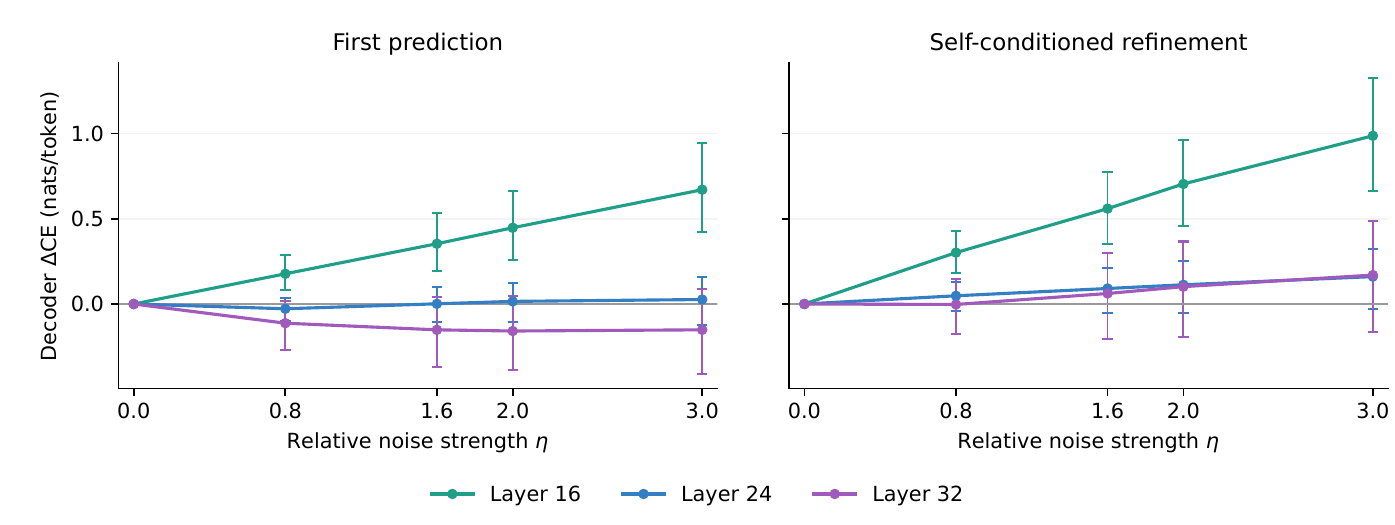}
\caption{\textbf{Reconstruction sensitivity of the three layer components.} CEDR-B epoch 18 on 256 GSM8K validation problems. Each curve corrupts only the indicated layer component. Left: first clean-latent prediction; right: one self-conditioned refinement at the same noisy input. Changes are relative to each pass's clean-input prediction; bars show paired 95\% problem-bootstrap intervals. Table~\ref{tab:reconstructioncurvevalues} gives the plotted values.}
\label{fig:layerreconstruction}
\end{figure}

\paragraph{Probe setup.}
We select 64 problems from each token-length quartile of our 1,000-problem GSM8K validation set, using selection seed 20260918. Both backbone and prompt encoder use EMA .9999; SCCFG is 1. For component $\ell$, set $t_\ell=2/[2+\eta\,\mathrm{RMS}(z_{\mathrm{clean}}^{(\ell)})]$ and corrupt its assistant-content coordinates as $t_\ell z_{\mathrm{clean}}^{(\ell)}+2(1-t_\ell)\epsilon^{(\ell)}$, with $\eta\in\{0.8,1.6,2,3\}$ and Gaussian noise seeds 42/123. The other components and all prompt, terminal, and padding positions remain clean. The canvas uses the gold solution length. The first prediction has no answer self-conditioning; the second uses that prediction as self-conditioning without changing the noisy input or advancing time.

We score same-position gold answer tokens, excluding the first assistant token, terminal tokens, prompt, and padding. Losses are averaged within each problem, then over noise seeds and problems equally. Confidence intervals use 4,000 paired problem-bootstrap replicates with seed 20260918. Negative changes indicate lower CE than the clean-input \emph{prediction} baseline. Zero-noise changes are identically zero.

\paragraph{Interpretation.}
Perturbing layer 16 produces greater CE degradation than perturbing layer 24 or 32 at every tested strength in both passes; the paired intervals for these layer contrasts exclude zero. The corresponding contrasts between layers 24 and 32 remain inconclusive. This supports treating the components as having different reconstruction sensitivities.

\begin{table}[htbp]
\centering\small
\setlength{\tabcolsep}{6pt}
\caption{Decoder $\Delta$CE in nats/token for Figure~\ref{fig:layerreconstruction}, with paired 95\% problem-bootstrap intervals. Each prediction pass uses its own clean-input baseline. At $\eta=0$, all entries are exactly zero.}
\label{tab:reconstructioncurvevalues}
\begin{tabular}{@{}rrcc@{}}
\toprule
Layer & $\eta$ & First prediction [95\% CI] & SC refinement [95\% CI] \\
\midrule
16 & 0.8 & $0.177\ [0.080, 0.285]$ & $0.302\ [0.183, 0.428]$ \\
16 & 1.6 & $0.354\ [0.192, 0.533]$ & $0.560\ [0.353, 0.777]$ \\
16 & 2 & $0.447\ [0.257, 0.661]$ & $0.704\ [0.459, 0.960]$ \\
16 & 3 & $0.671\ [0.422, 0.943]$ & $0.987\ [0.665, 1.328]$ \\
\addlinespace
24 & 0.8 & $-0.028\ [-0.100, 0.034]$ & $0.048\ [-0.041, 0.127]$ \\
24 & 1.6 & $0.001\ [-0.105, 0.097]$ & $0.091\ [-0.051, 0.214]$ \\
24 & 2 & $0.015\ [-0.104, 0.125]$ & $0.113\ [-0.051, 0.250]$ \\
24 & 3 & $0.026\ [-0.125, 0.160]$ & $0.161\ [-0.027, 0.325]$ \\
\addlinespace
32 & 0.8 & $-0.113\ [-0.267, 0.020]$ & $-0.003\ [-0.175, 0.147]$ \\
32 & 1.6 & $-0.152\ [-0.368, 0.039]$ & $0.062\ [-0.206, 0.300]$ \\
32 & 2 & $-0.159\ [-0.387, 0.045]$ & $0.102\ [-0.193, 0.367]$ \\
32 & 3 & $-0.152\ [-0.409, 0.088]$ & $0.170\ [-0.165, 0.486]$ \\
\bottomrule
\end{tabular}
\end{table}

\section{NFT implementation and gold-anchor analysis}
\label{app:nftdetails}

\subsection{Collection, rewards, and forward corruption}

Each collection round draws 24 prompt groups from the applicable training population; coding uses the screened reward cohort in Appendix~\ref{app:nfttraining}. The old policy produces 15 candidates per group using 32 Euler denoising steps, synchronous identity clocks, the logit-normal-quantile grid of Appendix~\ref{app:inferenceclocks}, initial noise scale two, prompt CFG one, and SCCFG two. Self-conditioning is recurrent during collection. Terminal token decoding is greedy and uses the same SCCFG scale. We save the terminal latent endpoint and its decoded answer length; generated text is used for reward evaluation without re-encoding it through Qwen. Each group additionally receives one gold endpoint from the fixed answer cache. No Qwen Transformer forward is required during NFT.

The generated reward is $0.75$ for a correct final answer and zero otherwise, while the gold reward is one. GSM8K uses numerical-answer correctness. MATH uses Math-Verify 0.9.0 with a five-second timeout; unparsed predictions, exceptions, and timeouts receive zero. Coding correctness requires all retained native training assertions to pass (Appendix~\ref{app:nfttraining}). All-correct generated groups are excluded before computing normalization statistics. For retained group $b$ and candidate $m$, define
\begin{equation}
 A_{bm}=\operatorname{clip}\!\left(
 \frac{R_{bm}-\bar R_b}{\sigma_R+10^{-4}},-5,5\right),
 \qquad
 \rho_{bm}=\operatorname{clip}\!\left(\frac12+\frac{A_{bm}}{10},0,1\right),
 \label{eq:nftrewardnormalization}
\end{equation}
where $\bar R_b$ averages all 16 endpoints and $\sigma_R$ is the population standard deviation of the retained \emph{raw} rewards across the complete collection round. Discarded groups do not enter these statistics. All-incorrect groups remain informative because they include a gold endpoint. The GSM8K implementations additionally stop before an update if an entire round has zero correct generations; the MATH and coding implementations disable that round-level guard and permit gold-anchored updates. A round with no retained groups supplies no optimizer update.

Correct generations remain in mixed groups. Because gold has a higher raw reward than correct generations, even all-correct groups would supply a gold-versus-generation preference if retained. Filtering removes those groups to concentrate updates on observed failures. The higher gold reward favors the reference reasoning trace within the latent objective, whereas generated-solution rewards depend only on final-answer correctness.

For each retained endpoint, draw eight independent time/noise/self-conditioning records, yielding at most $24\times16\times8=3072$ records per update. Each record draws $t=\operatorname{sigmoid}(-1.5+0.8\xi)$, with $\xi\sim\mathcal N(0,1)$, independently of the rollout grid. All representation groups use this same time, with unit group weights. Corrupt answer latents using Eq.~\eqref{eq:methodinterpolant} and convert predictions and endpoint targets using $\Delta_t=\max(1-t,0.05)$. Each policy restores its own clean prompt states, while endpoint, noise, answer states, and masks are shared across roles. Masks include the first EOS token and exclude prompt and subsequent padding positions.

\subsection{SCCFG construction and exact loss}

For each record, sample $b_{\mathrm{sc}}\sim\mathrm{Bernoulli}(0.5)$ and $\log(1+g)\sim\mathrm{Uniform}(\log1.5,\log6)$. Use the same $b_{\mathrm{sc}},g$ for current, old, and reference roles. Each role first predicts a clean endpoint under zero answer self-conditioning, restores its own clean prompt, and detaches that prediction. The main forward receives $b_{\mathrm{sc}}$ times the bootstrap prediction in answer positions; prompt positions remain clean regardless of $b_{\mathrm{sc}}$. The guidance-scale input $g$ is supplied to both passes. Convert both predictions to velocities and apply Eq.~\eqref{eq:nftcorrectedfield}.

Only the current main forward carries gradients. Bootstrap predictions, the entire guidance correction, old/reference computations, sampled endpoints, and data targets are detached, giving $\nabla_\Theta\widetilde v_{\mathrm{cur}}=\nabla_\Theta v_{\mathrm{cur}}^{\mathrm{main}}$ under the implemented differentiation. The current prompt encoder receives the main-forward chain-rule gradient, including its use as the clean condition. Its output gradients are accumulated across records and propagated through the shared prompt graph. This implements the same detached construction as the ELF guidance-adjusted target: $\widetilde v_{\mathrm{cur}}-u_t^{\mathrm{stab}}$ equals $v_{\mathrm{cur}}^{\mathrm{main}}-\sg[u_t^{\mathrm{stab}}+b_{\mathrm{sc}}(1-g^{-1})(v_{\mathrm{cur}}^{\mathrm{main}}-v_{\mathrm{cur}}^0)]$ when $u_t^{\mathrm{stab}}$ is detached. Removing the stop-gradient from the correction would change the update.

Let $b$ index retained training records and $i$ full-canvas token positions. With $d=1024$, define $N_{\mathrm{tok}}=\sum_{bi}M_{bi}$ across all ranks and microbatches of the update. Define
\begin{equation}
 \mathcal S(e;\omega)=\sum_{bi}M_{bi}\omega_b\frac{\|e_{bi}\|_2^2}{d}.
 \label{eq:nfttokensum}
\end{equation}
With $e_c=\widetilde v_{\mathrm{cur}}-\sg(u_t^{\mathrm{stab}})$ and $e_o=\sg(\widetilde v_{\mathrm{old}})-\sg(u_t^{\mathrm{stab}})$, the beta-one residuals are $e_+=e_c$ and $e_-=2e_o-e_c$. The executed loss is
\begin{equation}
 \mathcal L_{\mathrm{NFT}}=
 \frac{5}{N_{\mathrm{tok}}}\bigl[\mathcal S(e_+;\rho)+\mathcal S(e_-;1-\rho)\bigr]
 +\frac{0.1}{N_{\mathrm{tok}}}\mathcal S\bigl(\widetilde v_{\mathrm{cur}}-\sg(\widetilde v_{\mathrm{ref}});1\bigr).
 \label{eq:nftexecutedloss}
\end{equation}
This averages coordinates and valid tokens, rather than averaging a separate mean for each solution. Gradients are summed across ranks using the same update-wide denominator. The reference term is a squared-velocity penalty. We omit the adaptive residual normalization of the original DiffusionNFT implementation \citep{diffusionnft}; no decoder CE, prompt MSE, or separate correct-solution flow auxiliary loss is added.

\subsection{Policy roles and old-policy updates}
\label{sec:nftjoint}

Appendix~\ref{app:nfttraining} specifies initialization, trainable components, and optimizer settings.

The parameterization in Section~\ref{sec:conditionalencoding} makes the question tokens the external condition for all roles. Thus the current, old, and reference policies use their respective encoders $P_\phi$, $P_{\phi_{\mathrm{old}}}$, and $P_{\phi_{\mathrm{ref}}}$ while being compared under the same question $q$.

After successful update $n$, both old-policy components are updated by
\begin{equation}
 \Theta_{\mathrm{old}}\leftarrow \alpha_n\Theta_{\mathrm{old}}+(1-\alpha_n)\Theta_{\mathrm{cur}},
 \qquad \alpha_n=\begin{cases}\min(0.001n,0.5),&n<400,\\0.4,&n\ge400.\end{cases}
 \label{eq:nftoldema}
\end{equation}
The reference remains at initialization. The evaluation EMAs specified in Appendix~\ref{app:nfttraining} are distinct from this old-policy retention schedule.

\subsection{Proof and interpretation of gold anchoring}
\label{app:nftgoldproof}

\begin{proof}[Proof of Lemma~\ref{lem:goldanchor}]
All expectations below condition on $\mathcal H$ and a gold endpoint, so $f,h,\rho_G$ are fixed. Let $\sigma_G^2=\E[\|U-m_G\|_2^2\mid\mathcal H,\mathrm{gold}]$. Conditional variance decomposition gives
\begin{equation}
 \E[\ell_1(f,h;U,\rho_G)\mid\mathcal H,\mathrm{gold}]
 =\rho_G\|f-m_G\|_2^2
 +(1-\rho_G)\|f-(2h-m_G)\|_2^2+\sigma_G^2.
\end{equation}
For any vectors $a,b$, the identity $\omega\|f-a\|^2+(1-\omega)\|f-b\|^2=\|f-[\omega a+(1-\omega)b]\|^2+\omega(1-\omega)\|a-b\|^2$ completes the square. Substituting $a=m_G$, $b=2h-m_G$, and $\omega=\rho_G$ yields Eq.~\eqref{eq:nftgoldanchor}, with
\begin{equation}
 C_{\mathcal H}=\sigma_G^2+4\rho_G(1-\rho_G)\|m_G-h\|_2^2.
\end{equation}
The anchor weights $2\rho_G-1$ and $2(1-\rho_G)$ are nonnegative and sum to one.
\end{proof}

\paragraph{Random coefficients and reference interpretation.}
In the executed reward mapping, the gold raw reward exceeds its group's mean, so $\rho_G\in(1/2,1]$. Its value depends on the group's generations and the retained population. The conditioning information $\mathcal H$ therefore includes the realized coefficient and any retention information used to define the population; $m_G$ is computed under that same conditional distribution. The lemma does not assume coefficient--target independence. If that independence holds given state, time, and question, the coefficient can be removed from the conditioning of $m_G$. Identifying $m_G$ with a reference field requires the stronger idealization that the reference matches this gold posterior target under the same conditioning and corruption law. Correct final answers alone do not imply that equality. If the reference and old fields both equal $m_G$, the gold anchor equals the reference for every allowed $\rho_G$.

\paragraph{Relation to NFT and detached guidance.}
The moving old policy centers NFT's implicit positive/negative construction, while our fixed reference appears in the additional penalty of Eq.~\eqref{eq:nftexecutedloss}. Lemma~\ref{lem:goldanchor} characterizes a third source of anchoring supplied by gold data: attraction toward the gold posterior target, interpolated with the old field. It holds for the field-level quadratic with auxiliary inputs fixed. In the SCCFG implementation, differentiation treats the guidance correction as fixed within each update, even though it is recomputed from the current network on subsequent forwards. The lemma consequently interprets the local corrected-field regression. Group filtering changes the training population, and the information in self-conditioning differs between collection and forward-noising updates. Our benchmark comparisons assess the resulting complete procedure empirically.

\section{Evaluation protocols and uncertainty}
\label{app:experimentdetails}

\subsection{Benchmarks, checkpoints, and scoring}

\paragraph{Benchmarks and scoring.}
Each generation seed evaluates all 1,319 GSM8K test problems or all 500 MATH500 problems in fixed benchmark order. GSM8K accuracy compares extracted final answers with gold answers by numerical equality. Extraction prioritizes the last boxed number, then explicit answer markers, then the last number; normalization removes formatting and parses decimals and fractions. Math-Verify is the primary MATH scorer throughout, including the schedule comparison in Table~\ref{tab:inferenceclocks}; PRM800K scores are retained as a secondary diagnostic. Validation comparisons use the separate fixed 1,000-problem GSM8K and 500-problem MATH panels described in Appendix~\ref{app:inferenceclocks}. Checkpoint, EMA, clock, guidance, and seed choices are specified with each comparison.

\paragraph{Math checkpoints and generation settings.}
The headline and inference-allocation math evaluations use DMD2-style sampling, a 1,024-token canvas, and async clocks $(2.5,2,1.5)$. Table~\ref{tab:trainingstages} lists the supervised and NFT endpoints. Supervised evaluations select backbone/prompt EMA $.9999/.9999$; NFT evaluations select $.99/.99$. Table~\ref{tab:headlinecomparison} uses eight sampling seeds at 64 denoising steps. Table~\ref{tab:headlinelownfe} uses 64, 32, 16, and 8 seeds at 8, 16, 32, and 64 steps, respectively. Our rows in both tables use SCCFG 2 for GSM8K and SCCFG 3 for MATH500. The inference-allocation figures use the same SCCFG 2/3 settings as the headline tables. Representation and prompt-learning ablations retain ODE; Table~\ref{tab:inferenceclocks} compares schedules across all three samplers.

\subsection{External reports and parameter accounting}
Tables~\ref{tab:headlinecomparison}--\ref{tab:headlinelownfe} use the comparison papers' reported accuracies and sampling budgets, without re-evaluating their models. GSM8K uses 4-shot prompts for Qwen/LLaDA, 8-shot prompts for Dream/TESS~2, and zero-shot prompts for ours. FMLM+ grades Python solutions, whereas ELF-REG and CEDR produce natural-language reasoning. FMLM+ values come from its Table~8 \citep{posterior}.

ELF-REG comparisons use arXiv v3 \citep{scalingdlm}, selecting the strongest reported sampler setting at each model, benchmark, and displayed pass@1 budget, rather than restricting to its ODE headlines. Table~\ref{tab:elfregselected} records every selected setting. Its main sampler comparison (Table~4) uses four draws and stochastic clean-prediction re-noising with power exponent $p=0.5$ and noise multiplier $\eta=1.2$. Two entries use other settings: GSM8K-L at 63 denoising calls uses $p=1$, $\eta=1.2$ from its Table~30 (61.28\%, versus 60.37\% at $p=0.5$); MATH500 at seven calls uses the four-seed ODE result from its Table~4 (9.85\%, versus 9.15\% stochastic). ELF-REG's higher-pass results retain its 16-draw ODE protocol, from v3 Tables~18 (code) and~20 (MATH500); its stochastic study supports only pass@1/4. Its samplers retain prefix early stopping.

The $\mathbb S$-FLM row uses its enhanced hyperspherical configuration with top-1 velocity decoding, rather than its basic sampler. MLFM's 31.24\% result uses context-corrupted guidance and online token promotion: its reported 256 sampling steps need not equal 256 network calls, so Table~\ref{tab:headlinecomparison} marks NFE as unestablished. SDAR uses four steps per four-token block. ELF-REG's reported budget includes one terminal decoder call, so the pass@1 tables mark it with $^{*}$; the higher-pass tables give the denoising count directly. Our NFE excludes prompt encoding and terminal decoding.

\begin{table}[!htbp]
\centering\small
\setlength{\tabcolsep}{5pt}
\caption{\textbf{Selected ELF-REG pass@1 results and provenance.} Values are author-reported percentages from arXiv v3 \citep{scalingdlm}; SD is across generation seeds. NFE excludes the paper's terminal decoder call. Stochastic rows use clean-prediction re-noising with $\eta=1.2$ and the stated power exponent $p$. Mathematical and aliased coding rows use four draws; standard coding rows use 16. The final column cites the corresponding table in ELF-REG v3. No stochastic standard-scoring code results are reported.}
\label{tab:elfregselected}
\begin{tabular}{@{}lrlrrrr@{}}
\toprule
Benchmark / model & NFE & Sampler & $p$ & Pass@1 & SD & ELF-REG table\\
\midrule
GSM8K / B & 7 & Stochastic & 0.5 & 26.88 & 0.73 & Table~4\\
 & 15 & Stochastic & 0.5 & 38.53 & 0.99 & Table~4\\
 & 31 & Stochastic & 0.5 & 41.17 & 0.85 & Table~4\\
 & 63 & Stochastic & 0.5 & 42.97 & 0.35 & Table~4\\
\addlinespace
GSM8K / L & 7 & Stochastic & 0.5 & 48.69 & 1.21 & Table~4\\
 & 15 & Stochastic & 0.5 & 56.90 & 0.55 & Table~4\\
 & 31 & Stochastic & 0.5 & 59.95 & 0.47 & Table~4\\
 & 63 & Stochastic & 1 & 61.28 & 0.42 & Table~30\\
\addlinespace
MATH500 / L & 7 & ODE & --- & 9.85 & 1.00 & Table~4\\
 & 15 & Stochastic & 0.5 & 14.00 & 0.95 & Table~4\\
 & 31 & Stochastic & 0.5 & 17.25 & 1.36 & Table~4\\
 & 63 & Stochastic & 0.5 & 20.00 & 0.73 & Table~4\\
 & 127 & Stochastic & 0.5 & 20.15 & 1.29 & Table~4\\
\midrule
\multicolumn{7}{@{}l}{\textit{Code / L: standard scoring}}\\
HumanEval & 127 & ODE & --- & 22.26 & 2.30 & Table~15\\
HumanEval+ & 127 & ODE & --- & 21.15 & 2.13 & Table~15\\
MBPP-378 base & 127 & ODE & --- & 24.21 & 1.17 & Table~15\\
\midrule
\multicolumn{7}{@{}l}{\textit{Code / L: function-name aliasing}}\\
HumanEval & 127 & Stochastic & 0.5 & 33.08 & 3.16 & Table~4\\
HumanEval+ & 127 & Stochastic & 0.5 & 31.40 & 3.61 & Table~4\\
MBPP-378 base & 127 & Stochastic & 0.5 & 38.76 & 1.09 & Table~4\\
\bottomrule
\end{tabular}
\end{table}

Parameter counts include frozen components. Our inherited Qwen vocabulary input/output maps contain 545M parameters and are called once per response; the denoising backbones contain 90.4M/638M parameters. The $(+X)$ counts include the learned prompt encoder and terminal projection, also used once per response. ELF-REG includes its frozen Qwen prompt encoder, and its reported non-decoder counts retain training-only REPA projectors. MLFM includes unmerged adapters; SDAR counts both untied vocabulary maps. These conventions describe stored parameters and invocation frequency.

\subsection{Oracle pass@k and majority voting}
\label{app:passkprotocol}

With $N$ benchmark problems, $n$ saved draws per problem, and $C_b$ correct draws for problem $b$, the oracle estimator is
\begin{equation}
 \widehat{\operatorname{pass@}k}
 =\frac{1}{N}\sum_{b=1}^{N}
 \left[1-\frac{\binom{n-C_b}{k}}{\binom{n}{k}}\right],
 \qquad \binom{n-C_b}{k}=0\quad\text{if }n-C_b<k.
\end{equation}
Each depth $S\in\{8,16,32,64\}$ has $n=512/S$ saved draws per problem, giving $64,32,16,8$ draws, respectively. Thus every available draw contributes to each supported $k$. At $k=n$, this becomes the fraction of problems with any correct saved answer.

\paragraph{Voting and answer equivalence.}
We use ``majority vote'' for the plurality rule, which does not require more than half the samples to agree. We average voting accuracy over uniformly selected $k$-subsets of the same draws. Numerical normalization groups GSM8K answers; the pinned mutual Math-Verify comparison groups MATH500 answers. Invalid extractions abstain, and an all-invalid subset fails. Ties among the largest answer groups are resolved uniformly, with a uniformly selected representative within the chosen group; the reported accuracy averages over these choices. Gold correctness labels score the selected representative but do not determine the selection. Equivalence partitions permit exact hypergeometric subset averaging. For nontransitive MATH comparisons, we recluster within each subset, enumerating up to 65,536 subsets and otherwise using a fixed 16,384-subset approximation. At $k=1$, voting and oracle accuracy coincide; with two valid answers, uniform tie-breaking also prevents a systematic voting advantage over a single draw.

\subsection{Uncertainty and reporting conventions}
\label{app:uncertainty}

\paragraph{Sampling-seed variation.}
For seed accuracies $A_1,\ldots,A_n$, expressed as percentages, we report the mean $\bar A$ and sample standard deviation
\[
  \operatorname{SD}_{\mathrm{seed}}=\sqrt{\frac{1}{n-1}\sum_{m=1}^{n}(A_m-\bar A)^2}.
\]
These values describe generation randomness at a fixed checkpoint and test set. Problem-bootstrap uncertainty instead quantifies variation across test problems.

Bars in the representation learning curves are $\pm1$ seed SD; small parenthesized values in the corresponding tables report the same statistic. External baselines retain author-reported accuracies.

\paragraph{Problem-bootstrap uncertainty.}
Small brackets in Tables~\ref{tab:mathhigherpass} and~\ref{tab:codinghigherpass}, and parentheses in Tables~\ref{tab:budgetvaluesgsml}--\ref{tab:budgetvaluesgsmb}, Tables~\ref{tab:samplerbudgetgsmlode}--\ref{tab:samplerbudgetgsmbedm}, and Tables~\ref{tab:codingstandarddetail}--\ref{tab:codingaliasdetail} report the standard deviation of the benchmark mean across 4,000 problem-bootstrap replicates, using bootstrap seed 20260918. Each replicate resamples whole problems with replacement while preserving their complete draw sets. These problem-bootstrap standard errors quantify uncertainty across problems, conditional on the trained checkpoints, sampler choices, saved generations, and fixed voting-subset approximation. At the largest $k$, each resampled problem contributes its entire saved draw pool. All values are rounded independently from full precision.

\paragraph{Confidence intervals and paired comparisons.}
A 95\% CI is a pointwise percentile interval from 4,000 problem-bootstrap replicates. This convention applies to the prompt-curriculum and reconstruction-probe figures. For pre/post comparisons, we resample the same problem indices in both stages and form the difference within each replicate. Pointwise 95\% percentile intervals for changes therefore use paired differences. These measures condition on fixed trained checkpoints and exclude variation across independent training runs.

\section{Additional mathematical-reasoning results}
\label{app:mathresults}

This section collects sampling-seed variation and numerical inference-allocation results for the main-body comparisons. Appendix~\ref{app:experimentdetails} defines the protocols and uncertainty measures.

\subsection{Variation across evaluation seeds}
\label{app:seedvariability}

Table~\ref{tab:headlineseedsd} reports standard deviations across independent sampling seeds for all our settings in Tables~\ref{tab:headlinecomparison} and~\ref{tab:headlinelownfe}, using DMD2-style sampling throughout. Each seed generates one solution per problem across the full fixed test set. We use every completed sampling seed at each denoising budget: $n=64,32,16,8$ for NFE $8,16,32,64$, respectively. Checkpoints, guidance, scorers, and benchmark populations match the main tables.

For the other main-body comparisons, Tables~\ref{tab:inferenceclockfull} and~\ref{tab:stochasticclockdetail} in Appendix~\ref{app:schedulecomparisons} give the ODE and stochastic schedule results, respectively, with the same seed SDs as Table~\ref{tab:inferenceclocks}; Tables~\ref{tab:codingstandarddetail}--\ref{tab:codingaliasdetail} in Appendix~\ref{app:codingdetails} give pass@1 seed SDs in their final columns for the standard and alias-scored blocks of Table~\ref{tab:codingcomparison}, respectively. Their pass@$k$ columns additionally report problem-bootstrap SEs. Table~\ref{tab:representationcurvevalues} in Appendix~\ref{app:representationclock} also includes the four-seed epoch-six summary discussed in Sections~\ref{sec:representationablations} and~\ref{sec:conditionalencoding}. These seed statistics cover our evaluations; external baselines retain the authors' reported accuracies.

\begin{table}[htbp]
\caption{Mean accuracy $\pm$ standard deviation across sampling seeds, in percentage points. Each seed evaluates the complete benchmark at the same fixed checkpoint. Columns use 64, 32, 16, and 8 seeds, respectively. The 64-NFE column supplies the results in Table~\ref{tab:headlinecomparison}; SCCFG 3 supplies the MATH500 rows in Tables~\ref{tab:headlinecomparison} and~\ref{tab:headlinelownfe}. All rows and Figure~\ref{fig:headlinestages} use DMD2-style sampling.}
\label{tab:headlineseedsd}
\centering\small
\setlength{\tabcolsep}{4pt}
\begin{tabular}{@{}lrrrr@{}}
\toprule
Model / stage & 8 NFE & 16 NFE & 32 NFE & 64 NFE\\
Sampling seeds $n$ & 64 & 32 & 16 & 8\\
\midrule
\multicolumn{5}{@{}l}{\textit{GSM8K, CEDR-B, SCCFG 2}}\\
Pre-NFT & $24.95\pm0.68$ & $38.25\pm0.86$ & $44.39\pm0.67$ & $47.62\pm0.90$\\
Post-NFT & $28.48\pm0.79$ & $40.60\pm0.81$ & $45.59\pm0.64$ & $47.00\pm0.98$\\
\addlinespace
\multicolumn{5}{@{}l}{\textit{GSM8K, CEDR-L, SCCFG 2}}\\
Pre-NFT & $54.23\pm0.94$ & $62.57\pm0.70$ & $65.71\pm0.69$ & $67.13\pm0.58$\\
Post-NFT & $57.93\pm0.84$ & $64.96\pm0.78$ & $67.59\pm0.57$ & $68.97\pm0.50$\\
\addlinespace
\multicolumn{5}{@{}l}{\textit{MATH500, CEDR-L, SCCFG 3}}\\
Pre-NFT & $16.70\pm1.23$ & $21.75\pm1.48$ & $25.76\pm1.17$ & $28.58\pm0.90$\\
Post-NFT & $18.82\pm1.20$ & $25.26\pm1.32$ & $29.80\pm0.56$ & $32.40\pm1.17$\\
\addlinespace
\bottomrule
\end{tabular}
\end{table}

\subsection{CEDR-B inference allocation before and after NFT}
\label{app:elfbbudgets}

Figure~\ref{fig:elfbstages} extends the CEDR-L comparison in Figure~\ref{fig:headlinestages} to the smaller CEDR-B backbone on GSM8K. It uses supervised epoch 18 and NFT update 300, with the same asynchronous clocks and SCCFG 2. Across both model sizes, the allocation curves distinguish the coverage of parallel samples from the accuracy attainable by majority voting.

\begin{figure}[!htbp]
\centering
\includegraphics[width=.9\linewidth]{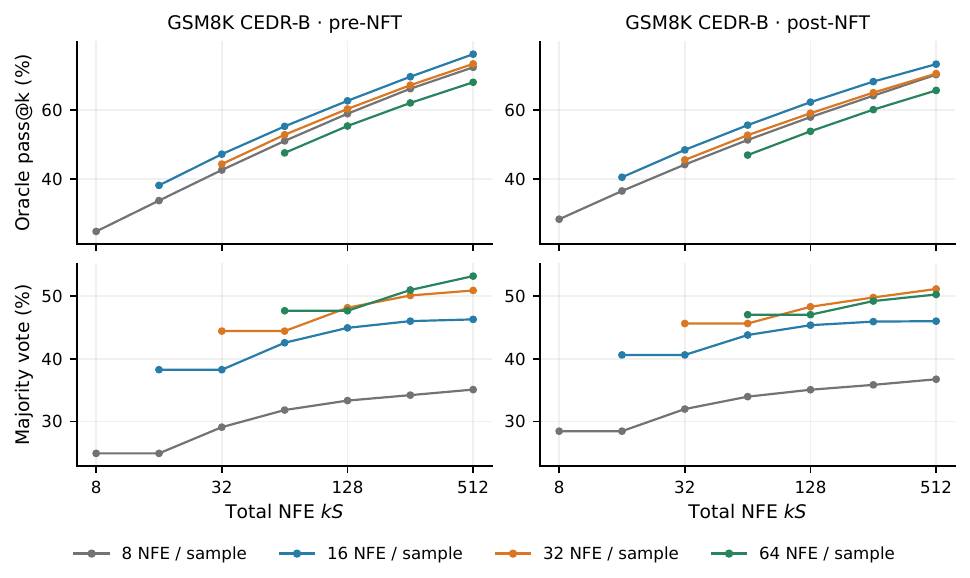}
\caption{GSM8K CEDR-B before and after NFT with DMD2-style sampling. Upper panels show oracle pass@$k$; lower panels show majority-vote accuracy. Colors denote denoising NFE per sample. The same 64/32/16/8 sample cohorts and total budget of 512 denoising calls are used as in Figure~\ref{fig:headlinestages}. Table~\ref{tab:budgetvaluesgsmb} gives every plotted value with problem-bootstrap SEs.}
\label{fig:elfbstages}
\end{figure}

\FloatBarrier
\subsection{Inference allocation across sampling algorithms}
\label{app:samplerallocations}

Figures~\ref{fig:samplerallocationgsml}--\ref{fig:samplerallocationgsmb} compare ODE, EDM-style index churn ($0.7$), and DMD2-style sampling at the fixed supervised and NFT endpoints. Columns fix denoising NFE per sample; rows separate oracle pass@$k$ and majority voting before and after NFT. The horizontal axis is the total denoising budget $kS$. We use all 64/32/16/8 saved draws at 8/16/32/64 NFE, respectively, up to 512 total calls, with identical ordered seed cohorts across samplers. Clocks remain $(2.5,2,1.5)$; GSM8K uses SCCFG 2 and MATH500 SCCFG 3. Checkpoints, EMA selectors, prompts, scorers, and voting rules match the headline comparisons in Appendix~\ref{app:experimentdetails}. These comparisons include each sampler's integration and terminal-state convention (Appendix~\ref{app:samplerimplementation}).

At 64 NFE, both stochastic samplers improve pass@1 over ODE at all six endpoints, but the ranking can change with the sample count and selection rule. For GSM8K CEDR-L after NFT, eight 64-step draws give ODE pass@8 of 83.02\%, compared with 80.21\% for EDM-style and 80.82\% for DMD2-style sampling. Yet DMD2-style majority voting reaches 71.81\%, versus ODE's 70.58\%. For MATH500 after NFT, the same allocation yields DMD2-style voting accuracy of 38.80\%, versus ODE's 33.38\%; at eight steps and 64 draws, ODE voting instead reaches 28.63\%, versus 27.80\% for DMD2-style sampling. Thus improvements in single-sample accuracy do not imply uniform improvements in oracle coverage or voting at every allocation. Our headline sampler remains fixed across all budgets.

\begin{figure}[p]
\centering
\includegraphics[width=\linewidth]{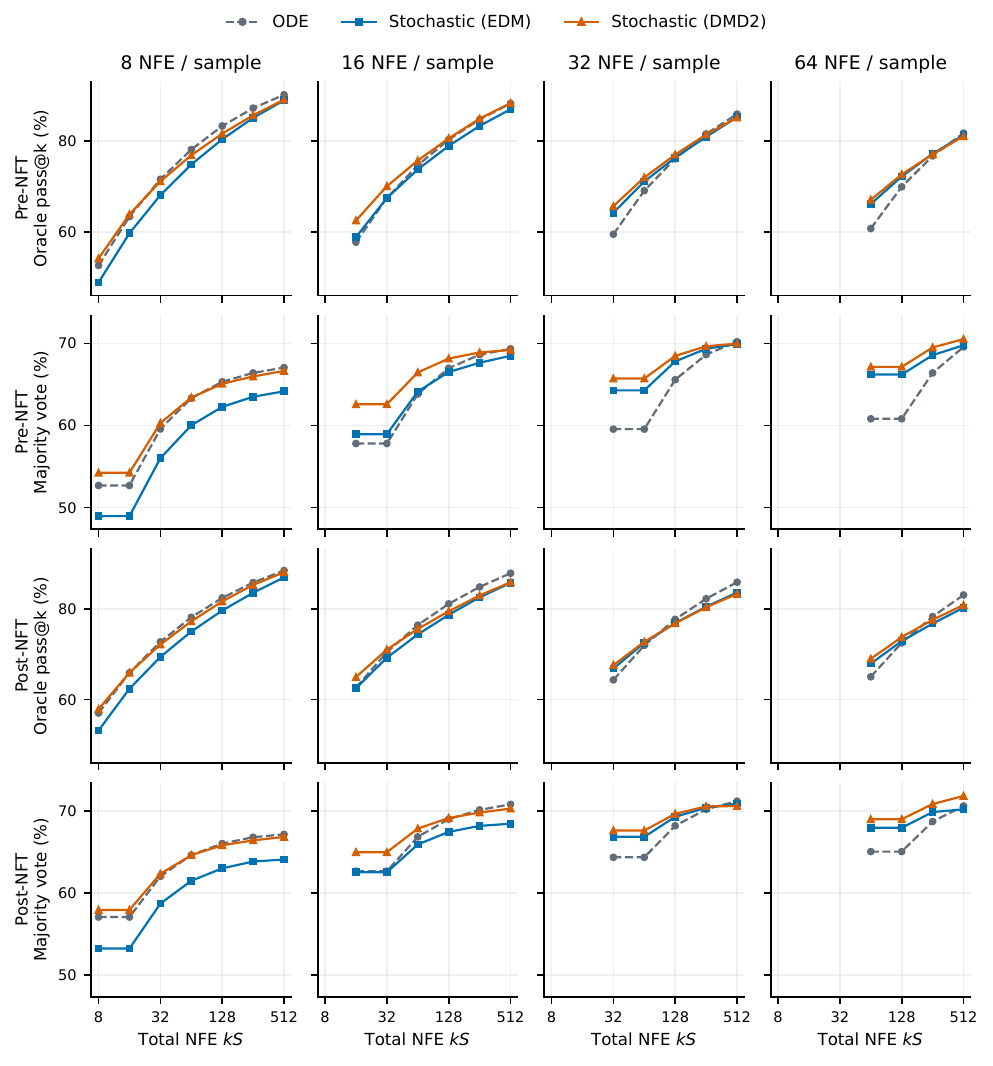}
\caption{\textbf{GSM8K, CEDR-L: inference allocation across samplers.} Pre-NFT is epoch 21; post-NFT is NFT update 600. Columns fix NFE per sample; rows show pre-/post-NFT oracle and majority-vote accuracy. Colors and markers distinguish samplers; axes match across samplers and training stages within each metric. Numerical values and uncertainty: Tables~\ref{tab:samplerbudgetgsmlode}--\ref{tab:samplerbudgetgsmledm} for ODE/EDM, and Table~\ref{tab:budgetvaluesgsml} for DMD2.}
\label{fig:samplerallocationgsml}
\end{figure}

\begin{figure}[p]
\centering
\includegraphics[width=\linewidth]{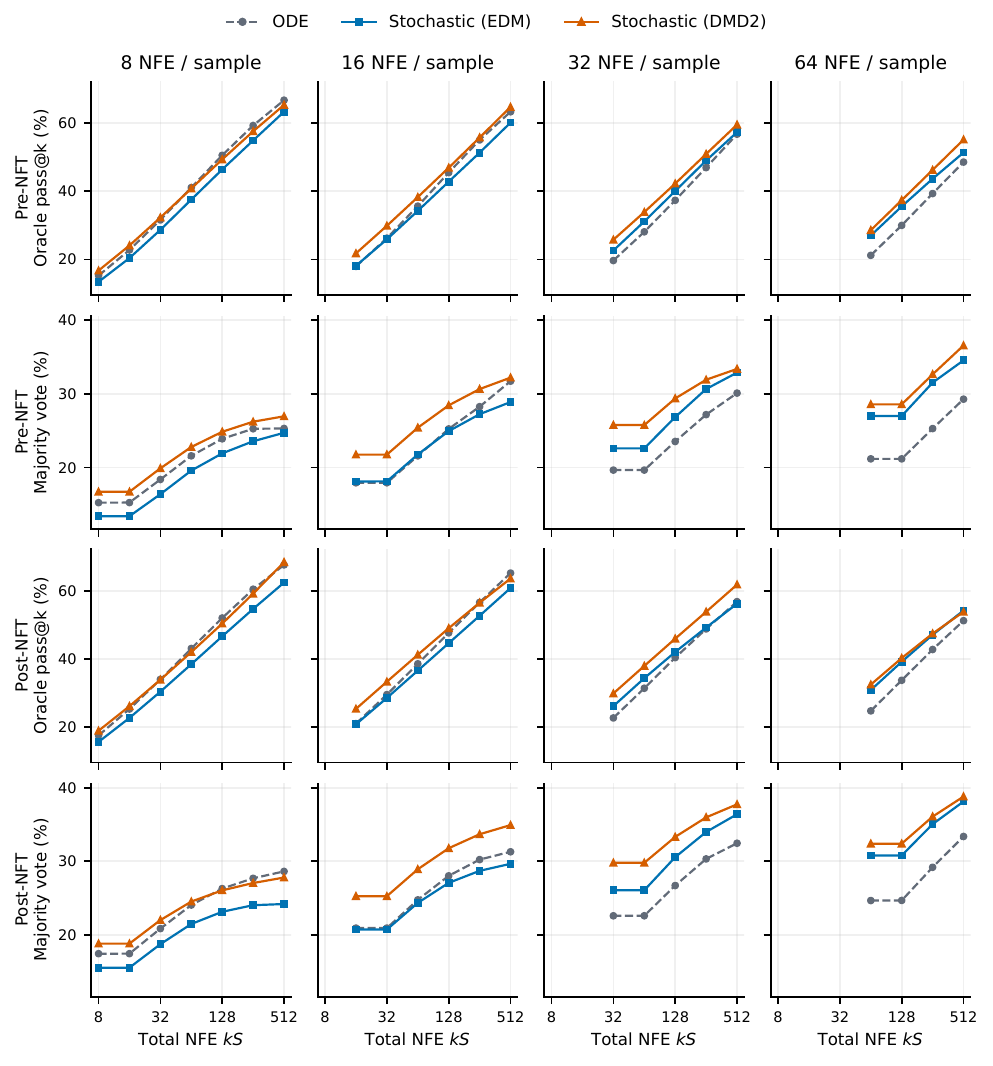}
\caption{\textbf{MATH500, CEDR-L: inference allocation across samplers.} Pre-NFT is epoch 21; post-NFT is NFT update 600. Columns fix NFE per sample; rows show pre-/post-NFT oracle and majority-vote accuracy. Colors and markers distinguish samplers; axes match across samplers and training stages within each metric. Numerical values and uncertainty: Tables~\ref{tab:samplerbudgetmathlode}--\ref{tab:samplerbudgetmathledm} for ODE/EDM, and Table~\ref{tab:budgetvaluesmathl} for DMD2.}
\label{fig:samplerallocationmathl}
\end{figure}

\begin{figure}[p]
\centering
\includegraphics[width=\linewidth]{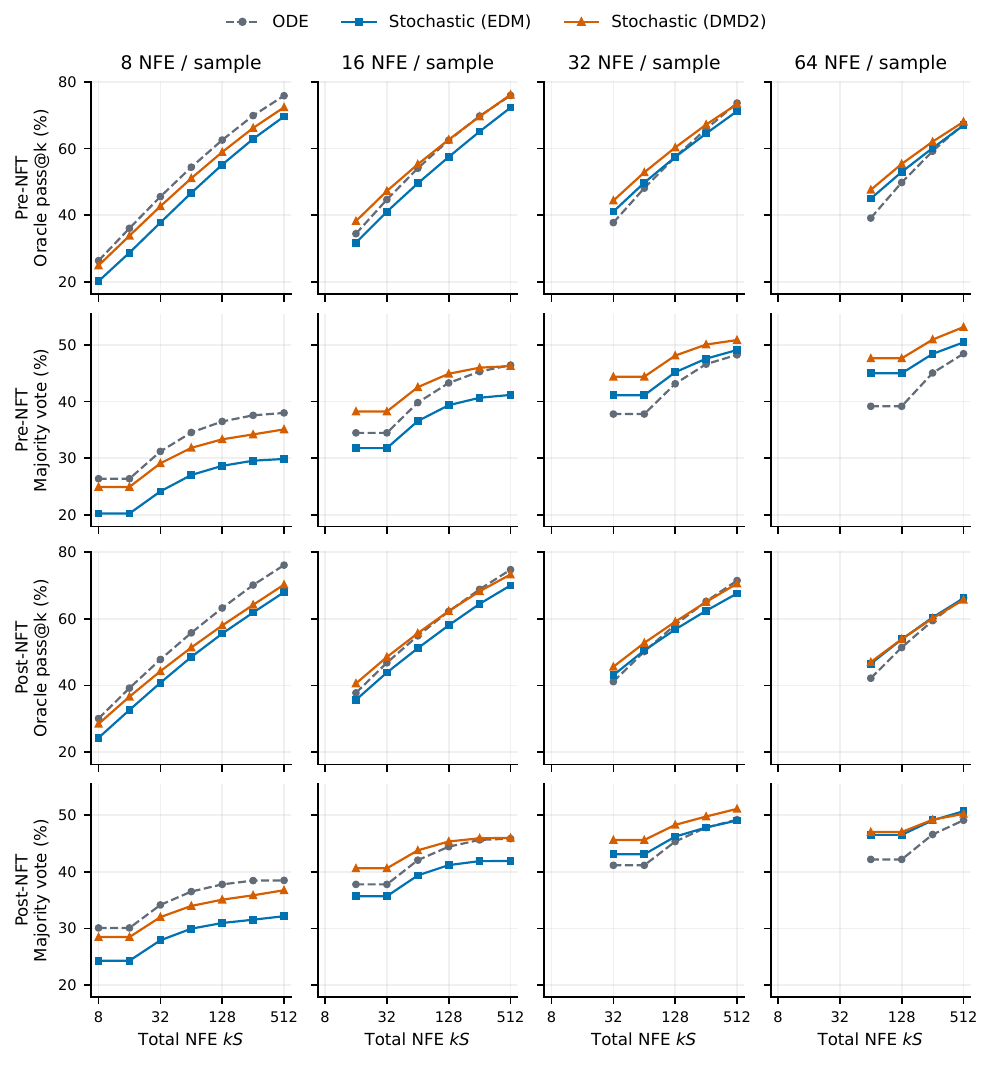}
\caption{\textbf{GSM8K, CEDR-B: inference allocation across samplers.} Pre-NFT is epoch 18; post-NFT is NFT update 300. Columns fix NFE per sample; rows show pre-/post-NFT oracle and majority-vote accuracy. Colors and markers distinguish samplers; axes match across samplers and training stages within each metric. Numerical values and uncertainty: Tables~\ref{tab:samplerbudgetgsmbode}--\ref{tab:samplerbudgetgsmbedm} for ODE/EDM, and Table~\ref{tab:budgetvaluesgsmb} for DMD2.}
\label{fig:samplerallocationgsmb}
\end{figure}

\FloatBarrier
\subsection{Numerical inference-allocation results}
\label{app:budgetvalues}

Tables~\ref{tab:budgetvaluesgsml}--\ref{tab:budgetvaluesgsmb} give every oracle and voting point in Figures~\ref{fig:headlinestages} and~\ref{fig:elfbstages}, using DMD2-style sampling and SCCFG 2 for GSM8K / 3 for MATH500. Each table pairs supervised and NFT checkpoints at identical denoising depths, sample counts, and ordered generation seeds. The underlying populations contain all 1,319 GSM8K or 500 MATH500 test problems.

Metrics and uncertainty are defined in Appendix~\ref{app:experimentdetails}.

\paragraph{Paired NFT improvements.}
For DMD2-style sampling at 64 steps, NFT changes CEDR-L pass@1 by $+1.84$ points ($[+0.99,+2.69]$) on GSM8K and $+3.83$ points ($[+2.18,+5.40]$) on MATH500. The corresponding CEDR-B change is $-0.63$ points ($[-1.60,+0.34]$). At 64 steps and eight samples, MATH500 oracle accuracy changes by $-1.40$ points ($[-4.80,+2.00]$) and voting by $+2.23$ points ($[-0.32,+4.79]$). These are paired 95\% problem-bootstrap intervals (Appendix~\ref{app:uncertainty}).

\paragraph{Interpreting NFT gains across allocations.}
At eight steps on GSM8K CEDR-L, NFT raises pass@1 from 54.23\% to 57.93\%, while pass@64 changes from 89.16\% to 88.02\%. On MATH500 at the same allocation, oracle accuracy rises from 65.40\% to 68.40\% and voting from 26.95\% to 27.80\%. Thus oracle and voting improvements both depend on the allocation. For independent draws with problem-specific success probability $p_i$, oracle success is $1-(1-p_i)^k$, so gains saturate on problems already solved frequently. Voting instead depends on the correct answer's frequency relative to competing answer groups. Changes in correctness probability, its distribution across problems, and answer-group frequencies can affect these metrics differently.

\begin{table}[!htbp]
\caption{GSM8K, CEDR-L: every point in Figure~\ref{fig:headlinestages}, with DMD2-style sampling and SCCFG 2. Entries are accuracy percentages with problem-bootstrap SEs in parentheses (Appendix~\ref{app:uncertainty}). Each denoising depth uses $n=512/S$ saved draws per problem; $k$-subset averaging uses all $n$.}
\label{tab:budgetvaluesgsml}
\centering\small
\setlength{\tabcolsep}{5pt}
\begin{tabular}{@{}rrr rrrr@{}}
\toprule
& & & \multicolumn{2}{c}{Oracle pass@$k$} & \multicolumn{2}{c}{Majority vote} \\
\cmidrule(lr){4-5}\cmidrule(l){6-7}
NFE $S$ & $k$ & $kS$ & Pre-NFT & Post-NFT & Pre-NFT & Post-NFT \\
\midrule
8 & 1 & 8 & 54.23 {\scriptsize (1.07)} & 57.93 {\scriptsize (1.12)} & 54.23 {\scriptsize (1.07)} & 57.93 {\scriptsize (1.12)} \\
8 & 2 & 16 & 63.92 {\scriptsize (1.10)} & 65.91 {\scriptsize (1.11)} & 54.23 {\scriptsize (1.07)} & 57.93 {\scriptsize (1.12)} \\
8 & 4 & 32 & 71.18 {\scriptsize (1.07)} & 72.13 {\scriptsize (1.08)} & 60.29 {\scriptsize (1.16)} & 62.32 {\scriptsize (1.18)} \\
8 & 8 & 64 & 76.92 {\scriptsize (1.01)} & 77.22 {\scriptsize (1.02)} & 63.41 {\scriptsize (1.19)} & 64.60 {\scriptsize (1.21)} \\
8 & 16 & 128 & 81.67 {\scriptsize (0.95)} & 81.60 {\scriptsize (0.96)} & 65.07 {\scriptsize (1.22)} & 65.80 {\scriptsize (1.23)} \\
8 & 32 & 256 & 85.72 {\scriptsize (0.89)} & 85.23 {\scriptsize (0.92)} & 65.97 {\scriptsize (1.25)} & 66.40 {\scriptsize (1.25)} \\
8 & 64 & 512 & 89.16 {\scriptsize (0.86)} & 88.02 {\scriptsize (0.90)} & 66.64 {\scriptsize (1.29)} & 66.82 {\scriptsize (1.28)} \\
\midrule
16 & 1 & 16 & 62.57 {\scriptsize (1.11)} & 64.96 {\scriptsize (1.14)} & 62.57 {\scriptsize (1.11)} & 64.96 {\scriptsize (1.14)} \\
16 & 2 & 32 & 70.10 {\scriptsize (1.09)} & 71.00 {\scriptsize (1.12)} & 62.58 {\scriptsize (1.11)} & 64.96 {\scriptsize (1.14)} \\
16 & 4 & 64 & 75.76 {\scriptsize (1.03)} & 75.56 {\scriptsize (1.07)} & 66.46 {\scriptsize (1.16)} & 67.83 {\scriptsize (1.19)} \\
16 & 8 & 128 & 80.63 {\scriptsize (0.96)} & 79.42 {\scriptsize (1.02)} & 68.15 {\scriptsize (1.19)} & 69.12 {\scriptsize (1.21)} \\
16 & 16 & 256 & 84.97 {\scriptsize (0.90)} & 82.91 {\scriptsize (0.97)} & 68.90 {\scriptsize (1.21)} & 69.79 {\scriptsize (1.22)} \\
16 & 32 & 512 & 88.40 {\scriptsize (0.88)} & 85.82 {\scriptsize (0.96)} & 69.19 {\scriptsize (1.26)} & 70.27 {\scriptsize (1.25)} \\
\midrule
32 & 1 & 32 & 65.71 {\scriptsize (1.13)} & 67.59 {\scriptsize (1.15)} & 65.71 {\scriptsize (1.13)} & 67.59 {\scriptsize (1.15)} \\
32 & 2 & 64 & 72.04 {\scriptsize (1.09)} & 72.72 {\scriptsize (1.12)} & 65.71 {\scriptsize (1.13)} & 67.59 {\scriptsize (1.15)} \\
32 & 4 & 128 & 77.06 {\scriptsize (1.05)} & 76.79 {\scriptsize (1.08)} & 68.47 {\scriptsize (1.17)} & 69.62 {\scriptsize (1.19)} \\
32 & 8 & 256 & 81.47 {\scriptsize (1.00)} & 80.31 {\scriptsize (1.04)} & 69.65 {\scriptsize (1.19)} & 70.54 {\scriptsize (1.22)} \\
32 & 16 & 512 & 85.22 {\scriptsize (0.99)} & 83.24 {\scriptsize (1.04)} & 69.97 {\scriptsize (1.24)} & 70.58 {\scriptsize (1.26)} \\
\midrule
64 & 1 & 64 & 67.13 {\scriptsize (1.15)} & 68.97 {\scriptsize (1.16)} & 67.13 {\scriptsize (1.15)} & 68.97 {\scriptsize (1.16)} \\
64 & 2 & 128 & 72.71 {\scriptsize (1.12)} & 73.76 {\scriptsize (1.13)} & 67.13 {\scriptsize (1.15)} & 68.97 {\scriptsize (1.16)} \\
64 & 4 & 256 & 77.12 {\scriptsize (1.09)} & 77.55 {\scriptsize (1.10)} & 69.50 {\scriptsize (1.19)} & 70.82 {\scriptsize (1.20)} \\
64 & 8 & 512 & 81.12 {\scriptsize (1.09)} & 80.82 {\scriptsize (1.09)} & 70.52 {\scriptsize (1.24)} & 71.81 {\scriptsize (1.24)} \\
\bottomrule
\end{tabular}
\end{table}
\begin{table}[!htbp]
\caption{MATH500, CEDR-L: every point in Figure~\ref{fig:headlinestages}, with DMD2-style sampling and SCCFG 3. Entries are accuracy percentages with problem-bootstrap SEs in parentheses (Appendix~\ref{app:uncertainty}). Each denoising depth uses $n=512/S$ saved draws per problem; $k$-subset averaging uses all $n$.}
\label{tab:budgetvaluesmathl}
\centering\small
\setlength{\tabcolsep}{5pt}
\begin{tabular}{@{}rrr rrrr@{}}
\toprule
& & & \multicolumn{2}{c}{Oracle pass@$k$} & \multicolumn{2}{c}{Majority vote} \\
\cmidrule(lr){4-5}\cmidrule(l){6-7}
NFE $S$ & $k$ & $kS$ & Pre-NFT & Post-NFT & Pre-NFT & Post-NFT \\
\midrule
8 & 1 & 8 & 16.70 {\scriptsize (1.16)} & 18.82 {\scriptsize (1.28)} & 16.70 {\scriptsize (1.16)} & 18.82 {\scriptsize (1.28)} \\
8 & 2 & 16 & 24.07 {\scriptsize (1.45)} & 26.06 {\scriptsize (1.54)} & 16.71 {\scriptsize (1.16)} & 18.82 {\scriptsize (1.28)} \\
8 & 4 & 32 & 32.28 {\scriptsize (1.67)} & 33.83 {\scriptsize (1.72)} & 19.89 {\scriptsize (1.40)} & 22.02 {\scriptsize (1.52)} \\
8 & 8 & 64 & 40.83 {\scriptsize (1.82)} & 41.96 {\scriptsize (1.84)} & 22.79 {\scriptsize (1.57)} & 24.53 {\scriptsize (1.66)} \\
8 & 16 & 128 & 49.38 {\scriptsize (1.92)} & 50.34 {\scriptsize (1.90)} & 24.85 {\scriptsize (1.69)} & 26.05 {\scriptsize (1.77)} \\
8 & 32 & 256 & 57.64 {\scriptsize (1.99)} & 59.14 {\scriptsize (1.94)} & 26.22 {\scriptsize (1.81)} & 27.05 {\scriptsize (1.86)} \\
8 & 64 & 512 & 65.40 {\scriptsize (2.12)} & 68.40 {\scriptsize (2.08)} & 26.95 {\scriptsize (1.98)} & 27.80 {\scriptsize (2.00)} \\
\midrule
16 & 1 & 16 & 21.75 {\scriptsize (1.37)} & 25.26 {\scriptsize (1.51)} & 21.75 {\scriptsize (1.37)} & 25.26 {\scriptsize (1.51)} \\
16 & 2 & 32 & 29.87 {\scriptsize (1.63)} & 33.25 {\scriptsize (1.74)} & 21.75 {\scriptsize (1.37)} & 25.26 {\scriptsize (1.51)} \\
16 & 4 & 64 & 38.30 {\scriptsize (1.80)} & 41.18 {\scriptsize (1.87)} & 25.41 {\scriptsize (1.61)} & 28.92 {\scriptsize (1.73)} \\
16 & 8 & 128 & 46.93 {\scriptsize (1.90)} & 48.98 {\scriptsize (1.96)} & 28.44 {\scriptsize (1.77)} & 31.78 {\scriptsize (1.87)} \\
16 & 16 & 256 & 55.81 {\scriptsize (1.99)} & 56.46 {\scriptsize (2.03)} & 30.65 {\scriptsize (1.90)} & 33.68 {\scriptsize (1.98)} \\
16 & 32 & 512 & 64.80 {\scriptsize (2.14)} & 63.60 {\scriptsize (2.16)} & 32.21 {\scriptsize (2.07)} & 34.93 {\scriptsize (2.11)} \\
\midrule
32 & 1 & 32 & 25.76 {\scriptsize (1.54)} & 29.80 {\scriptsize (1.65)} & 25.76 {\scriptsize (1.54)} & 29.80 {\scriptsize (1.65)} \\
32 & 2 & 64 & 33.90 {\scriptsize (1.74)} & 37.89 {\scriptsize (1.82)} & 25.76 {\scriptsize (1.54)} & 29.80 {\scriptsize (1.65)} \\
32 & 4 & 128 & 42.27 {\scriptsize (1.87)} & 45.91 {\scriptsize (1.94)} & 29.38 {\scriptsize (1.76)} & 33.32 {\scriptsize (1.84)} \\
32 & 8 & 256 & 50.98 {\scriptsize (2.00)} & 53.80 {\scriptsize (2.03)} & 31.92 {\scriptsize (1.90)} & 35.97 {\scriptsize (1.96)} \\
32 & 16 & 512 & 59.60 {\scriptsize (2.20)} & 61.80 {\scriptsize (2.19)} & 33.39 {\scriptsize (2.07)} & 37.74 {\scriptsize (2.13)} \\
\midrule
64 & 1 & 64 & 28.58 {\scriptsize (1.60)} & 32.40 {\scriptsize (1.76)} & 28.58 {\scriptsize (1.60)} & 32.40 {\scriptsize (1.76)} \\
64 & 2 & 128 & 37.39 {\scriptsize (1.82)} & 40.20 {\scriptsize (1.92)} & 28.58 {\scriptsize (1.60)} & 32.40 {\scriptsize (1.76)} \\
64 & 4 & 256 & 46.28 {\scriptsize (1.99)} & 47.41 {\scriptsize (2.06)} & 32.65 {\scriptsize (1.84)} & 36.07 {\scriptsize (1.94)} \\
64 & 8 & 512 & 55.20 {\scriptsize (2.24)} & 53.80 {\scriptsize (2.24)} & 36.58 {\scriptsize (2.09)} & 38.80 {\scriptsize (2.13)} \\
\bottomrule
\end{tabular}
\end{table}
\begin{table}[!htbp]
\caption{GSM8K, CEDR-B: every point in Figure~\ref{fig:elfbstages}, with DMD2-style sampling and SCCFG 2. Entries are accuracy percentages with problem-bootstrap SEs in parentheses (Appendix~\ref{app:uncertainty}). Each denoising depth uses $n=512/S$ saved draws per problem; $k$-subset averaging uses all $n$.}
\label{tab:budgetvaluesgsmb}
\centering\small
\setlength{\tabcolsep}{5pt}
\begin{tabular}{@{}rrr rrrr@{}}
\toprule
& & & \multicolumn{2}{c}{Oracle pass@$k$} & \multicolumn{2}{c}{Majority vote} \\
\cmidrule(lr){4-5}\cmidrule(l){6-7}
NFE $S$ & $k$ & $kS$ & Pre-NFT & Post-NFT & Pre-NFT & Post-NFT \\
\midrule
8 & 1 & 8 & 24.95 {\scriptsize (0.87)} & 28.48 {\scriptsize (0.96)} & 24.95 {\scriptsize (0.87)} & 28.48 {\scriptsize (0.96)} \\
8 & 2 & 16 & 33.86 {\scriptsize (1.02)} & 36.63 {\scriptsize (1.08)} & 24.95 {\scriptsize (0.87)} & 28.48 {\scriptsize (0.96)} \\
8 & 4 & 32 & 42.68 {\scriptsize (1.12)} & 44.26 {\scriptsize (1.16)} & 29.12 {\scriptsize (1.02)} & 32.00 {\scriptsize (1.08)} \\
8 & 8 & 64 & 51.08 {\scriptsize (1.17)} & 51.36 {\scriptsize (1.20)} & 31.84 {\scriptsize (1.11)} & 33.96 {\scriptsize (1.15)} \\
8 & 16 & 128 & 58.93 {\scriptsize (1.18)} & 57.95 {\scriptsize (1.21)} & 33.34 {\scriptsize (1.17)} & 35.06 {\scriptsize (1.20)} \\
8 & 32 & 256 & 66.15 {\scriptsize (1.19)} & 64.13 {\scriptsize (1.22)} & 34.21 {\scriptsize (1.23)} & 35.84 {\scriptsize (1.25)} \\
8 & 64 & 512 & 72.33 {\scriptsize (1.22)} & 70.20 {\scriptsize (1.25)} & 35.09 {\scriptsize (1.31)} & 36.73 {\scriptsize (1.31)} \\
\midrule
16 & 1 & 16 & 38.25 {\scriptsize (1.05)} & 40.60 {\scriptsize (1.10)} & 38.25 {\scriptsize (1.05)} & 40.60 {\scriptsize (1.10)} \\
16 & 2 & 32 & 47.26 {\scriptsize (1.14)} & 48.50 {\scriptsize (1.17)} & 38.25 {\scriptsize (1.05)} & 40.60 {\scriptsize (1.10)} \\
16 & 4 & 64 & 55.29 {\scriptsize (1.16)} & 55.64 {\scriptsize (1.18)} & 42.54 {\scriptsize (1.17)} & 43.76 {\scriptsize (1.20)} \\
16 & 8 & 128 & 62.65 {\scriptsize (1.16)} & 62.27 {\scriptsize (1.19)} & 44.90 {\scriptsize (1.23)} & 45.33 {\scriptsize (1.25)} \\
16 & 16 & 256 & 69.60 {\scriptsize (1.15)} & 68.20 {\scriptsize (1.19)} & 45.97 {\scriptsize (1.28)} & 45.90 {\scriptsize (1.29)} \\
16 & 32 & 512 & 76.12 {\scriptsize (1.18)} & 73.24 {\scriptsize (1.23)} & 46.25 {\scriptsize (1.36)} & 45.97 {\scriptsize (1.35)} \\
\midrule
32 & 1 & 32 & 44.39 {\scriptsize (1.11)} & 45.59 {\scriptsize (1.16)} & 44.39 {\scriptsize (1.11)} & 45.59 {\scriptsize (1.16)} \\
32 & 2 & 64 & 52.86 {\scriptsize (1.17)} & 52.74 {\scriptsize (1.20)} & 44.39 {\scriptsize (1.11)} & 45.59 {\scriptsize (1.16)} \\
32 & 4 & 128 & 60.30 {\scriptsize (1.18)} & 59.07 {\scriptsize (1.21)} & 48.10 {\scriptsize (1.21)} & 48.26 {\scriptsize (1.24)} \\
32 & 8 & 256 & 67.16 {\scriptsize (1.18)} & 65.01 {\scriptsize (1.22)} & 50.05 {\scriptsize (1.27)} & 49.73 {\scriptsize (1.29)} \\
32 & 16 & 512 & 73.31 {\scriptsize (1.22)} & 70.51 {\scriptsize (1.26)} & 50.85 {\scriptsize (1.35)} & 51.08 {\scriptsize (1.36)} \\
\midrule
64 & 1 & 64 & 47.62 {\scriptsize (1.16)} & 47.00 {\scriptsize (1.19)} & 47.62 {\scriptsize (1.16)} & 47.00 {\scriptsize (1.19)} \\
64 & 2 & 128 & 55.39 {\scriptsize (1.21)} & 53.86 {\scriptsize (1.23)} & 47.62 {\scriptsize (1.16)} & 47.00 {\scriptsize (1.19)} \\
64 & 4 & 256 & 62.04 {\scriptsize (1.23)} & 60.11 {\scriptsize (1.26)} & 50.92 {\scriptsize (1.26)} & 49.17 {\scriptsize (1.27)} \\
64 & 8 & 512 & 68.01 {\scriptsize (1.28)} & 65.66 {\scriptsize (1.32)} & 53.13 {\scriptsize (1.34)} & 50.21 {\scriptsize (1.36)} \\
\bottomrule
\end{tabular}
\end{table}

\FloatBarrier
\subsubsection{Numerical values for the other samplers}
\label{app:samplerbudgetvalues}
Tables~\ref{tab:samplerbudgetgsmlode}--\ref{tab:samplerbudgetgsmbedm} give every ODE and EDM-style point in Figures~\ref{fig:samplerallocationgsml}--\ref{fig:samplerallocationgsmb}. The DMD2-style values are already listed in Tables~\ref{tab:budgetvaluesgsml}--\ref{tab:budgetvaluesgsmb}. Parentheses report problem-bootstrap standard errors, with the same 4,000-replicate procedure and seed as the headline tables (Appendix~\ref{app:uncertainty}). All comparisons use the full benchmark and every saved draw at each NFE.

\begin{table}[p]
\centering\small
\setlength{\tabcolsep}{5pt}
\caption{GSM8K, CEDR-L: ODE values in Figure~\ref{fig:samplerallocationgsml}. Entries are accuracy percentages, with problem-bootstrap standard errors in parentheses (Appendix~\ref{app:uncertainty}). SCCFG 2; asynchronous clocks $(2.5,2,1.5)$.}
\label{tab:samplerbudgetgsmlode}
\begin{tabular}{@{}rrr rrrr@{}}
\toprule
& & & \multicolumn{2}{c}{Oracle pass@$k$} & \multicolumn{2}{c}{Majority vote} \\
\cmidrule(lr){4-5}\cmidrule(l){6-7}
NFE $S$ & $k$ & $kS$ & Pre-NFT & Post-NFT & Pre-NFT & Post-NFT \\
\midrule
8 & 1 & 8 & 52.69 {\scriptsize (1.04)} & 57.06 {\scriptsize (1.09)} & 52.69 {\scriptsize (1.04)} & 57.06 {\scriptsize (1.09)} \\
8 & 2 & 16 & 63.43 {\scriptsize (1.07)} & 65.89 {\scriptsize (1.09)} & 52.69 {\scriptsize (1.04)} & 57.06 {\scriptsize (1.09)} \\
8 & 4 & 32 & 71.64 {\scriptsize (1.03)} & 72.71 {\scriptsize (1.05)} & 59.57 {\scriptsize (1.14)} & 62.02 {\scriptsize (1.16)} \\
8 & 8 & 64 & 78.18 {\scriptsize (0.96)} & 78.11 {\scriptsize (1.00)} & 63.32 {\scriptsize (1.17)} & 64.61 {\scriptsize (1.19)} \\
8 & 16 & 128 & 83.35 {\scriptsize (0.90)} & 82.37 {\scriptsize (0.94)} & 65.32 {\scriptsize (1.20)} & 66.01 {\scriptsize (1.21)} \\
8 & 32 & 256 & 87.28 {\scriptsize (0.83)} & 85.75 {\scriptsize (0.89)} & 66.40 {\scriptsize (1.23)} & 66.78 {\scriptsize (1.24)} \\
8 & 64 & 512 & 90.22 {\scriptsize (0.81)} & 88.40 {\scriptsize (0.88)} & 67.06 {\scriptsize (1.28)} & 67.13 {\scriptsize (1.30)} \\
\midrule
16 & 1 & 16 & 57.80 {\scriptsize (1.07)} & 62.64 {\scriptsize (1.10)} & 57.80 {\scriptsize (1.07)} & 62.64 {\scriptsize (1.10)} \\
16 & 2 & 32 & 67.50 {\scriptsize (1.07)} & 70.46 {\scriptsize (1.08)} & 57.80 {\scriptsize (1.07)} & 62.64 {\scriptsize (1.10)} \\
16 & 4 & 64 & 74.73 {\scriptsize (1.02)} & 76.40 {\scriptsize (1.02)} & 63.83 {\scriptsize (1.14)} & 66.82 {\scriptsize (1.15)} \\
16 & 8 & 128 & 80.37 {\scriptsize (0.96)} & 81.07 {\scriptsize (0.97)} & 66.97 {\scriptsize (1.17)} & 68.98 {\scriptsize (1.17)} \\
16 & 16 & 256 & 84.78 {\scriptsize (0.91)} & 84.78 {\scriptsize (0.92)} & 68.62 {\scriptsize (1.19)} & 70.11 {\scriptsize (1.20)} \\
16 & 32 & 512 & 88.32 {\scriptsize (0.89)} & 87.79 {\scriptsize (0.91)} & 69.34 {\scriptsize (1.25)} & 70.79 {\scriptsize (1.24)} \\
\midrule
32 & 1 & 32 & 59.55 {\scriptsize (1.08)} & 64.33 {\scriptsize (1.10)} & 59.55 {\scriptsize (1.08)} & 64.33 {\scriptsize (1.10)} \\
32 & 2 & 64 & 69.14 {\scriptsize (1.07)} & 71.93 {\scriptsize (1.08)} & 59.55 {\scriptsize (1.08)} & 64.33 {\scriptsize (1.10)} \\
32 & 4 & 128 & 76.16 {\scriptsize (1.02)} & 77.67 {\scriptsize (1.03)} & 65.57 {\scriptsize (1.15)} & 68.18 {\scriptsize (1.15)} \\
32 & 8 & 256 & 81.61 {\scriptsize (0.97)} & 82.19 {\scriptsize (0.98)} & 68.62 {\scriptsize (1.18)} & 70.18 {\scriptsize (1.18)} \\
32 & 16 & 512 & 85.97 {\scriptsize (0.95)} & 85.82 {\scriptsize (0.97)} & 70.21 {\scriptsize (1.24)} & 71.18 {\scriptsize (1.25)} \\
\midrule
64 & 1 & 64 & 60.81 {\scriptsize (1.10)} & 65.03 {\scriptsize (1.12)} & 60.81 {\scriptsize (1.10)} & 65.03 {\scriptsize (1.12)} \\
64 & 2 & 128 & 69.99 {\scriptsize (1.09)} & 72.50 {\scriptsize (1.08)} & 60.81 {\scriptsize (1.10)} & 65.03 {\scriptsize (1.12)} \\
64 & 4 & 256 & 76.77 {\scriptsize (1.06)} & 78.27 {\scriptsize (1.04)} & 66.40 {\scriptsize (1.17)} & 68.69 {\scriptsize (1.17)} \\
64 & 8 & 512 & 81.80 {\scriptsize (1.06)} & 83.02 {\scriptsize (1.04)} & 69.56 {\scriptsize (1.24)} & 70.58 {\scriptsize (1.23)} \\
\bottomrule
\end{tabular}
\end{table}

\begin{table}[p]
\centering\small
\setlength{\tabcolsep}{5pt}
\caption{GSM8K, CEDR-L: Stochastic (EDM) values in Figure~\ref{fig:samplerallocationgsml}. Entries are accuracy percentages, with problem-bootstrap standard errors in parentheses (Appendix~\ref{app:uncertainty}). SCCFG 2; asynchronous clocks $(2.5,2,1.5)$.}
\label{tab:samplerbudgetgsmledm}
\begin{tabular}{@{}rrr rrrr@{}}
\toprule
& & & \multicolumn{2}{c}{Oracle pass@$k$} & \multicolumn{2}{c}{Majority vote} \\
\cmidrule(lr){4-5}\cmidrule(l){6-7}
NFE $S$ & $k$ & $kS$ & Pre-NFT & Post-NFT & Pre-NFT & Post-NFT \\
\midrule
8 & 1 & 8 & 48.95 {\scriptsize (1.04)} & 53.22 {\scriptsize (1.09)} & 48.95 {\scriptsize (1.04)} & 53.22 {\scriptsize (1.09)} \\
8 & 2 & 16 & 59.77 {\scriptsize (1.09)} & 62.39 {\scriptsize (1.12)} & 48.96 {\scriptsize (1.04)} & 53.22 {\scriptsize (1.09)} \\
8 & 4 & 32 & 68.15 {\scriptsize (1.07)} & 69.39 {\scriptsize (1.10)} & 56.01 {\scriptsize (1.15)} & 58.69 {\scriptsize (1.18)} \\
8 & 8 & 64 & 74.88 {\scriptsize (1.02)} & 74.98 {\scriptsize (1.05)} & 60.02 {\scriptsize (1.19)} & 61.48 {\scriptsize (1.21)} \\
8 & 16 & 128 & 80.44 {\scriptsize (0.95)} & 79.60 {\scriptsize (0.99)} & 62.26 {\scriptsize (1.23)} & 62.99 {\scriptsize (1.24)} \\
8 & 32 & 256 & 85.14 {\scriptsize (0.89)} & 83.50 {\scriptsize (0.94)} & 63.50 {\scriptsize (1.26)} & 63.82 {\scriptsize (1.27)} \\
8 & 64 & 512 & 89.01 {\scriptsize (0.86)} & 86.88 {\scriptsize (0.92)} & 64.16 {\scriptsize (1.31)} & 64.06 {\scriptsize (1.32)} \\
\midrule
16 & 1 & 16 & 58.93 {\scriptsize (1.10)} & 62.51 {\scriptsize (1.14)} & 58.93 {\scriptsize (1.10)} & 62.51 {\scriptsize (1.14)} \\
16 & 2 & 32 & 67.52 {\scriptsize (1.10)} & 69.22 {\scriptsize (1.12)} & 58.93 {\scriptsize (1.10)} & 62.51 {\scriptsize (1.14)} \\
16 & 4 & 64 & 73.81 {\scriptsize (1.05)} & 74.34 {\scriptsize (1.08)} & 64.14 {\scriptsize (1.17)} & 65.90 {\scriptsize (1.18)} \\
16 & 8 & 128 & 78.97 {\scriptsize (0.99)} & 78.67 {\scriptsize (1.02)} & 66.51 {\scriptsize (1.19)} & 67.44 {\scriptsize (1.21)} \\
16 & 16 & 256 & 83.38 {\scriptsize (0.94)} & 82.49 {\scriptsize (0.98)} & 67.64 {\scriptsize (1.21)} & 68.16 {\scriptsize (1.24)} \\
16 & 32 & 512 & 87.04 {\scriptsize (0.92)} & 85.67 {\scriptsize (0.97)} & 68.49 {\scriptsize (1.26)} & 68.43 {\scriptsize (1.28)} \\
\midrule
32 & 1 & 32 & 64.27 {\scriptsize (1.12)} & 66.83 {\scriptsize (1.14)} & 64.27 {\scriptsize (1.12)} & 66.83 {\scriptsize (1.14)} \\
32 & 2 & 64 & 71.16 {\scriptsize (1.10)} & 72.49 {\scriptsize (1.11)} & 64.27 {\scriptsize (1.12)} & 66.84 {\scriptsize (1.14)} \\
32 & 4 & 128 & 76.39 {\scriptsize (1.05)} & 76.83 {\scriptsize (1.07)} & 67.80 {\scriptsize (1.17)} & 69.27 {\scriptsize (1.18)} \\
32 & 8 & 256 & 80.97 {\scriptsize (1.00)} & 80.44 {\scriptsize (1.04)} & 69.34 {\scriptsize (1.20)} & 70.37 {\scriptsize (1.20)} \\
32 & 16 & 512 & 85.29 {\scriptsize (0.98)} & 83.55 {\scriptsize (1.03)} & 69.91 {\scriptsize (1.25)} & 70.88 {\scriptsize (1.25)} \\
\midrule
64 & 1 & 64 & 66.20 {\scriptsize (1.14)} & 67.93 {\scriptsize (1.16)} & 66.20 {\scriptsize (1.14)} & 67.93 {\scriptsize (1.16)} \\
64 & 2 & 128 & 72.25 {\scriptsize (1.11)} & 72.87 {\scriptsize (1.14)} & 66.20 {\scriptsize (1.14)} & 67.93 {\scriptsize (1.16)} \\
64 & 4 & 256 & 77.25 {\scriptsize (1.08)} & 76.72 {\scriptsize (1.10)} & 68.56 {\scriptsize (1.18)} & 69.84 {\scriptsize (1.20)} \\
64 & 8 & 512 & 81.35 {\scriptsize (1.08)} & 80.21 {\scriptsize (1.11)} & 69.77 {\scriptsize (1.24)} & 70.18 {\scriptsize (1.25)} \\
\bottomrule
\end{tabular}
\end{table}

\begin{table}[p]
\centering\small
\setlength{\tabcolsep}{5pt}
\caption{MATH500, CEDR-L: ODE values in Figure~\ref{fig:samplerallocationmathl}. Entries are accuracy percentages, with problem-bootstrap standard errors in parentheses (Appendix~\ref{app:uncertainty}). SCCFG 3; asynchronous clocks $(2.5,2,1.5)$.}
\label{tab:samplerbudgetmathlode}
\begin{tabular}{@{}rrr rrrr@{}}
\toprule
& & & \multicolumn{2}{c}{Oracle pass@$k$} & \multicolumn{2}{c}{Majority vote} \\
\cmidrule(lr){4-5}\cmidrule(l){6-7}
NFE $S$ & $k$ & $kS$ & Pre-NFT & Post-NFT & Pre-NFT & Post-NFT \\
\midrule
8 & 1 & 8 & 15.23 {\scriptsize (1.06)} & 17.45 {\scriptsize (1.17)} & 15.23 {\scriptsize (1.06)} & 17.45 {\scriptsize (1.17)} \\
8 & 2 & 16 & 22.75 {\scriptsize (1.35)} & 25.22 {\scriptsize (1.45)} & 15.24 {\scriptsize (1.06)} & 17.46 {\scriptsize (1.17)} \\
8 & 4 & 32 & 31.59 {\scriptsize (1.59)} & 33.93 {\scriptsize (1.67)} & 18.37 {\scriptsize (1.30)} & 20.86 {\scriptsize (1.42)} \\
8 & 8 & 64 & 41.11 {\scriptsize (1.77)} & 43.01 {\scriptsize (1.81)} & 21.58 {\scriptsize (1.49)} & 24.06 {\scriptsize (1.60)} \\
8 & 16 & 128 & 50.60 {\scriptsize (1.90)} & 51.99 {\scriptsize (1.90)} & 23.90 {\scriptsize (1.63)} & 26.30 {\scriptsize (1.73)} \\
8 & 32 & 256 & 59.35 {\scriptsize (1.98)} & 60.44 {\scriptsize (1.98)} & 25.26 {\scriptsize (1.76)} & 27.70 {\scriptsize (1.86)} \\
8 & 64 & 512 & 66.80 {\scriptsize (2.10)} & 67.60 {\scriptsize (2.11)} & 25.30 {\scriptsize (1.92)} & 28.63 {\scriptsize (1.99)} \\
\midrule
16 & 1 & 16 & 17.93 {\scriptsize (1.17)} & 20.93 {\scriptsize (1.31)} & 17.93 {\scriptsize (1.17)} & 20.93 {\scriptsize (1.31)} \\
16 & 2 & 32 & 26.24 {\scriptsize (1.46)} & 29.47 {\scriptsize (1.59)} & 17.93 {\scriptsize (1.17)} & 20.93 {\scriptsize (1.31)} \\
16 & 4 & 64 & 35.64 {\scriptsize (1.69)} & 38.51 {\scriptsize (1.77)} & 21.58 {\scriptsize (1.42)} & 24.77 {\scriptsize (1.57)} \\
16 & 8 & 128 & 45.51 {\scriptsize (1.86)} & 47.61 {\scriptsize (1.88)} & 25.24 {\scriptsize (1.61)} & 28.03 {\scriptsize (1.75)} \\
16 & 16 & 256 & 55.10 {\scriptsize (2.00)} & 56.56 {\scriptsize (1.96)} & 28.26 {\scriptsize (1.78)} & 30.24 {\scriptsize (1.89)} \\
16 & 32 & 512 & 63.40 {\scriptsize (2.17)} & 65.20 {\scriptsize (2.11)} & 31.72 {\scriptsize (2.06)} & 31.30 {\scriptsize (2.07)} \\
\midrule
32 & 1 & 32 & 19.65 {\scriptsize (1.27)} & 22.60 {\scriptsize (1.39)} & 19.65 {\scriptsize (1.27)} & 22.60 {\scriptsize (1.39)} \\
32 & 2 & 64 & 28.09 {\scriptsize (1.57)} & 31.33 {\scriptsize (1.66)} & 19.65 {\scriptsize (1.27)} & 22.61 {\scriptsize (1.39)} \\
32 & 4 & 128 & 37.36 {\scriptsize (1.80)} & 40.33 {\scriptsize (1.87)} & 23.53 {\scriptsize (1.54)} & 26.71 {\scriptsize (1.63)} \\
32 & 8 & 256 & 46.98 {\scriptsize (1.99)} & 48.77 {\scriptsize (2.01)} & 27.18 {\scriptsize (1.76)} & 30.33 {\scriptsize (1.83)} \\
32 & 16 & 512 & 56.80 {\scriptsize (2.24)} & 56.80 {\scriptsize (2.20)} & 30.11 {\scriptsize (2.03)} & 32.45 {\scriptsize (2.06)} \\
\midrule
64 & 1 & 64 & 21.18 {\scriptsize (1.37)} & 24.68 {\scriptsize (1.48)} & 21.18 {\scriptsize (1.37)} & 24.68 {\scriptsize (1.48)} \\
64 & 2 & 128 & 29.99 {\scriptsize (1.68)} & 33.67 {\scriptsize (1.75)} & 21.18 {\scriptsize (1.37)} & 24.69 {\scriptsize (1.48)} \\
64 & 4 & 256 & 39.35 {\scriptsize (1.96)} & 42.73 {\scriptsize (1.97)} & 25.29 {\scriptsize (1.66)} & 29.19 {\scriptsize (1.75)} \\
64 & 8 & 512 & 48.60 {\scriptsize (2.27)} & 51.20 {\scriptsize (2.23)} & 29.28 {\scriptsize (2.00)} & 33.38 {\scriptsize (2.04)} \\
\bottomrule
\end{tabular}
\end{table}

\begin{table}[p]
\centering\small
\setlength{\tabcolsep}{5pt}
\caption{MATH500, CEDR-L: Stochastic (EDM) values in Figure~\ref{fig:samplerallocationmathl}. Entries are accuracy percentages, with problem-bootstrap standard errors in parentheses (Appendix~\ref{app:uncertainty}). SCCFG 3; asynchronous clocks $(2.5,2,1.5)$.}
\label{tab:samplerbudgetmathledm}
\begin{tabular}{@{}rrr rrrr@{}}
\toprule
& & & \multicolumn{2}{c}{Oracle pass@$k$} & \multicolumn{2}{c}{Majority vote} \\
\cmidrule(lr){4-5}\cmidrule(l){6-7}
NFE $S$ & $k$ & $kS$ & Pre-NFT & Post-NFT & Pre-NFT & Post-NFT \\
\midrule
8 & 1 & 8 & 13.38 {\scriptsize (0.97)} & 15.55 {\scriptsize (1.12)} & 13.38 {\scriptsize (0.97)} & 15.55 {\scriptsize (1.12)} \\
8 & 2 & 16 & 20.37 {\scriptsize (1.28)} & 22.61 {\scriptsize (1.42)} & 13.39 {\scriptsize (0.97)} & 15.56 {\scriptsize (1.12)} \\
8 & 4 & 32 & 28.66 {\scriptsize (1.55)} & 30.32 {\scriptsize (1.64)} & 16.37 {\scriptsize (1.22)} & 18.75 {\scriptsize (1.38)} \\
8 & 8 & 64 & 37.58 {\scriptsize (1.76)} & 38.40 {\scriptsize (1.81)} & 19.58 {\scriptsize (1.42)} & 21.49 {\scriptsize (1.56)} \\
8 & 16 & 128 & 46.43 {\scriptsize (1.90)} & 46.60 {\scriptsize (1.93)} & 21.92 {\scriptsize (1.57)} & 23.15 {\scriptsize (1.69)} \\
8 & 32 & 256 & 54.95 {\scriptsize (2.00)} & 54.65 {\scriptsize (2.03)} & 23.57 {\scriptsize (1.71)} & 24.04 {\scriptsize (1.79)} \\
8 & 64 & 512 & 63.40 {\scriptsize (2.16)} & 62.40 {\scriptsize (2.19)} & 24.72 {\scriptsize (1.88)} & 24.20 {\scriptsize (1.91)} \\
\midrule
16 & 1 & 16 & 18.11 {\scriptsize (1.22)} & 20.73 {\scriptsize (1.34)} & 18.11 {\scriptsize (1.22)} & 20.73 {\scriptsize (1.34)} \\
16 & 2 & 32 & 25.88 {\scriptsize (1.52)} & 28.48 {\scriptsize (1.61)} & 18.11 {\scriptsize (1.22)} & 20.73 {\scriptsize (1.34)} \\
16 & 4 & 64 & 34.25 {\scriptsize (1.75)} & 36.51 {\scriptsize (1.79)} & 21.77 {\scriptsize (1.49)} & 24.35 {\scriptsize (1.59)} \\
16 & 8 & 128 & 42.84 {\scriptsize (1.91)} & 44.59 {\scriptsize (1.91)} & 24.94 {\scriptsize (1.67)} & 27.08 {\scriptsize (1.74)} \\
16 & 16 & 256 & 51.38 {\scriptsize (2.03)} & 52.65 {\scriptsize (2.02)} & 27.24 {\scriptsize (1.82)} & 28.71 {\scriptsize (1.85)} \\
16 & 32 & 512 & 60.20 {\scriptsize (2.21)} & 60.80 {\scriptsize (2.20)} & 28.89 {\scriptsize (2.02)} & 29.66 {\scriptsize (2.01)} \\
\midrule
32 & 1 & 32 & 22.60 {\scriptsize (1.39)} & 26.08 {\scriptsize (1.54)} & 22.60 {\scriptsize (1.39)} & 26.08 {\scriptsize (1.54)} \\
32 & 2 & 64 & 31.14 {\scriptsize (1.65)} & 34.29 {\scriptsize (1.78)} & 22.60 {\scriptsize (1.39)} & 26.08 {\scriptsize (1.54)} \\
32 & 4 & 128 & 40.10 {\scriptsize (1.84)} & 41.99 {\scriptsize (1.93)} & 26.84 {\scriptsize (1.65)} & 30.56 {\scriptsize (1.79)} \\
32 & 8 & 256 & 49.05 {\scriptsize (2.00)} & 49.18 {\scriptsize (2.03)} & 30.66 {\scriptsize (1.84)} & 33.99 {\scriptsize (1.95)} \\
32 & 16 & 512 & 57.40 {\scriptsize (2.20)} & 56.20 {\scriptsize (2.20)} & 32.89 {\scriptsize (2.04)} & 36.40 {\scriptsize (2.14)} \\
\midrule
64 & 1 & 64 & 27.00 {\scriptsize (1.57)} & 30.80 {\scriptsize (1.70)} & 27.00 {\scriptsize (1.57)} & 30.80 {\scriptsize (1.70)} \\
64 & 2 & 128 & 35.56 {\scriptsize (1.82)} & 39.12 {\scriptsize (1.89)} & 27.00 {\scriptsize (1.57)} & 30.80 {\scriptsize (1.70)} \\
64 & 4 & 256 & 43.69 {\scriptsize (2.01)} & 47.03 {\scriptsize (2.03)} & 31.53 {\scriptsize (1.84)} & 35.05 {\scriptsize (1.91)} \\
64 & 8 & 512 & 51.40 {\scriptsize (2.22)} & 54.20 {\scriptsize (2.23)} & 34.56 {\scriptsize (2.07)} & 38.13 {\scriptsize (2.12)} \\
\bottomrule
\end{tabular}
\end{table}

\begin{table}[p]
\centering\small
\setlength{\tabcolsep}{5pt}
\caption{GSM8K, CEDR-B: ODE values in Figure~\ref{fig:samplerallocationgsmb}. Entries are accuracy percentages, with problem-bootstrap standard errors in parentheses (Appendix~\ref{app:uncertainty}). SCCFG 2; asynchronous clocks $(2.5,2,1.5)$.}
\label{tab:samplerbudgetgsmbode}
\begin{tabular}{@{}rrr rrrr@{}}
\toprule
& & & \multicolumn{2}{c}{Oracle pass@$k$} & \multicolumn{2}{c}{Majority vote} \\
\cmidrule(lr){4-5}\cmidrule(l){6-7}
NFE $S$ & $k$ & $kS$ & Pre-NFT & Post-NFT & Pre-NFT & Post-NFT \\
\midrule
8 & 1 & 8 & 26.40 {\scriptsize (0.86)} & 30.08 {\scriptsize (0.95)} & 26.40 {\scriptsize (0.86)} & 30.08 {\scriptsize (0.95)} \\
8 & 2 & 16 & 36.10 {\scriptsize (1.02)} & 39.21 {\scriptsize (1.08)} & 26.40 {\scriptsize (0.86)} & 30.08 {\scriptsize (0.95)} \\
8 & 4 & 32 & 45.58 {\scriptsize (1.11)} & 47.77 {\scriptsize (1.14)} & 31.20 {\scriptsize (1.02)} & 34.14 {\scriptsize (1.09)} \\
8 & 8 & 64 & 54.40 {\scriptsize (1.15)} & 55.75 {\scriptsize (1.16)} & 34.55 {\scriptsize (1.11)} & 36.50 {\scriptsize (1.16)} \\
8 & 16 & 128 & 62.54 {\scriptsize (1.15)} & 63.20 {\scriptsize (1.15)} & 36.49 {\scriptsize (1.18)} & 37.77 {\scriptsize (1.22)} \\
8 & 32 & 256 & 69.86 {\scriptsize (1.15)} & 70.07 {\scriptsize (1.14)} & 37.57 {\scriptsize (1.24)} & 38.43 {\scriptsize (1.27)} \\
8 & 64 & 512 & 75.82 {\scriptsize (1.17)} & 76.04 {\scriptsize (1.16)} & 37.99 {\scriptsize (1.32)} & 38.47 {\scriptsize (1.33)} \\
\midrule
16 & 1 & 16 & 34.48 {\scriptsize (0.98)} & 37.77 {\scriptsize (1.05)} & 34.48 {\scriptsize (0.98)} & 37.77 {\scriptsize (1.05)} \\
16 & 2 & 32 & 44.71 {\scriptsize (1.09)} & 46.80 {\scriptsize (1.14)} & 34.48 {\scriptsize (0.98)} & 37.77 {\scriptsize (1.05)} \\
16 & 4 & 64 & 54.08 {\scriptsize (1.13)} & 54.86 {\scriptsize (1.17)} & 39.81 {\scriptsize (1.12)} & 42.03 {\scriptsize (1.17)} \\
16 & 8 & 128 & 62.46 {\scriptsize (1.14)} & 62.17 {\scriptsize (1.17)} & 43.28 {\scriptsize (1.19)} & 44.41 {\scriptsize (1.23)} \\
16 & 16 & 256 & 69.71 {\scriptsize (1.13)} & 68.79 {\scriptsize (1.16)} & 45.29 {\scriptsize (1.26)} & 45.61 {\scriptsize (1.29)} \\
16 & 32 & 512 & 75.89 {\scriptsize (1.16)} & 74.68 {\scriptsize (1.19)} & 46.42 {\scriptsize (1.36)} & 45.83 {\scriptsize (1.36)} \\
\midrule
32 & 1 & 32 & 37.80 {\scriptsize (1.01)} & 41.13 {\scriptsize (1.09)} & 37.80 {\scriptsize (1.01)} & 41.13 {\scriptsize (1.09)} \\
32 & 2 & 64 & 48.13 {\scriptsize (1.11)} & 50.16 {\scriptsize (1.16)} & 37.80 {\scriptsize (1.01)} & 41.13 {\scriptsize (1.09)} \\
32 & 4 & 128 & 57.46 {\scriptsize (1.14)} & 58.15 {\scriptsize (1.19)} & 43.14 {\scriptsize (1.14)} & 45.32 {\scriptsize (1.19)} \\
32 & 8 & 256 & 65.84 {\scriptsize (1.15)} & 65.17 {\scriptsize (1.20)} & 46.60 {\scriptsize (1.23)} & 47.79 {\scriptsize (1.26)} \\
32 & 16 & 512 & 73.62 {\scriptsize (1.20)} & 71.42 {\scriptsize (1.24)} & 48.24 {\scriptsize (1.33)} & 49.17 {\scriptsize (1.35)} \\
\midrule
64 & 1 & 64 & 39.17 {\scriptsize (1.04)} & 42.16 {\scriptsize (1.11)} & 39.17 {\scriptsize (1.04)} & 42.16 {\scriptsize (1.11)} \\
64 & 2 & 128 & 49.82 {\scriptsize (1.15)} & 51.36 {\scriptsize (1.19)} & 39.17 {\scriptsize (1.04)} & 42.16 {\scriptsize (1.11)} \\
64 & 4 & 256 & 59.20 {\scriptsize (1.21)} & 59.40 {\scriptsize (1.24)} & 45.04 {\scriptsize (1.19)} & 46.56 {\scriptsize (1.22)} \\
64 & 8 & 512 & 67.32 {\scriptsize (1.29)} & 66.26 {\scriptsize (1.32)} & 48.44 {\scriptsize (1.32)} & 49.07 {\scriptsize (1.34)} \\
\bottomrule
\end{tabular}
\end{table}

\begin{table}[p]
\centering\small
\setlength{\tabcolsep}{5pt}
\caption{GSM8K, CEDR-B: Stochastic (EDM) values in Figure~\ref{fig:samplerallocationgsmb}. Entries are accuracy percentages, with problem-bootstrap standard errors in parentheses (Appendix~\ref{app:uncertainty}). SCCFG 2; asynchronous clocks $(2.5,2,1.5)$.}
\label{tab:samplerbudgetgsmbedm}
\begin{tabular}{@{}rrr rrrr@{}}
\toprule
& & & \multicolumn{2}{c}{Oracle pass@$k$} & \multicolumn{2}{c}{Majority vote} \\
\cmidrule(lr){4-5}\cmidrule(l){6-7}
NFE $S$ & $k$ & $kS$ & Pre-NFT & Post-NFT & Pre-NFT & Post-NFT \\
\midrule
8 & 1 & 8 & 20.26 {\scriptsize (0.77)} & 24.31 {\scriptsize (0.88)} & 20.26 {\scriptsize (0.77)} & 24.31 {\scriptsize (0.88)} \\
8 & 2 & 16 & 28.80 {\scriptsize (0.94)} & 32.65 {\scriptsize (1.03)} & 20.26 {\scriptsize (0.77)} & 24.31 {\scriptsize (0.88)} \\
8 & 4 & 32 & 37.79 {\scriptsize (1.07)} & 40.79 {\scriptsize (1.13)} & 24.16 {\scriptsize (0.93)} & 27.91 {\scriptsize (1.02)} \\
8 & 8 & 64 & 46.71 {\scriptsize (1.15)} & 48.52 {\scriptsize (1.20)} & 27.05 {\scriptsize (1.04)} & 29.96 {\scriptsize (1.10)} \\
8 & 16 & 128 & 55.16 {\scriptsize (1.19)} & 55.52 {\scriptsize (1.23)} & 28.66 {\scriptsize (1.11)} & 30.96 {\scriptsize (1.16)} \\
8 & 32 & 256 & 62.86 {\scriptsize (1.22)} & 61.83 {\scriptsize (1.24)} & 29.57 {\scriptsize (1.17)} & 31.54 {\scriptsize (1.21)} \\
8 & 64 & 512 & 69.60 {\scriptsize (1.28)} & 67.93 {\scriptsize (1.28)} & 29.88 {\scriptsize (1.25)} & 32.16 {\scriptsize (1.28)} \\
\midrule
16 & 1 & 16 & 31.79 {\scriptsize (0.97)} & 35.67 {\scriptsize (1.06)} & 31.79 {\scriptsize (0.97)} & 35.67 {\scriptsize (1.06)} \\
16 & 2 & 32 & 41.08 {\scriptsize (1.10)} & 43.86 {\scriptsize (1.15)} & 31.79 {\scriptsize (0.97)} & 35.68 {\scriptsize (1.06)} \\
16 & 4 & 64 & 49.57 {\scriptsize (1.16)} & 51.14 {\scriptsize (1.19)} & 36.56 {\scriptsize (1.11)} & 39.35 {\scriptsize (1.17)} \\
16 & 8 & 128 & 57.48 {\scriptsize (1.18)} & 58.01 {\scriptsize (1.21)} & 39.36 {\scriptsize (1.19)} & 41.16 {\scriptsize (1.24)} \\
16 & 16 & 256 & 65.10 {\scriptsize (1.19)} & 64.43 {\scriptsize (1.22)} & 40.67 {\scriptsize (1.25)} & 41.87 {\scriptsize (1.28)} \\
16 & 32 & 512 & 72.25 {\scriptsize (1.24)} & 69.98 {\scriptsize (1.26)} & 41.17 {\scriptsize (1.33)} & 41.88 {\scriptsize (1.34)} \\
\midrule
32 & 1 & 32 & 41.11 {\scriptsize (1.10)} & 43.06 {\scriptsize (1.14)} & 41.11 {\scriptsize (1.10)} & 43.06 {\scriptsize (1.14)} \\
32 & 2 & 64 & 49.76 {\scriptsize (1.16)} & 50.48 {\scriptsize (1.21)} & 41.11 {\scriptsize (1.10)} & 43.06 {\scriptsize (1.14)} \\
32 & 4 & 128 & 57.44 {\scriptsize (1.19)} & 56.76 {\scriptsize (1.23)} & 45.17 {\scriptsize (1.20)} & 46.19 {\scriptsize (1.24)} \\
32 & 8 & 256 & 64.49 {\scriptsize (1.20)} & 62.37 {\scriptsize (1.24)} & 47.56 {\scriptsize (1.27)} & 47.79 {\scriptsize (1.29)} \\
32 & 16 & 512 & 71.19 {\scriptsize (1.24)} & 67.55 {\scriptsize (1.29)} & 49.09 {\scriptsize (1.37)} & 49.03 {\scriptsize (1.36)} \\
\midrule
64 & 1 & 64 & 44.99 {\scriptsize (1.16)} & 46.49 {\scriptsize (1.18)} & 44.99 {\scriptsize (1.16)} & 46.49 {\scriptsize (1.18)} \\
64 & 2 & 128 & 52.98 {\scriptsize (1.21)} & 53.80 {\scriptsize (1.21)} & 44.99 {\scriptsize (1.16)} & 46.49 {\scriptsize (1.18)} \\
64 & 4 & 256 & 60.14 {\scriptsize (1.24)} & 60.38 {\scriptsize (1.24)} & 48.40 {\scriptsize (1.25)} & 49.06 {\scriptsize (1.26)} \\
64 & 8 & 512 & 66.94 {\scriptsize (1.32)} & 66.11 {\scriptsize (1.30)} & 50.44 {\scriptsize (1.34)} & 50.67 {\scriptsize (1.35)} \\
\bottomrule
\end{tabular}
\end{table}

\FloatBarrier

\section{Code generation: protocols and detailed results}
\label{app:codingdetails}

\paragraph{Training and evaluation.}
Appendix~\ref{app:training} describes the coding architecture and curriculum. Supervised evaluations use flow/prompt EMA $0.9999/0.9999$; NFT update 100 uses $0.99/0.99$. Frozen-Qwen and prompt-MSE conditions at flow epoch 12 share identical backbone weights. Joint adaptation and NFT update both learned components.

\paragraph{Generation and scoring.}
We use native Qwen-chat benchmark prompts without assistant prefill, batch size 32, CFG 1, SCCFG 3, and async $(2.5,2,1.5)$. Answers fill the remaining positions of the 1,024-token canvas and are decoded greedily. NFE excludes one terminal decoder call. The 64-NFE ODE stage curves use seeds 42, 123, 456, and 789. The 128-NFE ODE and DMD2-style endpoints use ten seeds: $\{42,123,456,789,2026,31415,27182,16180,112358,104729\}$. Headlines use DMD2-style sampling.

HumanEval \citep{humaneval} contains 164 tasks; MBPP \citep{mbpp} uses the 378 tasks retained by EvalPlus \citep{evalplus}. We report base and full extended tests on the same generated programs, requiring both suites to pass for a plus score. Standard scoring applies native EvalPlus sanitization without function-name repair. Execution uses a minimum timeout of 4 seconds, reference-runtime multiplier four, and 4 GiB per program; failures and timeouts count as incorrect. Our separate alias condition changes extraction only: if the expected binding is absent and exactly one top-level function is present, append \texttt{expected\_name = generated\_name} and sanitize again. Existing bindings and ambiguous cases are unchanged. This changes which programs the evaluator accepts without changing generation.

\paragraph{Mapping external reports.}
Our PlaidQ comparison uses the author-reported results from arXiv version 1, subsequently withdrawn by the authors to address institutional resource-disclosure requirements \citep{plaidq}. PlaidQ Appendix B.7 distinguishes original MBPP-500 from the 378-task subset evaluated using base assertions; its latter column is named ``MBPP+'' \citep{plaidq}. We map only that latter column to our MBPP-378 base metric. PlaidQ uses prefix completion with a 128-token response budget, benchmark-specific extraction, and a 15-second program timeout. Its Table 1 supplies our pass@1 and pass@10 comparison values. ELF-REG v3 Table 15 supplies standard-scoring ODE pass@1, Table 4 supplies the stronger aliased stochastic pass@1, and Table 18 supplies aliased ODE pass@8/10 \citep{scalingdlm}. The stochastic results use four generation seeds, versus 16 for the ODE reports; no stochastic no-alias or pass@8/10 results are reported. Its sampler settings and seed SDs appear in Table~\ref{tab:elfregselected}; its Appendix B.3 documents the alias rule and larger execution limits (20-second minimum, reference multiplier ten, 6 GiB). Our alias rows match the entry-point repair and test-set definitions, while retaining our prompts and execution limits. Full extended MBPP+ values are not reported for either comparison paper. We retain their reported values without re-evaluating their models.

For additional discrete-model context, oDLM-0.6B reports HumanEval/HumanEval+ pass@1 of 17.87\%/16.40\% at 128 sampling steps with a 128-token response budget \citep[Table 1 and Appendix B.7]{odlm}. Edit Flow-1.3B reports 12.8\%/10.4\%, and its localized variant reports 14.0\%/10.4\% \citep[Table 3]{editflow}. Our DMD2-style learned-conditioning results without name repair are 39.51\%/36.10\% at joint epoch 14 and 40.98\%/37.56\% after NFT. These comparisons use the authors' reported accuracies; training data, initialization, and sampling protocols differ. The additional comparison covers HumanEval(+).

\paragraph{Comparison under function-name aliasing.}
Table~\ref{tab:codingcomparison} keeps aliasing separate from standard scoring. Aliasing raises post-NFT MBPP-378 base accuracy from 30.74\% to 45.87\%. Under the same repair category, CEDR-L exceeds ELF-REG-L's reported stochastic accuracies on all three available metrics. Prompts and execution limits still differ between papers; repairing entry-point names changes acceptance criteria without changing the generated programs.

\paragraph{Prompt lengths and encoder size.}
Unique formatted training prefixes average 84.17 tokens for GSM8K (275,661 prompts), 89.15 for MATH (508,953), and 233.44 for OpenCodeInstruct (4,772,508). Code prefixes have median length 216 and 99th percentile 506. These are native training prefixes including the assistant header, rather than benchmark question lengths. Longer, more complex prompts are one possible explanation for the larger MBPP+ degradation after prompt substitution. The CEDR-L contextual encoder has 121,333,180 parameters, excluding the frozen $151{,}936\times2{,}560$ lookup. ELF-REG's Qwen3-0.6B-Base contextual body has approximately 440M parameters excluding its approximately 156M lookup, while our frozen-Qwen reference uses a 4B teacher.

\paragraph{Training trajectory and prompt substitution.}
Figure~\ref{fig:promptcurriculum} and Table~\ref{tab:codingstages} compare frozen-Qwen flow epochs 4/8/12, substitution of the prompt-20 encoder at the unchanged epoch-12 backbone, and joint epochs 13/14. All coding points use four matched seeds, ODE64, and SCCFG 3. Prompt-MSE training has its own epoch counter and does not advance the flow-training epoch. Error bars use 4,000 problem-bootstrap replicates with seed 20260918; Table~\ref{tab:codingstageintervals} gives their bounds. Axes start at zero without an evaluated epoch-zero point.

\begin{table}[htbp]
\centering\small
\setlength{\tabcolsep}{3pt}
\caption{\textbf{Coding training and prompt substitution at 64 NFE.} Mean pass@1 (\%) with generation-seed SD in brackets, using ODE, SCCFG 3, and four matched seeds. The prompt-MSE row retains the epoch-12 backbone. These standard-scoring results supply the coding panel of Figure~\ref{fig:promptcurriculum}.}
\label{tab:codingstages}
\begin{tabular}{@{}lrrrrr@{}}
\toprule
Checkpoint / conditioner & Seeds & HE & HE+ & \shortstack{MBPP-378\\base} & MBPP+\\
\midrule
Epoch 4 / Qwen & 4 & 23.48 {\scriptsize [0.79]} & 21.80 {\scriptsize [1.04]} & 27.51 {\scriptsize [1.75]} & 24.27 {\scriptsize [1.23]}\\
Epoch 8 / Qwen & 4 & 30.49 {\scriptsize [2.95]} & 27.90 {\scriptsize [2.46]} & 36.57 {\scriptsize [0.90]} & 31.81 {\scriptsize [0.82]}\\
Epoch 12 / Qwen & 4 & 32.93 {\scriptsize [1.93]} & 30.34 {\scriptsize [2.01]} & 38.10 {\scriptsize [0.37]} & 32.94 {\scriptsize [0.55]}\\
Epoch 12 / prompt MSE 20 & 4 & 27.44 {\scriptsize [1.22]} & 25.46 {\scriptsize [1.35]} & 17.66 {\scriptsize [1.62]} & 15.15 {\scriptsize [1.39]}\\
Joint epoch 13 & 4 & 29.12 {\scriptsize [1.75]} & 26.37 {\scriptsize [1.35]} & 20.11 {\scriptsize [1.30]} & 17.53 {\scriptsize [1.55]}\\
Joint epoch 14 & 4 & 24.85 {\scriptsize [1.52]} & 22.41 {\scriptsize [2.36]} & 19.97 {\scriptsize [1.09]} & 17.59 {\scriptsize [0.70]}\\
\bottomrule
\end{tabular}
\end{table}

\begin{table}[htbp]
\centering\small
\setlength{\tabcolsep}{3pt}
\caption{Pointwise 95\% problem-bootstrap intervals (\%) for the coding curves in Figure~\ref{fig:promptcurriculum}, using 4,000 replicates.}
\label{tab:codingstageintervals}
\begin{tabular}{@{}lrr@{}}
\toprule
Checkpoint / conditioner & HumanEval+ & MBPP+\\
\midrule
Epoch 4 / Qwen & $[17.07,26.98]$ & $[20.83,27.98]$\\
Epoch 8 / Qwen & $[22.56,33.54]$ & $[27.84,35.85]$\\
Epoch 12 / Qwen & $[24.85,36.13]$ & $[28.90,37.17]$\\
Epoch 12 / prompt MSE 20 & $[20.12,31.10]$ & $[12.37,18.06]$\\
Joint epoch 13 & $[20.73,32.32]$ & $[14.55,20.63]$\\
Joint epoch 14 & $[17.53,27.59]$ & $[14.68,20.63]$\\
\bottomrule
\end{tabular}
\end{table}

\paragraph{Ten-seed variation and multiple attempts.}
Tables~\ref{tab:codingstandarddetail}--\ref{tab:codingaliasdetail} give ODE and DMD2-style results at the supervised and NFT endpoints. For each problem with $c$ successes among $n=10$ draws, pass@$k$ is $1-\binom{n-c}{k}/\binom{n}{k}$, averaged over problems. We report $k=1,2,4,8,10$; pass@8 averages all eight-draw subsets of the ten-sample pool. Parentheses give problem-bootstrap standard errors, while the final column gives the separate sample SD of per-seed pass@1. Bootstrap resampling retains all draws for each sampled problem and uses 4,000 replicates with seed 20260918. Appendix~\ref{app:uncertainty} defines both uncertainty measures.

\begin{table}[htbp]
\centering\footnotesize
\setlength{\tabcolsep}{2.5pt}
\caption{\textbf{Standard scoring without aliases.} CEDR-L at 128 denoising calls and SCCFG 3, using ten seeds. Pre-NFT is joint epoch 14; post-NFT is update 100. Pass@$k$ values (\%) have problem-bootstrap SEs in parentheses; the final column gives the separate generation-seed SD of pass@1. Both MBPP columns use 378 tasks.}
\label{tab:codingstandarddetail}
\begin{tabular}{@{}llrrrrrr@{}}
\toprule
Stage & Benchmark & pass@1 & pass@2 & pass@4 & pass@8 & pass@10 & \shortstack{pass@1\\seed SD}\\
\midrule
\multicolumn{8}{@{}l}{\textit{ODE}}\\
Pre-NFT & HE & 29.45 {\scriptsize (2.82)} & 38.02 {\scriptsize (3.25)} & 45.26 {\scriptsize (3.54)} & 50.88 {\scriptsize (3.77)} & 52.44 {\scriptsize (3.86)} & 2.34\\
Pre-NFT & HE+ & 27.20 {\scriptsize (2.79)} & 34.95 {\scriptsize (3.20)} & 41.64 {\scriptsize (3.49)} & 47.11 {\scriptsize (3.76)} & 48.78 {\scriptsize (3.88)} & 2.53\\
Pre-NFT & MBPP-378 base & 21.16 {\scriptsize (1.59)} & 29.19 {\scriptsize (1.95)} & 37.15 {\scriptsize (2.21)} & 45.05 {\scriptsize (2.47)} & 47.62 {\scriptsize (2.59)} & 1.53\\
Pre-NFT & MBPP+ & 18.49 {\scriptsize (1.53)} & 25.63 {\scriptsize (1.91)} & 32.38 {\scriptsize (2.19)} & 39.12 {\scriptsize (2.46)} & 41.53 {\scriptsize (2.59)} & 1.21\\
Post-NFT & HE & 30.91 {\scriptsize (2.86)} & 39.72 {\scriptsize (3.22)} & 47.89 {\scriptsize (3.44)} & 55.85 {\scriptsize (3.68)} & 58.54 {\scriptsize (3.83)} & 1.84\\
Post-NFT & HE+ & 28.29 {\scriptsize (2.83)} & 36.14 {\scriptsize (3.20)} & 43.35 {\scriptsize (3.47)} & 49.88 {\scriptsize (3.77)} & 51.83 {\scriptsize (3.91)} & 1.78\\
Post-NFT & MBPP-378 base & 22.12 {\scriptsize (1.68)} & 29.65 {\scriptsize (1.98)} & 37.06 {\scriptsize (2.22)} & 44.15 {\scriptsize (2.48)} & 46.30 {\scriptsize (2.58)} & 0.75\\
Post-NFT & MBPP+ & 19.50 {\scriptsize (1.60)} & 26.43 {\scriptsize (1.94)} & 32.99 {\scriptsize (2.20)} & 39.14 {\scriptsize (2.47)} & 41.01 {\scriptsize (2.57)} & 0.66\\
\midrule
\multicolumn{8}{@{}l}{\textit{DMD2-style}}\\
Pre-NFT & HE & 39.51 {\scriptsize (3.17)} & 47.68 {\scriptsize (3.44)} & 54.04 {\scriptsize (3.59)} & 59.31 {\scriptsize (3.75)} & 60.98 {\scriptsize (3.84)} & 3.20\\
Pre-NFT & HE+ & 36.10 {\scriptsize (3.15)} & 43.31 {\scriptsize (3.43)} & 48.83 {\scriptsize (3.62)} & 53.46 {\scriptsize (3.81)} & 54.88 {\scriptsize (3.91)} & 2.60\\
Pre-NFT & MBPP-378 base & 30.32 {\scriptsize (2.03)} & 36.78 {\scriptsize (2.24)} & 42.32 {\scriptsize (2.37)} & 47.68 {\scriptsize (2.52)} & 49.47 {\scriptsize (2.60)} & 1.26\\
Pre-NFT & MBPP+ & 26.69 {\scriptsize (1.97)} & 32.41 {\scriptsize (2.20)} & 37.24 {\scriptsize (2.35)} & 41.99 {\scriptsize (2.51)} & 43.65 {\scriptsize (2.59)} & 0.95\\
Post-NFT & HE & 40.98 {\scriptsize (3.36)} & 47.15 {\scriptsize (3.52)} & 52.47 {\scriptsize (3.62)} & 57.21 {\scriptsize (3.76)} & 58.54 {\scriptsize (3.82)} & 1.81\\
Post-NFT & HE+ & 37.56 {\scriptsize (3.35)} & 42.86 {\scriptsize (3.52)} & 47.36 {\scriptsize (3.67)} & 51.00 {\scriptsize (3.85)} & 51.83 {\scriptsize (3.91)} & 2.04\\
Post-NFT & MBPP-378 base & 30.74 {\scriptsize (2.06)} & 36.85 {\scriptsize (2.26)} & 41.93 {\scriptsize (2.39)} & 46.63 {\scriptsize (2.53)} & 48.15 {\scriptsize (2.60)} & 1.18\\
Post-NFT & MBPP+ & 26.96 {\scriptsize (2.00)} & 32.31 {\scriptsize (2.24)} & 36.35 {\scriptsize (2.38)} & 39.92 {\scriptsize (2.51)} & 41.01 {\scriptsize (2.57)} & 1.51\\
\bottomrule
\end{tabular}
\end{table}

\begin{table}[htbp]
\centering\footnotesize
\setlength{\tabcolsep}{2.5pt}
\caption{\textbf{Single-function-name-alias scoring.} CEDR-L at 128 denoising calls and SCCFG 3, using ten seeds. Pre-NFT is joint epoch 14; post-NFT is update 100. Pass@$k$ values (\%) have problem-bootstrap SEs in parentheses; the final column gives the separate generation-seed SD of pass@1. Both MBPP columns use 378 tasks.}
\label{tab:codingaliasdetail}
\begin{tabular}{@{}llrrrrrr@{}}
\toprule
Stage & Benchmark & pass@1 & pass@2 & pass@4 & pass@8 & pass@10 & \shortstack{pass@1\\seed SD}\\
\midrule
\multicolumn{8}{@{}l}{\textit{ODE}}\\
Pre-NFT & HE & 29.63 {\scriptsize (2.83)} & 38.17 {\scriptsize (3.25)} & 45.43 {\scriptsize (3.54)} & 50.99 {\scriptsize (3.77)} & 52.44 {\scriptsize (3.86)} & 2.43\\
Pre-NFT & HE+ & 27.38 {\scriptsize (2.79)} & 35.11 {\scriptsize (3.19)} & 41.88 {\scriptsize (3.48)} & 47.60 {\scriptsize (3.74)} & 49.39 {\scriptsize (3.87)} & 2.63\\
Pre-NFT & MBPP-378 base & 34.21 {\scriptsize (1.97)} & 43.55 {\scriptsize (2.21)} & 51.69 {\scriptsize (2.35)} & 58.52 {\scriptsize (2.47)} & 60.32 {\scriptsize (2.53)} & 1.87\\
Pre-NFT & MBPP+ & 29.74 {\scriptsize (1.95)} & 37.75 {\scriptsize (2.23)} & 44.49 {\scriptsize (2.39)} & 50.64 {\scriptsize (2.55)} & 52.65 {\scriptsize (2.63)} & 1.66\\
Post-NFT & HE & 31.22 {\scriptsize (2.88)} & 39.84 {\scriptsize (3.23)} & 47.90 {\scriptsize (3.44)} & 55.85 {\scriptsize (3.68)} & 58.54 {\scriptsize (3.83)} & 1.97\\
Post-NFT & HE+ & 28.60 {\scriptsize (2.86)} & 36.26 {\scriptsize (3.21)} & 43.36 {\scriptsize (3.47)} & 49.88 {\scriptsize (3.77)} & 51.83 {\scriptsize (3.91)} & 1.94\\
Post-NFT & MBPP-378 base & 36.08 {\scriptsize (2.02)} & 45.22 {\scriptsize (2.24)} & 52.91 {\scriptsize (2.35)} & 59.38 {\scriptsize (2.46)} & 61.11 {\scriptsize (2.51)} & 1.44\\
Post-NFT & MBPP+ & 31.22 {\scriptsize (1.99)} & 39.19 {\scriptsize (2.24)} & 45.77 {\scriptsize (2.39)} & 51.64 {\scriptsize (2.54)} & 53.44 {\scriptsize (2.61)} & 1.38\\
\midrule
\multicolumn{8}{@{}l}{\textit{DMD2-style}}\\
Pre-NFT & HE & 39.57 {\scriptsize (3.17)} & 47.75 {\scriptsize (3.44)} & 54.07 {\scriptsize (3.59)} & 59.31 {\scriptsize (3.75)} & 60.98 {\scriptsize (3.84)} & 3.22\\
Pre-NFT & HE+ & 36.16 {\scriptsize (3.15)} & 43.37 {\scriptsize (3.44)} & 48.86 {\scriptsize (3.62)} & 53.46 {\scriptsize (3.81)} & 54.88 {\scriptsize (3.91)} & 2.64\\
Pre-NFT & MBPP-378 base & 44.89 {\scriptsize (2.24)} & 52.05 {\scriptsize (2.39)} & 57.12 {\scriptsize (2.46)} & 60.83 {\scriptsize (2.52)} & 61.90 {\scriptsize (2.56)} & 1.69\\
Pre-NFT & MBPP+ & 38.62 {\scriptsize (2.24)} & 44.77 {\scriptsize (2.42)} & 49.01 {\scriptsize (2.53)} & 52.02 {\scriptsize (2.60)} & 52.91 {\scriptsize (2.63)} & 1.33\\
Post-NFT & HE & 40.98 {\scriptsize (3.36)} & 47.15 {\scriptsize (3.52)} & 52.47 {\scriptsize (3.62)} & 57.21 {\scriptsize (3.76)} & 58.54 {\scriptsize (3.82)} & 1.81\\
Post-NFT & HE+ & 37.56 {\scriptsize (3.35)} & 42.86 {\scriptsize (3.52)} & 47.36 {\scriptsize (3.67)} & 51.00 {\scriptsize (3.85)} & 51.83 {\scriptsize (3.91)} & 2.04\\
Post-NFT & MBPP-378 base & 45.87 {\scriptsize (2.28)} & 52.45 {\scriptsize (2.41)} & 57.05 {\scriptsize (2.47)} & 60.36 {\scriptsize (2.51)} & 61.38 {\scriptsize (2.54)} & 1.46\\
Post-NFT & MBPP+ & 39.37 {\scriptsize (2.28)} & 44.87 {\scriptsize (2.44)} & 48.56 {\scriptsize (2.52)} & 51.32 {\scriptsize (2.59)} & 52.12 {\scriptsize (2.62)} & 1.25\\
\bottomrule
\end{tabular}
\end{table}

\paragraph{Joint adaptation without a prompt-MSE anchor.}
Table~\ref{tab:codingjointdeltas} compares joint epoch 14 with the immediate prompt-20 swap at 64 NFE. Both stages use the same four seeds, guidance, and scoring, and flow/prompt EMA $0.9999/0.9999$. The comparison measures joint adaptation of the encoder and denoiser through the downstream objective.

\begin{table}[htbp]
\centering\small
\setlength{\tabcolsep}{3pt}
\caption{\textbf{Joint epoch 14 minus the prompt-20 swap.} ODE differences in percentage points at 64 NFE and SCCFG 3, using four matched seeds. Brackets give paired 95\% problem-bootstrap intervals. Both stages use flow/prompt EMA $.9999/.9999$.}
\label{tab:codingjointdeltas}
\begin{tabular}{@{}llrr@{}}
\toprule
Scoring & Benchmark & $\Delta$ pass@1 [95\% CI] & $\Delta$ pass@4 [95\% CI]\\
\midrule
Standard & HumanEval & $-2.59$ {\scriptsize $[-5.95,+0.61]$} & $+0.61$ {\scriptsize $[-4.88,+6.10]$}\\
Standard & HumanEval+ & $-3.05$ {\scriptsize $[-6.25,+0.00]$} & $+0.00$ {\scriptsize $[-4.88,+4.88]$}\\
Standard & MBPP-378 base & $+2.31$ {\scriptsize $[+0.33,+4.30]$} & $+2.38$ {\scriptsize $[-1.32,+6.35]$}\\
Standard & MBPP+ & $+2.45$ {\scriptsize $[+0.66,+4.23]$} & $+3.44$ {\scriptsize $[+0.00,+6.88]$}\\
\midrule
Alias & HumanEval & $-2.29$ {\scriptsize $[-5.49,+0.91]$} & $+0.61$ {\scriptsize $[-4.88,+6.10]$}\\
Alias & HumanEval+ & $-2.74$ {\scriptsize $[-5.95,+0.30]$} & $+0.00$ {\scriptsize $[-4.88,+4.88]$}\\
Alias & MBPP-378 base & $+1.85$ {\scriptsize $[-0.40,+4.03]$} & $+1.06$ {\scriptsize $[-2.65,+5.03]$}\\
Alias & MBPP+ & $+2.25$ {\scriptsize $[+0.33,+4.23]$} & $+1.59$ {\scriptsize $[-1.85,+5.03]$}\\
\bottomrule
\end{tabular}
\end{table}

\paragraph{NFT with execution-based rewards.}
Tables~\ref{tab:codingnftdeltas} and~\ref{tab:codingsdedeltas} compare NFT update 100 against joint epoch 14 under ODE and DMD2-style sampling, respectively. Both use ten matched seeds, 128 NFE, SCCFG 3, and the same scoring rules. Flow/prompt evaluation EMAs are $0.9999/0.9999$ before NFT and $0.99/0.99$ afterward. Under DMD2-style sampling, mean pass@1 increases on all four benchmarks, but each paired 95\% interval includes zero. Pass@10 decreases on all four; only the standard MBPP+ decrease has an interval excluding zero. Thus the coding evidence supports modest shifts in single-sample accuracy, rather than a uniform improvement in oracle coverage. The tables report paired problem-bootstrap intervals for pass@1 and pass@10.

\begin{table}[htbp]
\centering\small
\setlength{\tabcolsep}{3pt}
\caption{\textbf{ODE: NFT update 100 minus joint epoch 14.} Differences in percentage points with paired 95\% problem-bootstrap intervals, using ten matched seeds at 128 NFE and SCCFG 3.}
\label{tab:codingnftdeltas}
\begin{tabular}{@{}llrr@{}}
\toprule
Scoring & Benchmark & $\Delta$ pass@1 [95\% CI] & $\Delta$ pass@10 [95\% CI]\\
\midrule
Standard & HumanEval & $+1.46$ {\scriptsize $[-0.12,+3.05]$} & $+6.10$ {\scriptsize $[+0.61,+10.98]$}\\
Standard & HumanEval+ & $+1.10$ {\scriptsize $[-0.37,+2.56]$} & $+3.05$ {\scriptsize $[-1.83,+7.93]$}\\
Standard & MBPP-378 base & $+0.95$ {\scriptsize $[+0.03,+1.88]$} & $-1.32$ {\scriptsize $[-4.23,+1.59]$}\\
Standard & MBPP+ & $+1.01$ {\scriptsize $[+0.13,+1.85]$} & $-0.53$ {\scriptsize $[-3.17,+1.85]$}\\
\midrule
Alias & HumanEval & $+1.59$ {\scriptsize $[-0.00,+3.11]$} & $+6.10$ {\scriptsize $[+0.61,+10.98]$}\\
Alias & HumanEval+ & $+1.22$ {\scriptsize $[-0.18,+2.62]$} & $+2.44$ {\scriptsize $[-2.44,+7.32]$}\\
Alias & MBPP-378 base & $+1.88$ {\scriptsize $[+0.85,+2.88]$} & $+0.79$ {\scriptsize $[-1.85,+3.44]$}\\
Alias & MBPP+ & $+1.48$ {\scriptsize $[+0.53,+2.41]$} & $+0.79$ {\scriptsize $[-1.59,+3.17]$}\\
\bottomrule
\end{tabular}
\end{table}

\begin{table}[htbp]
\centering\small
\setlength{\tabcolsep}{3pt}
\caption{\textbf{DMD2-style sampling: NFT update 100 minus joint epoch 14.} Differences in percentage points with paired 95\% problem-bootstrap intervals, using ten matched seeds at 128 NFE and SCCFG 3.}
\label{tab:codingsdedeltas}
\begin{tabular}{@{}llrr@{}}
\toprule
Scoring & Benchmark & $\Delta$ pass@1 [95\% CI] & $\Delta$ pass@10 [95\% CI]\\
\midrule
Standard & HumanEval & $+1.46$ {\scriptsize $[-0.43,+3.41]$} & $-2.44$ {\scriptsize $[-6.71,+2.44]$}\\
Standard & HumanEval+ & $+1.46$ {\scriptsize $[-0.55,+3.54]$} & $-3.05$ {\scriptsize $[-7.32,+1.22]$}\\
Standard & MBPP-378 base & $+0.42$ {\scriptsize $[-0.63,+1.59]$} & $-1.32$ {\scriptsize $[-3.97,+1.32]$}\\
Standard & MBPP+ & $+0.26$ {\scriptsize $[-0.66,+1.30]$} & $-2.65$ {\scriptsize $[-5.03,-0.26]$}\\
\midrule
Alias & HumanEval & $+1.40$ {\scriptsize $[-0.49,+3.29]$} & $-2.44$ {\scriptsize $[-6.71,+2.44]$}\\
Alias & HumanEval+ & $+1.40$ {\scriptsize $[-0.55,+3.41]$} & $-3.05$ {\scriptsize $[-7.32,+1.22]$}\\
Alias & MBPP-378 base & $+0.98$ {\scriptsize $[-0.19,+2.17]$} & $-0.53$ {\scriptsize $[-2.91,+2.12]$}\\
Alias & MBPP+ & $+0.74$ {\scriptsize $[-0.29,+1.83]$} & $-0.79$ {\scriptsize $[-3.17,+1.59]$}\\
\bottomrule
\end{tabular}
\end{table}

\section{Generation examples and trajectories}
\label{app:qualitative}

\subsection{Generated solutions}
\label{app:generations}

Figure~\ref{fig:headlineexamples} presents one correct GSM8K generation and one correct MATH500 generation from the post-NFT CEDR-L models using ODE sampling. The solutions combine intermediate calculations into a final answer: successive multiplicative changes in the GSM8K example, and substitution followed by rational arithmetic in the MATH500 example. Both are drawn from saved benchmark generations and retain the original model text.

\begin{figure}[htbp]
\centering
\setlength{\fboxsep}{8pt}
\fbox{\begin{minipage}{0.94\linewidth}
\small
\setlength{\fboxsep}{3pt}
\textbf{GSM8K: reasoning across successive changes}\par\medskip
\textbf{Problem.} A new program had 60 downloads in the first month. The number of downloads in the second month was three times as many as the downloads in the first month, but then reduced by 30\% in the third month. How many downloads did the program have total over the three months?

\medskip\textbf{Generated solution.}\par\smallskip
Let's break down the downloads for each month:

First month: The program had 60 downloads.

Second month: The number of downloads three times as many as the downloads in the first month, so 3 * 60 = 180 downloads.

Third month: The number of downloads reduced by 30\%, which is 180 - (30\% of 180) = 180 - 0.3 * 180 = 180 - 54 = 126 downloads.

To get the total number of downloads over the three months, we add up the downloads for each month: 60 + 180 + 126 = 240 + 126 = 366

Thus, the program had a total of $\boxed{366}$ downloads over the three months.
\end{minipage}}

\medskip
\fbox{\begin{minipage}{0.94\linewidth}
\small
\setlength{\fboxsep}{3pt}
\textbf{MATH500: evaluating a rational function}\par\medskip
\textbf{Problem.} If $f(x) = \frac{3x-2}{x-2}$, what is the value of $f(-2) +f(-1)+f(0)$? Express your answer as a common fraction.

\medskip\textbf{Generated solution.}\par\smallskip
To find the value of $f(-2)+f(-1)+f(0)$ we need to substitute $x=-2, x=-1,$ and $x=0$ into the function $f(x)$.

For $f(-2)$: $f(-2) = \frac{3(-2) - 2}{-2 - 2} = \frac{-6 - 2}{-4} = \frac{-8}{-4} = 2$

For $f(-1)$: $f(-1) = \frac{3(-1) - 2}{-1 - 2} = \frac{-3 - 2}{-3} = \frac{-5}{-3} = \frac{5}{3}$

For $f(0)$: $f(0) = \frac{3(0) - 2}{0 - 2} = \frac{-2}{-2} = 1$

Now we can add these together to find the sum: $f(-2) + f(-1) + f(0) = 2 + \frac{5}{3} + 1 = \frac{11}{3} + 1 = \frac{11}{3} + \frac{3}{3} = \frac{14}{3}$ $\boxed{\frac{14}{3}}$
\end{minipage}}
\caption{Correct generations from CEDR-L at 64 denoising NFEs, async $(2.5,2,1.5)$, SCCFG 2, and seed 42. Both examples use NFT update 600. Displayed equations have been moved inline to save space; model wording and calculations are unchanged. Both answers and the intermediate calculations are correct.}
\label{fig:headlineexamples}
\end{figure}

\subsection{Intermediate generation trajectory}
\label{app:trajectories}

Figure~\ref{fig:gsmtrajectory} shows the ODE trajectory at steps 60, 62, and 64 of the GSM8K example in Figure~\ref{fig:headlineexamples}. After sampling, we apply the terminal decoder (time $1$, SCCFG~2, zero decoder self-conditioning) to each saved state. These are post-hoc decoder readouts of the saved noisy intermediate states. The final text exactly reproduces the archived generation.

Light blue marks exact token-ID equality at the same position in the final answer, before its EOS; matches need not persist. Panels label the local clocks $(t_{16},t_{24},t_{32})$: $0$ is noise and $1$ is clean. Each readout retains its native extraction and EOS; only whitespace layout changes.

\begin{figure}[p]
\centering
\includegraphics[width=\linewidth,height=0.84\textheight,keepaspectratio]{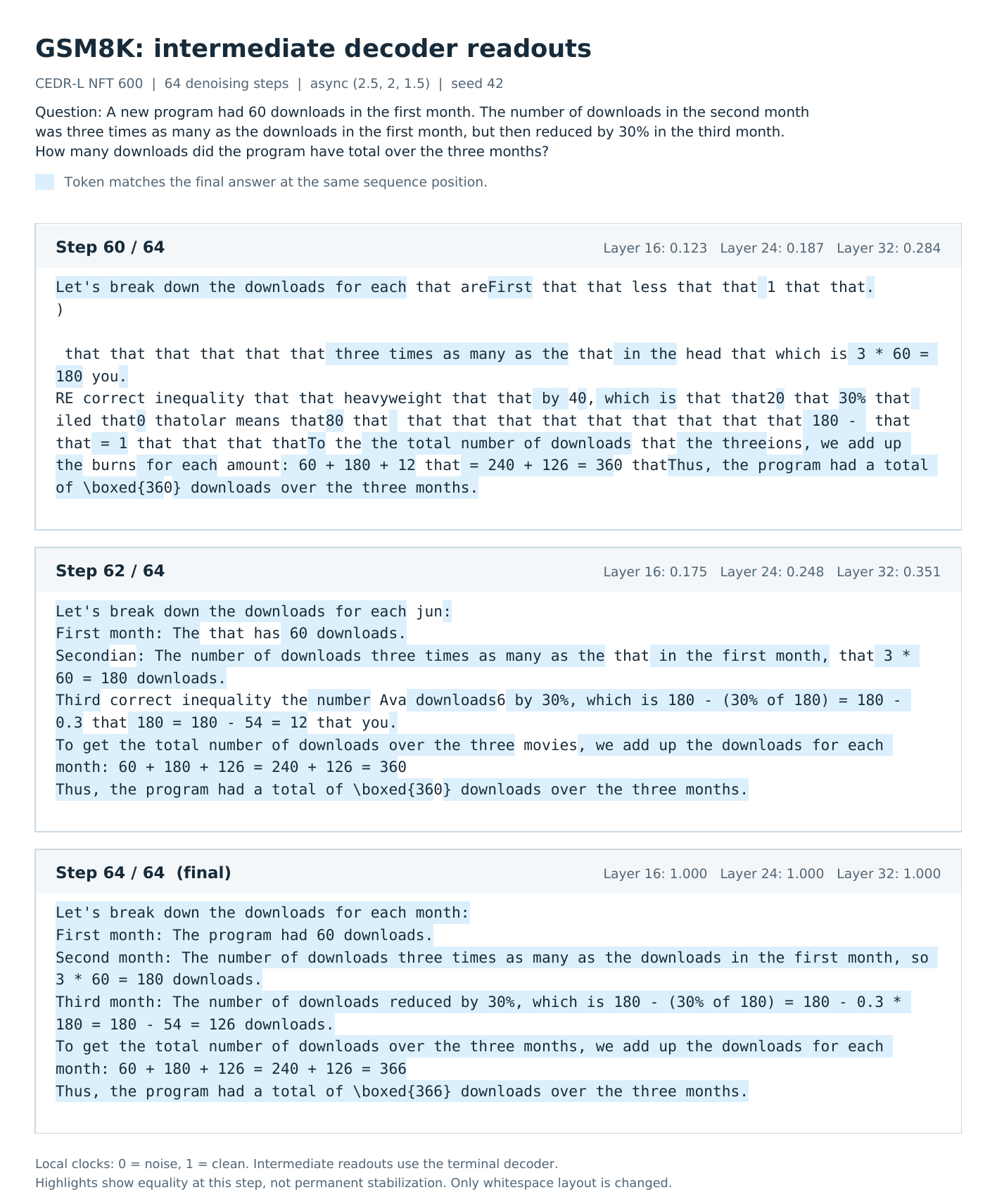}
\caption{\textbf{GSM8K ODE trajectory: steps 60, 62, and 64.} Decoder readouts for the example in Figure~\ref{fig:headlineexamples}, with the final solution in the bottom panel. CEDR-L NFT update 600; backbone/prompt EMA $.99/.99$; seed 42; async $(2.5,2,1.5)$; 64 denoising steps; CFG~1 and SCCFG~2. Light blue identifies tokens matching the final answer at the same position.}
\label{fig:gsmtrajectory}
\end{figure}

\end{document}